\documentclass[onefignum,onetabnum]{siamart251216}

\usepackage{makecell}
\usepackage{siunitx}
\usepackage{booktabs} 
\usepackage{subcaption}
\usepackage{multirow}

\title{A Stabilized Path-Space Approach to Diffusion-Based Posterior Sampling}

\author{Evan Scope Crafts\thanks{Oden Institute for Computational Engineering and Sciences, The University of Texas at Austin, Austin, TX (\email{escopec@utexas.edu}). ESC partially completed this work during an internship at Mitsubishi Electric Research Laboratories (MERL).}
\and Umberto Villa\footnotemark[1] \thanks{Department of Biomedical Engineering, The University of Texas at Austin, Austin, TX (\email{uvilla@austin.utexas.edu}).}
\and Saviz Mowlavi\thanks{Mitsubishi Electric Research Laboratories, Cambridge, MA
  (\email{mowlavi@merl.com}, \email{yma@merl.com}, \email{mansour@merl.com}, \email{wali@merl.com}).}
\and Yanting Ma\footnotemark[3]
\and Hassan Mansour\footnotemark[3]
\and Wael H. Ali\footnotemark[3]}

\ifpdf
\hypersetup{
  pdftitle={A Stabilized Path-Space Approach to Diffusion-Based Posterior Sampling},
  pdfauthor={E. Scope Crafts, U. Villa, S. Mowlavi, Y. Ma, H. Mansour, and W. H. Ali}
}
\fi

\usepackage{xcolor}
\usepackage{booktabs}

\definecolor{todored}{RGB}{200,0,0}
\definecolor{fixorange}{RGB}{230,120,20}
\definecolor{citeblue}{RGB}{0,90,180}
\definecolor{fillpurple}{RGB}{120,0,160}
\definecolor{metagreen}{RGB}{0,130,90}

\usepackage{amssymb}

\newcommand{\bsx}{\mathbf{x}}
\newcommand{\bsy}{\mathbf{y}}

\newcommand{\bsd}{\mathbf{d}}
\newcommand{\bsu}{\mathbf{u}}
\newcommand{\bsw}{\mathbf{w}}
\newcommand{\bsv}{\mathbf{v}}
\newcommand{\bsz}{\mathbf{z}}

\newcommand{\bsA}{\mathbf{A}}

\newcommand{\bsN}{\mathbf{N}}
\newcommand{\bsB}{\mathbf{B}}

\DeclareMathOperator*{\argmax}{arg\,max}
\DeclareMathOperator*{\argmin}{arg\,min}

\newsiamremark{remark}{Remark}
\newsiamremark{hypothesis}{Hypothesis}
\crefname{hypothesis}{Hypothesis}{Hypotheses}
\newsiamthm{claim}{Claim}
\newsiamremark{fact}{Fact}
\newsiamremark{problem}{Problem}
\crefname{fact}{Fact}{Facts}

\begin{document}

\maketitle

\begin{abstract}
Diffusion models provide expressive data-driven priors for Bayesian inverse problems, but many diffusion posterior samplers rely on heuristic guidance approximations that can fail for nonlinear operators and multimodal posteriors. In this work, we develop a stabilized path-space framework for diffusion-based posterior sampling. Starting from a base diffusion process whose terminal marginal represents the prior, we define a likelihood-weighted target measure on trajectories and cast posterior sampling as learning a controlled stochastic process whose path measure matches this target. This formulation connects diffusion posterior sampling to stochastic optimal control while preserving the Bayesian structure needed for uncertainty quantification. We introduce a time reparameterization that makes the path-space control problem well posed by removing the bias induced by the unknown initial value function, without auxiliary training. We then learn the control via a trust-region path-space optimization method with log-variance objectives. The path-space perspective also unifies our learned control approach with existing guidance-based samplers, quantifies the sampling error induced by approximate controls, and yields importance sampling corrections for asymptotically exact posterior expectations. We evaluate the proposed framework on a suite of benchmark inverse problems with analytically characterized or high-quality reference posteriors, enabling principled assessment of sampling accuracy and uncertainty quantification. These experiments provide insight into the behavior of diffusion-based posterior samplers and demonstrate improved accuracy and robustness over leading approaches.
\end{abstract}



\section{Introduction}
\label{sec:intro}
Bayesian inverse problems (BIPs) arise in a wide range of scientific and engineering applications, where the goal is to infer an unknown quantity from indirect and noisy observations. In many fields, such as fluid dynamics \cite{cotter2009bayesian}, geophysics \cite{calvetti2018inverse}, medical imaging \cite{tarvainen2013bayesian}, and computational biology \cite{ingraham2023illuminating}, the resulting posterior distribution is high-dimensional and often multi-modal. In such regimes, posterior sampling, rather than point estimation, is essential for reliable uncertainty quantification and downstream decision-making.

A central challenge in BIPs is the specification of informative prior distributions. Classical priors, such as Gaussian or sparsity-based models, are often too restrictive to capture the complex structure of realistic signals. Recent advances in deep generative models have provided powerful data-driven alternatives \cite{oliviero2025generative, bora2017compressed}, enabling the use of expressive priors trained on representative datasets. When incorporated into Bayesian inference, these models have significantly improved performance on challenging inverse problems by encoding rich prior information unavailable to traditional approaches.

Among learned priors, diffusion models \cite{ho_2020_denoising, song2019generative, song2021scorebased} have become the predominant paradigm for modeling complex, high-dimensional distributions, particularly in image and video domains. These models define a time-indexed transport process that gradually transforms a simple reference distribution (typically Gaussian) into the complex data distribution \cite{song2019generative, song2021scorebased}, and can be viewed as part of a broader measure-transport perspective on sampling \cite{marzouk2016sampling}.
While more recent variants such as flow matching \cite{lipman2023flow} and stochastic interpolants \cite{albergo2025stochastic} have emerged with different theoretical foundations \cite{peluchetti2023non}, 
we refer to all these dynamics-based generative models as diffusion-based models throughout this work, as they share the same fundamental principle of dynamical measure transport between distributions, as detailed in Section~\ref{sec:background} and highlighted in \cite{gao2025diffusion, domingoadjoint2025}. When a trained diffusion model is used to approximate the prior distribution, posterior sampling amounts to modifying the dynamics to incorporate likelihood information, giving rise to a broad class of diffusion-based posterior sampling methods.

\subsection{Diffusion-Based Posterior Sampling}

Existing diffusion-based posterior samplers can be broadly categorized into two paradigms. The first modifies the reverse diffusion dynamics directly using likelihood guidance \cite{chung2023diffusion, song2022solving, song2023pseudoinverseguided, chung2022improving, bansal2023universal, janati2025bridging}. The second class decouples the prior and data consistency steps. For instance, Decoupled Annealing Posterior Sampling (DAPS) \cite{zhang2025improving} alternates between using the base diffusion process to enforce prior consistency and using Markov Chain Monte Carlo (MCMC) algorithms to enforce data consistency. This decoupling strategy can improve stability in challenging image reconstruction problems.

Despite their widespread adoption, both coupled and decoupled heuristic methods suffer from fundamental limitations. Recent studies \cite{scopecrafts2025benchmarking, zhang2025improving} have shown that both classes of approaches often fail catastrophically in challenging regimes, such as strongly non-linear inverse problems or even linear inverse problems with multi-modal posteriors. Here the core issue is that all these methods rely on the approximation of certain conditional distributions associated with the pre-trained diffusion model (in particular, the \textit{denoising distribution}, see, e.g.,  \cite{scopecrafts2025benchmarking} for a definition) as Gaussian or Dirac distributions. When the true conditionals are complex or multi-modal, these approximations can lead to substantial posterior bias.

These limitations have motivated the development of methods that learn conditional corrections to a pretrained diffusion model rather than relying solely on sampling-time approximations. 
In this class of approaches, conditional score or drift models can be trained from paired data consisting of prior samples and corresponding measurements \cite{dhariwal2021diffusion,didi2023framework,denker2024deft}. While effective when such data are available, these approaches are typically amortized over a prescribed training distribution and may require substantial data and computation for each class of inverse problems.

\subsection{The Path-Space Formulation}
A more flexible approach, which we develop in this work for Bayesian inverse problems, formulates posterior sampling as a measure-matching problem in the \emph{path space} of the diffusion model. 
Starting from a base diffusion process whose terminal marginal represents the prior $\pi_{\mathrm{pr}}(\bsx)$, and given an observation $\bsy$ with likelihood $\pi_{\mathrm{like}}(\bsy|\bsx)$, we define a target path measure by reweighting trajectories according to the likelihood of their terminal state. Posterior sampling is then cast as the problem of finding a controlled diffusion, obtained by adding a drift control, whose path measure matches this likelihood-weighted target. When this matching holds, the terminal marginal of the controlled process is the Bayesian posterior.

This path-space formulation replaces local heuristic score corrections with a variational problem over trajectory measures. One may optimize divergences between the controlled path measure and the likelihood-weighted target, such as Kullback–-Leibler (KL) or cross-entropy (CE) divergences, or related proximal divergences over probability measures \cite{nusken2021solving, richter2024improved, holdijk2023stochastic, baptista2025proximal}. This variational perspective treats posterior sampling as a global optimization problem over trajectory measures, connects naturally to stochastic optimal control and Schrödinger bridge formulations \cite{de2021diffusion, chen2016relation, pavon2023onlocal, zhang2023mean, berner2024an}, and provides a principled foundation for incorporating importance sampling corrections that recover asymptotically exact posterior expectations even when the controlled dynamics are only approximately optimized \cite{zhang2021path}.

The path-space perspective has gained traction in recent work on diffusion-based sampling and fine-tuning. Key algorithmic advances include off-policy log-variance objectives \cite{richter2024improved, nusken2021solving}, connections to reinforcement learning through reward-weighted path measures \cite{uehara2024fine}, and the use of lean adjoint estimators for efficient gradient computation \cite{domingoadjoint2025}.
These tools have led to strong results in image diffusion fine-tuning \cite{domingoadjoint2025} and sampling from multi-particle energy functions \cite{havens2025adjoint,liu2025adjoint}. Our work specializes this path-space viewpoint to Bayesian inverse problems, where the target measure is induced by a likelihood and must have the Bayesian posterior as its terminal marginal.

A fundamental challenge in this path-space formulation is the initial value function bias problem. When the base diffusion couples its initial and terminal states, the likelihood-weighted target path measure generally has an initial marginal that is incompatible with controlled processes initialized from the original reference distribution. Previous work addresses related issues using auxiliary value-function learning \cite{uehara2024fine,tang2024fine,liu2025adjoint} or memoryless noise schedules that remove the mismatch asymptotically \cite{domingoadjoint2025}. In this work, we introduce a simple time reparameterization of the base process that exactly ensures the statistical independence between initial and terminal states. This non-intrusive modification preserves the terminal prior distribution and requires no additional neural network training, enabling posterior sampling to be formulated as a well-posed optimization problem in path space.

Beyond ensuring well-posedness, optimizing divergences between path measures remains challenging in practice. Recent work has shown that unconstrained optimization in path space can lead to unstable updates, particularly in high-dimensional or multi-modal settings \cite{blessing2025trust, guo2025proximal}. To address this, we adopt the optimization approach introduced in \cite{blessing2025trust}, which uses KL divergence-based trust regions to ensure controlled exploration during optimization. This geometric constraint is complemented by the use of log-variance objectives \cite{nusken2021solving, richter2024improved}, which provide an off-policy optimization framework that reduces gradient variance and improves numerical robustness.

While trust-region constraints and log-variance objectives address the practical challenges of optimization, a key contribution of this work is to show that the path-space formulation itself enables a broader interpretation of diffusion-based posterior sampling. In particular, methods based on local score guidance or other heuristic drift modifications can be viewed as implicitly defining approximate path measures through controlled modifications of a common reference diffusion. Within this framework, importance sampling weights arise naturally as Radon–Nikodym derivatives between trajectory measures. 
This enables asymptotically exact posterior expectations even from approximate samplers, providing a unifying correction mechanism for both learned and heuristic diffusion-based posterior sampling methods.

Finally, to complement the theoretical and algorithmic developments described above, we conduct a systematic empirical evaluation using benchmark inverse problems from \cite{scopecrafts2025benchmarking}.  
Across all problem settings, we evaluate sampling performance using a consistent set of complementary criteria, including direct comparisons of generated samples as well as statistical measures that assess each method’s ability to capture both local and global structure of the posterior distribution. The results demonstrate that the path space approach consistently provides state-of-the-art performance, significantly surpassing the performance of leading heuristic approaches.

\subsection{Contributions}

Our key contributions are summarized as follows:

\textbf{A well-posed path-space formulation for posterior sampling.}
We develop a path-space formulation of diffusion-based posterior sampling for Bayesian inverse problems. Starting from a base diffusion process whose terminal marginal represents the prior, we define a likelihood-weighted target path measure whose terminal marginal is the Bayesian posterior. Matching this target measure yields an optimal control, i.e., an additive drift correction to the base process. We identify the initial-value-function bias that can make this target incompatible with controlled processes initialized from the original reference distribution, and introduce a simple time reparameterization that restores the required separability condition while preserving the terminal distributions.

\textbf{Path-space learning and correction of posterior samplers.} To approximate the optimal control in practice, we adopt a trust-region path-space optimization strategy together with log-variance objectives. Within the path-space posterior sampling framework, learned and heuristic diffusion posterior samplers are interpreted as approximate controls. This perspective allows us to quantify the posterior sampling error induced by control error and to derive Radon--Nikodym importance weights that yield asymptotically exact posterior expectations even when the underlying sampler is approximate.

\textbf{Systematic empirical evaluation.} We conduct a comprehensive benchmark study of diffusion-based posterior samplers on a suite of linear and nonlinear inverse problems, including stylized versions of Gaussian random sensing, inpainting, x-ray tomography, and phase retrieval. Our evaluation leverages the path-space perspective to provide rigorous and interpretable comparisons across problems of varying complexity.

\subsection{Outline} The remainder of the paper is organized as follows. Section~\ref{sec:background} reviews SDE-based generative modeling and presents diffusion models, flow-based models, and stochastic interpolants in a common framework. Section~\ref{sec:prob_form} formulates diffusion-based posterior sampling as a path-space measure-matching problem. Section~\ref{sec:solution_approach} presents the proposed approach, including the time reparameterization technique for ensuring well-posedness. Section~\ref{sec:samperror_importsamp} analyzes control-induced sampling error and derives path-space importance weights for posterior expectations. Sections~\ref{sec:num_studies} and~\ref{sec:results} describe the numerical setup and report the experimental results. Section~\ref{sec:discuss_conclude} concludes the paper.

\section{Background: Generative Models as SDEs}
\label{sec:background}

In this section, we briefly describe conditions under which three commonly-used generative modeling frameworks ---diffusion models, flow models, and stochastic interpolants---can be written in the common language of stochastic differential equations (SDEs). 

The relationship between diffusion models and SDEs is well established. In \cite{song2021scorebased}, Song et al. showed that the early discrete-time variants of diffusion models introduced in \cite{song2019generative} and \cite{ho_2020_denoising} could be written as discretizations of a time-reversed SDE. The original SDE, which gradually adds noise to samples from the target data distribution, can be written as 
\begin{equation}
\label{eq:forward_diffusion_sde}
d\bsx_t = \mathbf{g}(\bsx_t, t) \, dt + h(t) \, d \tilde{\bsw}_t, \quad t \in [0, T].
\end{equation}
Here $d \tilde{\bsw}_t$ is a Wiener process moving backwards in time, $\bsx_t \in \mathbb{R}^D$ is the state variable, $\bsx_T$ is sampled from the data distribution,\footnote{Note that we are using the opposite time convention from what is typically employed in diffusion modeling, with $t=0$ here representing the noise distribution. This choice was made for the sake of consistency with the path-space and stochastic control literature.} $\mathbf{g}(\bsx, t): \mathbb{R}^D \times [0, T] \to \mathbb{R}^D$ is a time-indexed drift term, and $h(t): [0, T] \to \mathbb{R}_{+}$ controls the diffusion rate. The time-reversed SDE, which maps a fixed noise distribution to samples from the target data distribution, can be written as 
\begin{equation}
\label{eq:backward_diffusion_sde}
d\bsx_t = [\mathbf{g}(\bsx_t, t) - h^2(t) \; \nabla_{\bsx} \log p_t(\bsx_t)] \, dt +  h(t) \, d\bsw_t.
\end{equation}
Here $\nabla_{\bsx} \log p_t(\bsx_t)$ is the \emph{score} of the time-$t$ marginal distribution of $\bsx_t$, which can be learned from data \cite{vincent2011connection}.

Unlike diffusion models, flow models are inherently deterministic. However, they can be cast into the SDE framework by considering the evolution of their marginal densities. For a flow defined by
\begin{equation}
d\bsx_t = \bsv(\bsx_t, t) \, dt, \quad t \in [0, T],  \label{eq:flow_velocity}
\end{equation}
with $\bsx_0 \sim p_0$ sampled from the noise distribution, the marginal densities $p_t(\bsx_t)$ induced by the above process satisfy the Fokker-Planck equation
\begin{equation}
\frac{\partial p_t(\bsx_t)}{\partial t} = - \nabla_{\bsx} \cdot \left [ \bsv(\bsx_t, t) \, p_t(\bsx_t) \right].
\label{eq:det_fokker_plank}
\end{equation}
For some function $g_s(t): [0, T] \to \mathbb{R}$, Nelson's identity \cite{nelson1966derivation}
\begin{equation}
\frac{1}{2} g_s(t)^2 \Delta_{\bsx} p_t(\bsx_t) = \frac{1}{2} g_s(t)^2 \nabla_{\bsx} \cdot (p_t(\bsx_t) \nabla_{\bsx} \log p_t(\bsx_t))
\label{eq:nelsons_identity}
\end{equation}
can therefore be used to rewrite \eqref{eq:det_fokker_plank} as
\begin{equation*}
\frac{\partial p_t(\bsx_t)}{\partial t} = - \nabla_{\bsx} \cdot \left( (\bsv(\bsx_t, t) + \frac{1}{2} g_s(t)^2 \, \nabla_{\bsx} \log p_t(\bsx_t)) p_t(\bsx_t) \right) + \frac{1}{2} g_s(t)^2 \Delta_{\bsx} p_t(\bsx_t).
\end{equation*}
By construction this new Fokker-Planck equation has the same time-$t$ marginals as \eqref{eq:flow_velocity}, but corresponds to the SDE 
\begin{equation*}
d\bsx_t = \left(\bsv(\bsx_t, t) + \frac{1}{2} g_s(t)^2 \, \nabla_{\bsx} \log p_t(\bsx_t) \right) dt  + g_s(t) \, d\bsw_t,
\end{equation*}
where here $d \bsw_t$ is Brownian motion. Note that for many commonly used flow models, $\nabla_{\bsx} \log p_t(\bsx_t)$ can be obtained analytically given knowledge of $\bsv(\bsx_t, t)$. The above construction thus enables a pre-trained flow model to admit SDE-based solvers. Since SDE solvers may be more robust to error in the learned flow than their deterministic counterparts (see, e.g., Section 2.4 of \cite{albergo2025stochastic}), the use of SDE solvers for flow models has garnered significant research interest in recent years  \cite{singh2024stochastic}. 

Stochastic interpolants \cite{albergo2025stochastic} can be viewed as a generalization of both diffusion and flow models. They are based on the following equation:
\begin{equation*}
\bsx_t = \boldsymbol{I}(\bsx_0, \bsx_T, t) + \gamma(t) \, \bsz,
\end{equation*}
where $\bsz \sim \mathcal{N}(\mathbf{0}, \mathbf{I})$, $\gamma(0) = \gamma(T) = 0$, and the interpolation function $\boldsymbol{I}$ satisfies $\boldsymbol{I}(\bsx_0, \bsx_T, 0) = \bsx_0$ and $\boldsymbol{I}(\bsx_0, \bsx_T, T) = \bsx_T$. In the generative modeling setting, $\bsx_0$ is sampled from a fixed noise distribution and $\bsx_T$ is a sample from the data distribution. If $\gamma(t) > 0$ for all $t \in (0, T)$, then the above equation describes a process that can be modeled as an SDE with drift and diffusion coefficients obtained by taking the limit $\bsx_t - \bsx_{t + \epsilon}$ as $\epsilon \to 0$. If $\gamma(t) = 0$ (e.g., in a rectified flow model \cite{liu2022flow}, which uses the linear interpolation strategy $\boldsymbol{I}(\bsx_0, \bsx_T, t) = (1 - t) \, \bsx_0 + t \, \bsx_T$), then the interpolation process is deterministic over part or all of the trajectory space. However, an SDE can still be obtained using Nelson's identity as described above.

\section{The Path-Space Formulation of Diffusion-Based Posterior Sampling}
\label{sec:prob_form}

In this section, we introduce the diffusion-based posterior sampling problem considered in this work. We then show how this problem can be reformulated as a search for a particular measure over the path space of the diffusion model. 

\subsection{Problem Formulation}

We assume we have access to a \textit{diffusion-type} generative model that approximately samples from the prior distribution $\pi_{\mathrm{pr}}(\bsx)$ of a Bayesian inference problem with unknown parameter $\bsx \in \mathbb{R}^D$. This generative model can be written as the following SDE:
\begin{equation}
\label{eq:base_process}
d\bsx_t = \bsA_b(\bsx_t, t) \, dt + \bsB_b(t) \, d\bsw_t, \quad \bsx_0 \sim \mathcal{N}(\mathbf{0}, \mathbf{I}_D), \quad t \in [0, T]. 
\end{equation}
In the above expression $\bsx_t \in \mathbb{R}^D$ is the SDE state, $\bsA_b(\cdot, \cdot): \mathbb{R}^D \times [0, T] \to \mathbb{R}^D$ is the drift term, $\bsB_b(\cdot): [0, T] \to \mathbb{R}^{D \times D}$ is the diffusion term, and $d \bsw_t$ is standard Brownian motion. We assume that $\bsB_b(t)$ is an invertible matrix for all $t \in [0, T]$, and that $\bsA_b(\bsx_t, t)$ satisfies the ``standard'' Lipschitz and linear growth conditions required for \eqref{eq:base_process} to have a $t$-continuous strong solution (see, e.g., Theorem 5.2.1 of \cite{oksendal2003stochastic}).\footnote{Flow models correspond to $\bsB_b(t) \equiv \mathbf{0}$ for all $t$ and thus violate this assumption. However, as detailed in Section \ref{sec:background}, any flow model can be converted into a non-trivial SDE using Nelson's identity.} Further, if the generative model is trained exactly, we have that $p_{T}(\cdot) = \pi_{\mathrm{pr}}(\cdot)$, where $p_T$ is the time-$T$ marginal distribution of the above diffusion process. In what follows, we refer to this process as the \textit{base process}.

Our goal, given the base diffusion process and a measurement $\bsy$ from a known likelihood function $\pi_{\mathrm{like}}(\bsy \mid \bsx)$, is to obtain samples from the posterior distribution $\pi_{\mathrm{post}}(\bsx \mid \bsy)$, which can be written via Bayes' Theorem as 
\begin{equation}
    \pi_{\mathrm{post}}(\bsx \mid \bsy) \propto \pi_{\mathrm{like}}(\bsy \mid \bsx) \, \pi_{\mathrm{pr}}(\bsx). 
\end{equation}
In particular, we formulate the problem as seeking a measurable function $\bsu_b(\cdot, \cdot): \mathbb{R}^D \times [0, T) \to \mathbb{R}^D$ such that the modified process
\begin{equation}
\label{eq:controlled_process_noreparam}
d\bsx_t^{\bsu_b} =[ \bsA_b(\bsx^{\bsu_b}_t, t) + \bsB_b(t) \, {\bsu_b}\left(\bsx_t^{\bsu_b}, t\right)] \, dt + \bsB_b(t) \; d\bsw_t, \quad \bsx_0^{{\bsu_b}} \sim \mathcal{N}(\mathbf{0}, \mathbf{I}_D), \quad t \in [0, T],
\end{equation}
has the same time-$T$ marginal distribution $p_T^{{\bsu_b}}$ as the posterior distribution, i.e., $\bsx_T^{{\bsu_b}} \sim \pi_{\mathrm{post}}(\bsx \mid \bsy)$. Note that the $\bsB_b(t)$ term in front of the control can be dropped without loss of generality, as we have assumed $\bsB_b(t)$ is invertible; however it is standard to incorporate it as it simplifies the resulting mathematics. Also note that here we assume ${\bsu_b} \in \mathcal{U}([0, T))$, where $\mathcal{U}(\cdot)$ is the space of measurable functions over the given interval satisfying the same growth conditions as required for $\bsA_b$. Note that ${\bsu_b}$ can be interpreted as the control term in a stochastic optimal control problem. In what follows, we therefore refer to ${\bsu_b}$ as the control, and \eqref{eq:controlled_process_noreparam} as the \textit{controlled process}. 

The above problem is formalized in the following problem statement.  

\begin{problem} [Diffusion Posterior Sampling]
\label{problem:posterior_samp}
Consider a Bayesian inverse problem with prior distribution $\pi_{\mathrm{pr}}(\bsx)$, $\bsx \in \mathbb{R}^D$ and likelihood function $\pi_{\mathrm{like}}(\bsy \mid \bsx)$, $\bsy \in \mathbb{R}^K$. Given the likelihood, a diffusion-type model of the form \eqref{eq:base_process} that samples from the prior distribution, and a measurement $\bsy$, find a feedback control ${\bsu_b} \in \mathcal{U}([0, T))$ such that $\bsx_T^{{\bsu_b}} \sim \pi_{\mathrm{post}}(\bsx \mid \bsy)$.
\end{problem}

\subsection{Reformulation in Path Space}

 Problem \ref{problem:posterior_samp} above is not well-posed in the sense that there is an infinite number of controls that solve the problem. For example, if $\bsA \equiv 0$, it is straightforward to see that there are an infinite number of time-dependent scalings of ${\bsu_b}$ that leave $p_{T}^{\bsu_b}$ unchanged. In this work, we overcome this ill-posedness by reformulating the posterior sampling problem in \emph{path space}. 
 
Formally, we define the path space as
 \begin{equation}
 \label{eq:path_space_deff}
 \mathcal{C}(s) \triangleq \{ \bsx(t): [s, T] \to \mathbb{R}^D \mid \bsx(t) \text{ continuous} \},
 \end{equation}
 where here $s < T$, and equip it with the supremum norm and the corresponding Borel-$\sigma$-algebra generated by the open sets under this norm. This leads to the following problem formulation. 

 \begin{problem}[Path-Space Posterior Sampling]
 \label{problem:path_space_sampling}
 Let $\mathcal{C}(s)$ be the path space defined in \eqref{eq:path_space_deff} and let $P$ be the measure induced by a diffusion-type process of the form given in \eqref{eq:base_process} over the space. The \emph{target path measure} $Q^*$ is defined as the path measure over $\mathcal{C}(s)$ whose Radon-Nikodym derivative with respect to $P$ (see Appendix \ref{sec:appendix_pathmeasures} for a definition) is proportional to the likelihood function, i.e., 
 \begin{equation}
 \frac{d Q^*}{dP}(\bsx_{0:T}) \triangleq \frac{\pi_{\mathrm{like}}(\bsy \mid \bsx_T)}{Z}, \quad Z \triangleq \mathbb{E}_{P} \left [ \pi_{\mathrm{like}}(\bsy \mid \bsx_T) \right]. 
 \label{eq:Qstar_def}
\end{equation}
Note that this definition ensures that the marginal distribution of $Q^*$ at time $T$ is the posterior distribution under the assumption that $p_T = \pi_{\mathrm{pr}}$. We seek a control ${\bsu_b} \in \mathcal{U}([s, T))$ with corresponding path measure $P^{{\bsu_b}}$ such that $ dP^{{\bsu_b}} / dQ^*$ is constant, i.e., $P^{{\bsu_b}} = Q^*$.
 \end{problem}

 The above problem formulation resolves the ill-posedness of Problem \ref{problem:posterior_samp} in the sense that if there is a ${\bsu_b}$ that solves Problem \ref{problem:path_space_sampling}, it is necessarily unique.\footnote{This follows from Girsanov's theorem \cite{oksendal2003stochastic}: since any $\bsu'_b \in \mathcal{U}$ can be written as $\bsu'_b= \bsu_b + \boldsymbol{\delta}$, if $\boldsymbol{\delta}$ is non-zero on a measurable set, this change in drift leads to a non-trivial change in the corresponding path measure.} However, in general Problem \ref{problem:path_space_sampling} does not have a solution. The reason for this is that $Q^*$ will not have the same Gaussian initial distribution as $P^{{\bsu_b}}$. In particular, as demonstrated in \cite{domingoadjoint2025, liu2025adjoint}, we obtain from \eqref{eq:Qstar_def} that the joint distribution of $(\bsx_0,\bsx_T)$ under $Q^*$ is
\begin{equation}
    q^*_{0,T}(\bsx_0,\bsx_T)
    =
    p_{0,T}(\bsx_0,\bsx_T)
    \frac{\pi_{\mathrm{like}}(\bsy \mid \bsx_T)}{Z},
    \label{eq:qstar_joint_initial_terminal}
\end{equation}
where $p_{0,T}(\bsx_0,\bsx_T)$ denotes the corresponding joint distribution under the base process $P$ given by \eqref{eq:base_process}. Marginalizing \eqref{eq:qstar_joint_initial_terminal} over $\bsx_T$ gives the initial marginal of the target path measure:
\begin{equation}
    q^*_0(\bsx_0) = p_0(\bsx_0) \frac{ \mathbb{E}_P \left[ \pi_{\mathrm{like}}(\bsy \mid \bsx_T) \mid \bsx_0 \right]}{Z}.
    \label{eq:qstar_initial_marginal}
\end{equation}
Equivalently, defining the initial value function
\begin{equation}
    V(\bsx) \triangleq - \log \mathbb{E}_P \left[ \pi_{\mathrm{like}}(\bsy \mid \bsx_T) \mid \bsx_0 = \bsx \right],
\end{equation}
we may write
\begin{equation}
    q^*_0(\bsx_0) = p_0(\bsx_0) \frac{\exp(-V(\bsx_0))}{Z}.
    \label{eq:qstar_initial_value_function}
\end{equation}

Equation \eqref{eq:qstar_initial_value_function} shows that, in general, the initial marginal of $Q^*$ differs from the initial marginal $p_0$ of the base process. The two coincide only when $ V(\bsx)$ is constant. A sufficient condition for this to hold is the separability of the initial and terminal states under the base process, $p_{0,T}(\bsx_0,\bsx_T) = p_0(\bsx_0) p_T(\bsx_T)$.
In that case,
\begin{equation}
    V(\bsx) = - \log \mathbb{E}_P \left[ \pi_{\mathrm{like}}(\bsy \mid \bsx_T) \mid \bsx_0 = \bsx \right] = - \log \mathbb{E}_P \left[ \pi_{\mathrm{like}}(\bsy \mid \bsx_T) \right] = - \log Z,
\end{equation}
and $ q^*_0(\bsx_0) = p_0(\bsx_0)$.

However, many diffusion-type generative models have non-trivial coupling between the time-$0$ and time-$T$ states. In these settings, Problem \ref{problem:path_space_sampling} does not have an exact solution.

The lack of exact solutions to Problem \ref{problem:path_space_sampling} when $p_{0, T}(\bsx_0, \bsx_T)$ is not separable is referred to as the \textit{initial value function bias} problem in the literature \cite{domingoadjoint2025}, and two main approaches have been proposed to address it. In \cite{uehara2024fine} and \cite{tang2024fine}, the authors use an auxiliary neural network to learn $V(\bsx)$ and the corresponding initial distribution $q_0^*$. In \cite{domingoadjoint2025}, Domingo-Enrich et al. address the problem by using Nelson's identity, as described in Section \ref{sec:background}, to inject more stochasticity into the base process. The resulting base process solves the bias problem since the correlation between $\bsx_0$ and $\bsx_T$ decreases exponentially as $T \to \infty$; however, $p_{0, T}(\bsx_0, \bsx_T)$ is never exactly separable.

\section{Solution Approach}
\label{sec:solution_approach}

In this section, we propose a novel technique, referred to here as the \textit{time reparameterization trick} (TRT), that requires only a minor modification to the base process but ensures the separability condition $p_{0, T}(\bsx_0, \bsx_T) = p_0(\bsx_0) \, p_T(\bsx_T)$ introduced in the previous section is exactly satisfied. We then describe the approach used to solve the time-reparameterized problem.

\subsection{The Time Reparameterization Trick}

The proposed time reparameterization trick is based on the observation that both the base process defined in \eqref{eq:base_process} and the controlled process defined in \eqref{eq:controlled_process_noreparam} can be trivially extended to one with a deterministic initial condition. In particular, we consider the following extended version of the base process:
\begin{equation}
\label{eq:uncontrolled_reparam_dynamics}
d\bsx_{t} = \bsA(\bsx_{t}, t) \, dt + \bsB(t) \, d\bsw_{t}, \quad \bsx_{-1} = \mathbf{0}, \quad t \in [-1, T],
\end{equation}
where here $\bsA$ and $\bsB$ are defined in terms of the original base process as 
\begin{equation}
\bsA(\bsx_t, t) \triangleq \begin{cases}
    \bsA_b(\bsx_t, t) \quad \mathrm{if} \; t \geq 0, \\
    \mathbf{0} \hspace{1.48cm} \mathrm{if} \; t < 0,
\end{cases} \quad \bsB(t) \triangleq \begin{cases}
    \bsB_b(t) \quad \mathrm{if} \; t \geq 0, \\
    \mathbf{I}_D \hspace{.8cm} \mathrm{if} \; t < 0.
\end{cases}
\end{equation}
Note that the drift and diffusion coefficients above trivially satisfy the standard regularity conditions required for \eqref{eq:uncontrolled_reparam_dynamics} to admit $t$-continuous strong solutions, and therefore the SDE induces a measure over the path space $\mathcal{C}(s)$ defined in \eqref{eq:path_space_deff}, where here $s = -1$. 
Further, the above process satisfies $\bsx_0 \sim \mathcal{N}(\mathbf{0}, \mathbf{I}_D)$, and therefore the time-$t$ marginals of the above process are the same as the time-$t$ marginals of the base process in \eqref{eq:base_process} for $t \geq 0$. In particular, we have that $\bsx_{T} \sim \pi_{\mathrm{pr}}(\bsx)$ if the base generative model is well-trained.

As with the original base process, we define the corresponding controlled version of \eqref{eq:uncontrolled_reparam_dynamics} as 
\begin{equation}
\label{eq:control_reparam_dynamics}
d\bsx_{t}^{\bsu} =[ \bsA(\bsx_{t}^{{\bsu}}, t) + \bsB(t) \, {\bsu}(\bsx_{t}^{\bsu}, t)] \, dt + \bsB(t) \, d\bsw_{t}, \quad \bsx_{-1}^{\bsu} = \mathbf{0}, \quad t \in [-1, T],
\end{equation}
where here $\bsu(\cdot, \cdot): \mathbb{R}^D \times [-1, T) \to \mathbb{R}^D$ satisfies $\bsu \in \mathcal{U}([-1, T))$. Note that the modified base and controlled processes both have time discontinuities in their drift and diffusion coefficients at $t=0$. However, this poses no issue in our problem context, as shown by the following theorem.

\begin{theorem}
Consider the path-space posterior sampling problem given in Problem \eqref{problem:path_space_sampling} with base path measure $P$ corresponding to the SDE in \eqref{eq:uncontrolled_reparam_dynamics}. Assume that the likelihood $\pi_{\mathrm{like}}(\bsy \mid \bsx_T)$ is strictly positive and bounded, with a score function $\nabla_{\bsx} \log \pi_{\mathrm{like}}(\bsy \mid \bsx_T)$ that is globally Lipschitz. Then there exists a control $\bsu \in \mathcal{U}([-1, T))$ that solves Problem \eqref{problem:path_space_sampling}, i.e., the corresponding controlled path measure $P^{\bsu}$ induced by \eqref{eq:control_reparam_dynamics} coincides with the target posterior measure $Q^*$. 
\label{theorem:sol_existence}
\end{theorem}

\begin{proof} 
We first note that by definition, $Q^*$ is absolutely continuous with respect to $P$, and its Radon-Nikodym derivative corresponds to the density process
\begin{equation*}
z_t \triangleq \mathbb{E}_P \left[ \frac{dQ^*}{dP} \mid \mathcal{M}_t \right], \quad t \in [-1, T], \end{equation*}
where $\mathcal{M}_t$ is the filtration induced by the Brownian motion in \eqref{eq:uncontrolled_reparam_dynamics}. By the tower property of conditional expectations, $z_t$ is a martingale with respect to the filtration $\mathcal{M}_t$. Since $P$ and $Q^*$ share a fixed initial condition at $t=-1$, we have that $z_{-1} = \mathbb{E}_P[ dQ^* / dP ] = 1$, and since by assumption $\pi_{\mathrm{like}}(\bsy \mid \bsx_T) > 0$, $z_t$ is a strictly positive martingale.

Now by the martingale representation theorem (MRT) \cite{oksendal2003stochastic}, the fact that $z_t$ is a martingale implies that there exists an $\mathcal{M}_t$ adapted, square-integrable process $\boldsymbol{\phi}_t$ such that the dynamics of $z_t$ are given by:
\begin{equation*}
d z_t = \langle \boldsymbol{\phi}_t, d\bsw_t \rangle.
\end{equation*}
Further, since $z_t$ is strictly positive, the transformation $\bsd_t = \boldsymbol{\phi}_t / z_t$ is well-defined and we can write
\begin{equation}
d z_t = z_t \, \langle \bsd_t, d\bsw_t \rangle,\label{eq:mrt_dynamics}
\end{equation}
where $\bsd_t$ is also $\mathcal{M}_t$ adapted and square integrable. By applying It\^o’s lemma \cite{oksendal2003stochastic} to $\log z_t$, we obtain that the solution to this SDE is the exponential
\begin{align*}
z_t &= z_{-1} \exp \left( \int_{-1}^t \langle \bsd_s, d\bsw_s \rangle - \frac{1}{2} \int_{-1}^t \left \| \bsd_s \right \|_2^2 \, ds \right) \\
&= \exp \left( \int_{-1}^t \langle \bsd_s, d\bsw_s \rangle - \frac{1}{2} \int_{-1}^t \left \| \bsd_s \right \|_2^2 \, ds \right),
\end{align*}
where the second equality follows from $z_{-1} = 1$ due to the deterministic initial condition. 

Next, note that since the base SDE \eqref{eq:uncontrolled_reparam_dynamics} is Markovian and $d Q^* / d P$ is a functional only of the terminal state $\bsx_T$, the density process can be represented as $z_t = g(\bsx_t, t)$, where $g(\bsx_t, t) \triangleq \mathbb{E}_P[d Q^* / d P \mid \bsx_t ]$. This Markovian structure implies that the integrand $\bsd_t$ from the martingale representation theorem must be a function of the current state and time. Specifically, by applying It\^{o}'s lemma to $g(\bsx_t, t)$ and using the fact that $g$ is a martingale and therefore must have a corresponding SDE with zero drift term, we obtain
$$
dz_t = \left \langle \bsB(t)^T \nabla_{\bsx} \, g(\bsx_t, t), d\bsw_t \right \rangle, 
$$
and therefore 
$$
dz_t = g(\bsx_t, t) \left \langle \bsB(t)^T \nabla_{\bsx} \log g(\bsx_t, t), d\bsw_t \right \rangle.
$$
Comparing the above equation with \eqref{eq:mrt_dynamics}, we can therefore identify the integrand as $\bsd_t = \bsB(t)^T \nabla_{\bsx} \log g(\bsx_t, t)$.

Now consider the path measure $P^{\bsu}$ induced by the controlled SDE \eqref{eq:control_reparam_dynamics} with a given control $\bsu$. By Girsanov’s theorem \cite{oksendal2003stochastic}, the Radon-Nikodym derivative of $P^{\bsu}$ with respect to $P$ restricted to the filtration $\mathcal{M}_t$, denoted $z_t^{\bsu}$, is
\begin{equation}
z_t^{\bsu} \triangleq \exp \left( \int_{-1}^t \langle \bsu(\bsx_s, s), d\bsw_s \rangle - \frac{1}{2} \int_{-1}^t \left \| \bsu(\bsx_s, s) \right \|_2^2 \, ds \right).\label{eq:girsanov_density}
\end{equation}
By setting $\bsu(\bsx_t, t) = \bsd_t = \bsB(t)^T \nabla_{\bsx} \log g(\bsx_t, t)$, the density process $z_T^{\bsu}$ matches $z_T = d Q^* / dP$ exactly. Thus, $P^{\bsu}$ and $Q^*$ are identical almost everywhere with respect to $P$.

What remains to show is that $\nabla_{\bsx} \log g(\bsx_t, t)$, $t \in [-1, T)$, is well-defined and that $\bsu = \bsB(t)^T \nabla_{\bsx} \log g(\bsx_t, t)$ is in $\mathcal{U}([-1, T))$. Here we first note that by the Kolmogorov backward equation \cite{oksendal2003stochastic}, $g$ satisfies a parabolic PDE. This, together with the assumption that $\pi_{\text{like}}(\bsy \mid \bsx_T)$ is strictly positive and bounded, ensures that the conditional expectation $g(t, \bsx)$ is bounded away from zero and smooth in $\bsx$ for all $t < T$, despite the time discontinuity in the base process coefficients. Consequently, $\nabla_{\bsx} \log g(\bsx_t, t)$ is well-defined. Further, since we have assumed that $\nabla_{\bsx} \log \pi_{\text{like}}$ is globally Lipschitz, the spatial regularity of the terminal condition is propagated backward by the heat kernel of the base SDE. Thus, both the gradient $\nabla_{\bsx} \log g(t, \bsx)$ and the corresponding control $\bsu$ inherit the Lipschitz and linear growth properties required for $\bsu \in \mathcal{U}$, completing the proof.
\end{proof}

\begin{remark}
    The proof techniques (MRT, It\^{o}'s Lemma, and Girsanov's Theorem) and general strategy used in the above proof are well-established; see Appendix \ref{sec:appendix_pathmeasures} for a review. However, the probabilistic approach used to form the density process corresponding to $dQ^* / dP$, which avoids explicit construction of a corresponding stochastic optimal control problem or invocation of the Hamilton-Jacobi-Bellman equation (see, e.g., the proof of Theorem 2.2 in \cite{nusken2021solving}), is somewhat non-standard.
\end{remark}

\begin{remark}
    The above proof actually gives the optimal control
    \begin{equation}
    \bsu^*(\bsx_t, t) =  \bsB(t)^T \, \nabla_{\bsx} \log g(\bsx_t, t) = \bsB(t)^T \, \nabla_{\bsx} \log \mathbb{E}_P[d Q^* / d P \mid \bsx_t], \label{eq:analytical_optimal_control}
    \end{equation} a formula which is well known in the diffusion posterior sampling literature (see, e.g.,  \cite{scopecrafts2025benchmarking}). However,  $d Q^* / d P$ is defined in terms of $\pi_{\mathrm{like}}(\bsy \mid \bsx_T)$, with $\bsx_T$ obtained from the prior-sampling base process. The conditional expectation in the above formula can therefore have very high variance if the KL divergence between the posterior and prior is large, making its direct use for posterior sampling impractical. Further, the computation of the gradient in the above expression requires differentiating through the sampling trajectory, which can be expensive to compute. In practice, many leading diffusion posterior sampling approaches replace the conditional distribution underlying the optimal control with a tractable approximation. Although these approximations lead to computationally efficient guidance rules, they collapse the full conditional structure of the terminal state and can therefore introduce significant bias; see Section \ref{sec:samperror_importsamp} for an extended discussion.
    \end{remark}

\begin{figure}
    \centering
    \includegraphics[width=\linewidth]{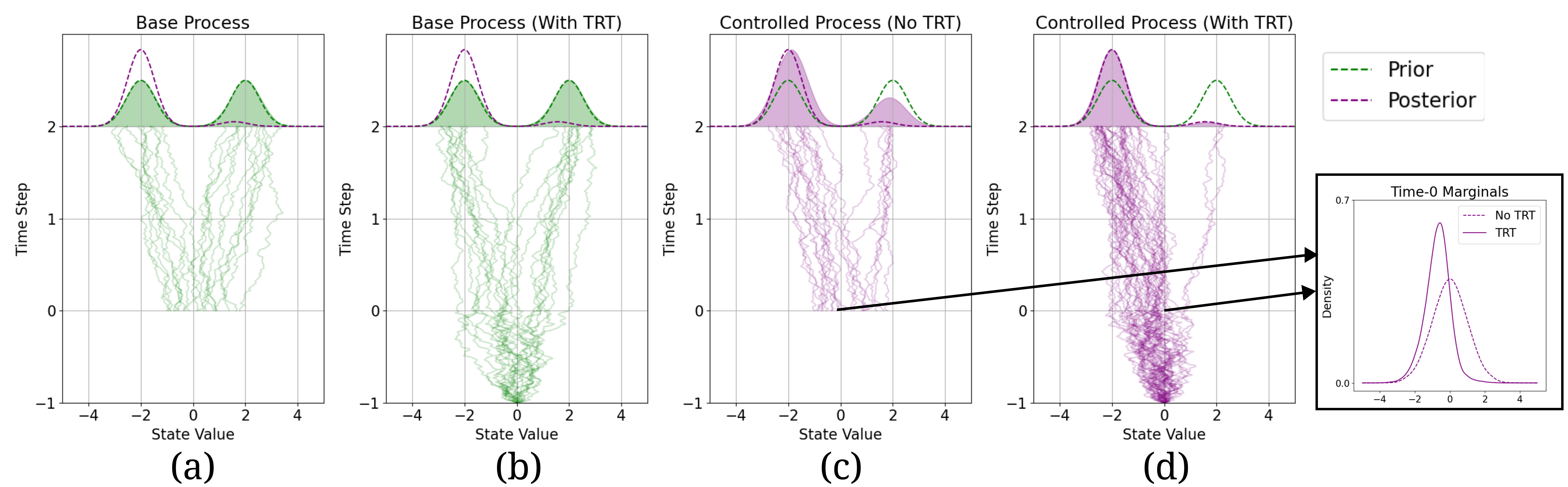}
    \caption{Illustration of Theorem \ref{theorem:sol_existence}. (a-b) Trajectories from the base process with and without the time reparameterization trick (TRT). Note that both base processes sample from the prior distribution of the given inverse problem. (c-d) Trajectories from the corresponding controlled process, which are intended to sample from the posterior. A comparison of the time-zero marginal distributions of the two processes is also shown. As can be seen, the controlled process with TRT has non-Gaussian time-zero marginals and correct posterior samples, while using the path-sampling approach without the TRT leads to significant sampling error.}
    \label{fig:trt_example}
\end{figure}

We demonstrate the use of the time-reparameterization trick (TRT) on a toy example problem with one-dimensional Gaussian mixture prior and linear-Gaussian likelihood. In particular, we constructed an analytically defined flow model, with stochasticity injected using Nelson's identity as in Section \ref{sec:background}. We then solved Problem \ref{problem:path_space_sampling} analytically for both the modified and unmodified processes using the control given by \eqref{eq:analytical_optimal_control}; see Appendix \ref{sec:appendix_flow} for details regarding the construction of the base process and computation of the optimal control. Figure \ref{fig:trt_example} shows the base and controlled processes before and after applying the TRT. As can be seen, the TRT resolves the value function bias problem, with the controlled process with TRT correctly sampling from the posterior distribution of the example problem. 

\subsection{Solving the Reparameterized Problem}
\label{subsec:solving_approach}

The time reparameterization trick ensures that the path-space posterior sampling problem is well-posed. To solve the reparameterized problem, we consider an iterative diffusion optimization (IDO) approach \cite{nusken2021solving, domingo2024stochastic}. In IDO approaches, iterative gradient-based algorithms are used to minimize objectives of the form
\begin{equation}
L(\bsu) \triangleq \mathcal{D}(P^{\bsu}, Q^*),
\label{eq:abstract_prob_form}
\end{equation}
where $\mathcal{D}$ is a statistical divergence. Here the computation of the divergence relies on the Radon-Nikodym derivative of $Q^*$ with respect to $P^{\bsu}$ evaluated along trajectories $\bsx_{-1:T}$:
\begin{align}
\frac{dQ^*}{dP^{\bsu}}(\bsx_{-1:T}) &= \frac{dQ^*}{dP}(\bsx_{-1:T}) \frac{dP}{dP^{\bsu}}(\bsx_{-1:T}) \nonumber \\
&= \frac{\pi_{\mathrm{like}}(\bsy \mid \bsx_T)}{Z} \exp \left( - \int_{-1}^T \langle \bsu(\bsx_s, s), d\bsw_s^{\bsu} \rangle - \frac{1}{2} \int_{-1}^T \left \| \bsu(\bsx_s, s) \right \|_2^2 \, ds \right), \label{eq:radon_nikodym_optimal}
\end{align}
where here have used the definition of $Q^*$ in \eqref{eq:Qstar_def}, the Girsanov density from \eqref{eq:girsanov_density}, and the identity $d\bsw_s = d\bsw_s^{\bsu} + \bsu(\bsx_s, s) \, ds$, where $d\bsw_s^{\bsu}$ is Brownian motion under $P^{\bsu}$ (see Appendix \ref{sec:appendix_pathmeasures} for details).

In this work we use the trust-region IDO approach introduced in \cite{blessing2025trust}. This method employs the KL divergence $\mathcal{D}_{\mathrm{KL}}$ and enforces a trust region constraint for stability; see Appendix \ref{app:ablation} for an ablation study regarding the impact of the trust region on optimization stability. 

Specifically, the trust region approach solves the sequence of sub-problems 
\begin{equation}
\label{eq:trust_region_subproblem}
\bsu_{i+1} = \argmin_{\bsu \in \mathcal{U}}  \mathcal{D}_{\mathrm{KL}}(P^\bsu \, \Vert \, Q^*)  \quad \text{subject to} \quad \mathcal{D}_{\mathrm{KL}}(P^{\bsu} \, \Vert \, P^{\bsu_i}) \leq \epsilon,
\end{equation}
where here $i = 0, \dots, I-1$, $\epsilon > 0$ is the size of the trust region, $\bsu_0$ is the control initialization, and  $\mathcal{D}_{\mathrm{KL}}(P^{\bsu} \, \Vert \, P^{\bsu_i})$ denotes the KL divergence between $P^{\bsu}$ and $P^{\bsu_i}$, which by Girsanov's theorem (see Appendix \ref{sec:appendix_pathmeasures}) can be written as a quadratic penalty, i.e.,
\begin{equation}
\mathcal{D}_{\mathrm{KL}}(P^{\bsu} \, \Vert \, P^{\bsu_i}) = \mathbb{E}_{P^{\bsu}} \left[ \frac{1}{2} \int_{-1}^T \| \bsu(\bsx_t, t) - \bsu_i(\bsx_t, t) \|_2^2 \, dt \right].
\label{eq:kl_trust_region}
\end{equation}
The subproblem in \eqref{eq:trust_region_subproblem} is solved by first forming the corresponding Lagrangian
$$
\mathcal{L}_i(\bsu, \lambda) \triangleq \mathcal{D}_{\mathrm{KL}}(P^\bsu \, \Vert \, Q^*) + \lambda (\mathcal{D}_{\mathrm{KL}}(P^{\bsu} \, \Vert \, P^{\bsu_i}) - \epsilon),
$$
where $\lambda \geq 0$ is the Lagrange multiplier. A series of saddle point problems 
$$\max_{\lambda \geq 0} \min_{\bsu \in \mathcal{U}} \mathcal{L}_i(\bsu, \lambda) 
$$
is then solved. Since $\mathcal{L}_i$ is convex in $\bsu$ and concave in $\lambda$ (see Section 2 in \cite{blessing2025trust}), this problem admits unique solutions in both variables, denoted by  $\bsu_{i+1}$ and $\lambda_{i}$, respectively.

A key insight of \cite{blessing2025trust}
 is that, for fixed $\lambda$, the inner minimization admits a dual representation that depends only on expectations under $P^{\bsu_i}$. In particular, the dual objective can be written as 
\begin{equation}
\mathcal{L}_i^{\mathrm{dual}}(\lambda)
=
-(1+\lambda)
\log \mathbb{E}_{P^{\bsu_i}}
\left[
\left(
\frac{dQ^*}{dP^{\bsu_i}}(\bsx_{-1:T})
\right)^{\frac{1}{1+\lambda}}
\right]
-
\lambda \epsilon.
\label{eq:dual_objective}
\end{equation}
 We compute $\lambda_i = \argmax_{\lambda \ge 0} \mathcal{L}_i^{\mathrm{dual}}(\lambda)$ with a one-dimensional numerical optimization routine, where the expectation is approximated using samples from $P^{\bsu_i}$.

Having determined $\lambda_i$, the next step is to characterize the updated path measure. By Proposition 2.2 in \cite{blessing2025trust},
\begin{equation}
\frac{dP^{\bsu_{i+1}}}{dP^{\bsu_i}} \propto \left(\frac{dQ^*}{dP^{\bsu_i}} \right)^{\frac{1}{1 + \lambda_{i}}},
\label{eq:uiplus1_measure_exp}
\end{equation}
which, together with \eqref{eq:radon_nikodym_optimal} and knowledge of $\lambda_i$, gives an expression for the trajectory measure $P^{\bsu_{i+1}}$ but does not give direct access to $\bsu_{i+1}$. What remains is therefore to estimate $\bsu_{i+1}$ given this expression for the trajectory measure, together with the set of samples from $P^{\bsu_i}$ previously used to estimate $\lambda_i$. This can be done by solving 
\begin{equation}
\argmin_{\bsu \in \mathcal{U}} \mathcal{D}(P^{\bsu}, P^{\bsu_{i+1}}),
\label{eq:ido_subproblem}
\end{equation}
where $\mathcal{D}$ is any statistical divergence measure; the sub-problem of computing $\bsu_{i+1}$ therefore has the same form as the original problem in \eqref{eq:abstract_prob_form}. 

In this work, we use the log-variance divergence measure introduced in \cite{nusken2021solving} as the divergence measure in \eqref{eq:ido_subproblem}. This divergence has a number of appealing theoretical properties, including robustness to Monte-Carlo gradient errors when $\bsu$ is close to the solution and stability in high-dimensional settings \cite{nusken2021solving}. For some $\bsv \in \mathcal{U}([-1, T))$, the log-variance divergence $\mathcal{D}_{\mathrm{logvar}}^{\bsv}$ between $P^{\bsu}$ and $P^{\bsu_{i+1}}$ is defined as follows:
\begin{equation}
\label{eq:logvar_diverge}
\mathcal{D}_{\mathrm{logvar}}^{\bsv}(P^{\bsu}, P^{\bsu_{i+1}}) \triangleq \mathrm{Var}_{P^{\bsv}} \left [ \log \frac{dP^{\bsu_{i+1}}}{dP^{\bsu}} \right] = \mathrm{Var}_{P^{\bsv}} \left [ \log \frac{dP^{\bsu_{i+1}}}{dP^{\bsu_i}} \frac{dP^{\bsu_{i}}}{dP^{\bsu}} \right].
\end{equation}
Here $\mathrm{Var}_{P^{\bsv}} \left [ \cdot \right]$ denotes the variance of the given quantity under the measure $P^{\bsv}$, $dP^{\bsu_i} / dP^{\bsu}$ can be computed using Girsanov's theorem, and $dP^{\bsu_{i+1}} / dP^{\bsu_i}$ is given by \eqref{eq:uiplus1_measure_exp}. While in principle $\bsv$ can be any control, we set $\bsv = \bsu_i$ to straightforwardly take advantage of the set of samples from $P^{\bsu_i}$ used to compute $\lambda_i$. This leads to a finite sample approximation of \eqref{eq:logvar_diverge}. Finally, $\bsu$ is approximated as a neural network, and gradient descent algorithms are used to solve the finite-sample approximation of \eqref{eq:logvar_diverge}. 

To summarize, the trust region approach introduced in \cite{blessing2025trust} induces a series of subproblems, as given in \eqref{eq:trust_region_subproblem}. Solving the sub-problem involves three stages. First, a set of trajectory samples corresponding to the previous iterate $\bsu_i$ is computed. Then, the multiplier $\lambda_i$ is computed from \eqref{eq:dual_objective} using a non-linear solver. Finally, the set of trajectory samples, together with $\lambda_i$, are used to estimate $\bsu_{i+1}$ in an IDO framework using the log-variance objective \eqref{eq:logvar_diverge}. This procedure is closely related to the Proximal Diffusion Neural Sampler (PDNS) of \cite{guo2025proximal}, which also replaces a one-shot path-space optimization with a sequence of local KL-regularized updates. The distinction is that PDNS formulates locality through an explicit proximal penalty and develops proximal relative-entropy and cross-entropy formulations, whereas our approach based on \cite{blessing2025trust} uses a KL trust-region constraint and recovers the intermediate control with a log-variance objective.

\section{Error Bounds and Importance Sampling}
\label{sec:samperror_importsamp}

Problem \ref{problem:path_space_sampling} formulates posterior sampling as the search for an optimal control in an infinite-dimensional function space. In practice, it is necessary to consider finite-dimensional approximations, which introduces error in the learned control. Additional error can also arise if the resulting optimization problem is not solved exactly. 

This section characterizes how control approximation error affects posterior sampling error and shows how the resulting bias can be corrected at the level of posterior expectations using path-space importance weights. Throughout, we assume that the base process \eqref{eq:base_process} samples exactly from the prior, i.e., $\bsx_T\sim\pi_{\mathrm{pr}}$.

\subsection{Sampling Error Bounds}

The following theorem characterizes the impact of error in the path-space control on the error in the corresponding posterior samples. Here, the path-space control may be obtained either by learning (as in the trust-region approach) or via guidance-based constructions such as DPS and $\Pi$GDM. The proof follows directly from Girsanov's theorem and the data-processing inequality. This proof strategy is well established in the literature (see, e.g., \cite{chen2024probabilistic}). 

\begin{theorem} \label{theorem:girsanov_bound}
    Let $\bsu^*$ be the solution to Problem \ref{problem:path_space_sampling}, i.e., $\bsu^*$ samples trajectories from the target measure $Q^*$ through the SDE
    $$
d\bsx_{t}^{\bsu^*} =\left[ \bsA(\bsx_{t}^{{\bsu^*}}, t) + \bsB(t) \, {\bsu}(\bsx_{t}^{\bsu^*}, t)\right] \, dt + \bsB(t) \, d\bsw_{t}, \quad \bsx_{-1}^{\bsu^*} = \mathbf{0}, \quad t \in [-1, T]. 
    $$
    Let $\bsu \in \mathcal{U}$ be a suboptimal control for the same process. Then the KL divergence between their time-$T$ distributions satisfies 
    $$
    \mathcal{D}_{\mathrm{KL}}(p_T^{\bsu} \, \Vert \, p_T^{\bsu^*} ) =     \mathcal{D}_{\mathrm{KL}}(p_T^{\bsu} \, \Vert \, \pi_{\mathrm{post}}) \leq  \mathbb{E}_{P^{\bsu}} \left [\int_{-1}^{T} \frac{1}{2} \left \|  \bsu(\bsx_{t}^{\bsu}, t) - \bsu^*(\bsx_{t}^{\bsu}, t) \right \|_2^2 dt \right].
    $$
\end{theorem}
\begin{proof}
The Girsanov theorem implies that the Radon-Nikodym derivative of $P^{\bsu}$ with respect to $P^{\bsu^*}$ is given by 
$$
\frac{dP^{\bsu}}{dP^{\bsu^*}} = \mathrm{\exp}\left( \int_{-1}^{T} \langle \bsu(\bsx_t, t) - \bsu^*(\bsx_{t}, t), d\bsw_t^{\bsu} \rangle + \frac{1}{2} \left \|  \bsu(\bsx_{t}, t) - \bsu^*(\bsx_{t}, t)  \right \|_2^2 dt  \right),
$$
where $d\bsw_t^{\bsu} = d\bsw_t - \bsu(\bsx_t, t) dt$ is Brownian motion under $P^{\bsu}$. 
So it follows that 
\begin{align*}
\mathcal{D}_{\mathrm{KL}}(P^{\bsu} \, \Vert \,  P^{\bsu^*}) &= \mathbb{E}_{P^{\bsu}} \Big [ \int_{-1}^{T} \langle \bsu(\bsx_t^{\bsu}, t) - \bsu^*(\bsx_t^{\bsu}, t), d\bsw_t^{\bsu} \rangle \\
& \hspace{3cm} + \int_{-1}^T \frac{1}{2} \left \|  \bsu(\bsx_t^{\bsu}, t) - \bsu^*(\bsx_t^{\bsu}, t)  \right \|_2^2 dt \Big ] \\
&=  \mathbb{E}_{P^{\bsu}} \left [\int_{-1}^{T} \frac{1}{2} \left \|  \bsu(\bsx_t^{\bsu}, t) - \bsu^*(\bsx_t^{\bsu}, t) \right \|_2^2 dt \right],
\end{align*}
where we have used the fact that$\int_{-1}^t \langle \bsu - \bsu^*, d\bsw_s^{\bsu} \rangle$ is an It\^{o} integral and therefore has expectation zero. An application of the data processing inequality then gives the desired result. 
\end{proof}

\subsection{Posterior Expectations via Path-Space Importance Sampling}

Finite dimensional control parametrizations and inexact optimization lead to errors in the samples obtained via the controlled SDE in \eqref{eq:control_reparam_dynamics}. If left uncorrected, these errors can cause bias in downstream Monte-Carlo estimators of posterior expectations. The path-space formulation yields a natural importance reweighting strategy that corrects this bias. In particular, the following theorem shows that posterior expectations can be expressed as weighted expectations under the controlled path measure $P^{\bsu}$, with weights given by the corresponding Radon--Nikodym derivative.

\begin{theorem}
\label{theorem:reweighting}
Let $\bsu$ be a suboptimal control for the controlled SDE given by \eqref{eq:control_reparam_dynamics}. Define the path-space importance weight
\begin{equation}
\label{eq:path_space_weight}
w(\bsx_{-1:T}^{\bsu})
\triangleq
\pi_{\mathrm{like}}(\bsy \mid \bsx_T^{\bsu})
\exp\!\left(
- \int_{-1}^T \langle \bsu(\bsx_t^{\bsu}, t), d\bsw_t^{\bsu} \rangle
- \frac{1}{2} \int_{-1}^T \| \bsu(\bsx_t^{\bsu}, t) \|_2^2 \, dt
\right),
\end{equation}
where $d\bsw_t^{\bsu} = d\bsw_t - \bsu(\bsx_t^{\bsu}, t)\,dt$ is Brownian motion under $P^{\bsu}$. Then for any measurable function $h$, it holds that
\begin{equation}
\label{eq:reweighting_formula}
\mathbb{E}_{\pi_{\mathrm{post}}(\bsx \mid \bsy)}[h(\bsx)]
=
Z^{-1}\,
\mathbb{E}_{P^{\bsu}}
\big[
h(\bsx_T^{\bsu})\, w(\bsx_{-1:T}^{\bsu})
\big],
\end{equation}
where
\begin{align}
\label{eq:reweighting_formula_normalizing}
Z
&\triangleq
\mathbb{E}_{P}\!\left[\pi_{\mathrm{like}}(\bsy \mid \bsx_T)\right]
=
\mathbb{E}_{\bsx \sim \pi_{\mathrm{pr}}}\!\left[\pi_{\mathrm{like}}(\bsy \mid \bsx)\right] = 
\mathbb{E}_{P^{\bsu}}
\big[
w(\bsx_{-1:T}^{\bsu})
\big].
\end{align}
\end{theorem}

\begin{proof}
By definition of the posterior,
\[
\mathbb{E}_{\pi_{\mathrm{post}}(\bsx \mid \bsy)}[h(\bsx)]
=
Z^{-1}
\mathbb{E}_{\pi_{\mathrm{pr}}(\bsx)}
\big[
h(\bsx)\pi_{\mathrm{like}}(\bsy \mid \bsx)
\big],
\]
where
\[
Z
=
\mathbb{E}_{\bsx \sim \pi_{\mathrm{pr}}}
\big[
\pi_{\mathrm{like}}(\bsy \mid \bsx)
\big]
\]
is the posterior normalizing constant. Since $\bsx_T \sim \pi_{\mathrm{pr}}(\cdot)$ under the base process, this can be rewritten as
\[
\mathbb{E}_{\pi_{\mathrm{post}}(\bsx \mid \bsy)}[h(\bsx)]
=
Z^{-1}
\mathbb{E}_{P}
\big[
h(\bsx_T)\pi_{\mathrm{like}}(\bsy \mid \bsx_T)
\big].
\]
Applying Girsanov's theorem to change measure from $P$ to $P^{\bsu}$ yields
\[
\mathbb{E}_{\pi_{\mathrm{post}}(\bsx \mid \bsy)}[h(\bsx)]
=
Z^{-1}
\mathbb{E}_{P^{\bsu}}
\big[
h(\bsx_T^{\bsu})\, w(\bsx_{-1:T}^{\bsu})
\big],
\]
which proves \eqref{eq:reweighting_formula}. Setting $h \equiv 1$ gives
\[
Z
=
\mathbb{E}_{P^{\bsu}}
\big[
w(\bsx_{-1:T}^{\bsu})
\big],
\]
completing the proof.
\end{proof}

\begin{remark}
    The above theorem implies that asymptotically exact posterior expectations can be computed using the trivial control $\bsu \equiv 0$. However, the importance weights, which in this setting are given by $\pi_{\mathrm{like}}(\bsy \mid \bsx_T^{\bsu})$, can have incredibly high variance. This makes this approach impractical for many problems and illustrates the practical necessity of using a control that approximates the true optimal control. 
\end{remark}

A practical advantage of the optimization strategy adopted in this work is that it promotes stable path-space importance weights. The log-variance objective directly targets the variability of the Radon--Nikodym weights, while the KL trust-region constraint regularizes the change of measure between successive controlled processes \cite{blessing2025trust}. In our setting, this combination serves both purposes: it provides a stable mechanism for learning approximate posterior controls and yields path-space importance weights that remain sufficiently well behaved for reweighting posterior expectations.

Theorem~\ref{theorem:reweighting} also applies beyond samplers explicitly derived from path-space optimization. Any diffusion posterior sampler that can be written as a controlled SDE of the form \eqref{eq:control_reparam_dynamics} admits the corresponding path-space weights. This includes common guidance-based methods, which can be interpreted as using approximate controls obtained by simplifying the intractable optimal-control formula in \eqref{eq:analytical_optimal_control}. For example, the Diffusion Posterior Sampling (DPS) algorithm \cite{chung2023diffusion} uses the following control to obtain posterior samples:
\begin{equation}
\label{eq:dps_control}
\bsu(\bsx_t, t) = \mathbf{B}(t)^T \nabla_{\bsx} \log \pi_{\mathrm{like}}(\bsy \mid \bar{\bsx}_T(\bsx_t)), 
\end{equation}
with $\bar{\bsx}_T(\bsx_t) \triangleq \mathbb{E}_{P}[\bsx_T \mid \bsx_t]$ computed using Tweedie’s formula \cite{efron2011tweedie}. This corresponds to replacing $p_{T\mid t}(\bsx_T\mid\bsx_t)$ in \eqref{eq:analytical_optimal_control} by a Dirac mass at its conditional mean. Similarly, $\Pi$GDM \cite{song2023pseudoinverseguided} uses a Gaussian approximation with white-noise covariance for linear-Gaussian likelihoods. Although this approximation often improves over the DPS control for linear inverse problems, it remains unimodal and can therefore be inaccurate when the prior-induced conditional distribution is multimodal, particularly for $t \ll T$. From the path-space viewpoint, both DPS and $\Pi$GDM define approximate controls, and Theorem~\ref{theorem:reweighting} provides a principled mechanism for correcting the resulting bias in posterior expectations.

\section{Numerical Studies}
\label{sec:num_studies}

We evaluate the proposed path-space posterior sampling framework on a suite of diagnostic BIPs with reliable reference posteriors. In contrast to much of the existing literature on diffusion-based posterior sampling, where evaluation often relies on perceptual metrics or qualitative sample diversity, our experiments enable principled quantitative validation against exact posterior distributions in the linear--Gaussian settings and against high-quality long-run MCMC reference samples in the nonlinear settings.

We follow the benchmark design of \cite{scopecrafts2025benchmarking} and compare our framework against three leading diffusion-based posterior sampling methods: diffusion posterior sampling (DPS) \cite{chung2023diffusion}, pseudo-inverse guided diffusion models ($\Pi$GDM) \cite{song2023pseudoinverseguided}, and decoupled annealing posterior sampling (DAPS) \cite{zhang2025improving}. Note that throughout all experiments, we deliberately construct the prior-sampling base SDE so that the initial value function bias problem is minimal. The focus of this section is therefore restricted to rigorously assessing whether path-space control learning can outperform leading approaches that don't require additional training beyond what is required to obtain the base SDE.  

\subsection{Benchmark Inverse Problems}\label{sec:benchmarks}
All benchmarks infer an unknown parameter $\bsx \in \mathbb{R}^D$ from observations $\bsy \in \mathbb{R}^K$ generated by a forward model and noise process. Throughout, we use a \emph{canonical} Gaussian mixture prior
\begin{equation}
	\label{eq:gmm_prior}
	\pi_{\mathrm{pr}}(\bsx) \;=\; \sum_{i=1}^{N_m} w_i \, \mathcal{N}(\bsx;\boldsymbol{\mu}_i,\boldsymbol{\Sigma}_i),
	\qquad
	w_i \ge 0,\ \sum_{i=1}^{N_m} w_i = 1,
\end{equation}
with $N_m=3$. The mixture parameters $\{w_i,\boldsymbol{\mu}_i,\boldsymbol{\Sigma}_i\}_{i=1}^{N_m}$ are chosen to induce multi-modal posteriors.

We consider four benchmark inverse problems spanning linear and nonlinear forward models, partial observations, and ill-posed measurement operators. 

\paragraph{Forming the prior-sampling SDE}

In our formulation, we assume access to a pre-trained SDE that samples from the prior distribution. For the Gaussian mixture priors used in our benchmarks, this SDE can be specified in closed form: under common Gaussian transition kernels, the Gaussian-mixture structure is preserved, so the time-dependent scores $\nabla_{\bsx}\log p_t(\bsx_t)$ in \eqref{eq:backward_diffusion_sde} admit exact evaluation. As a result, these experiments isolate algorithmic error in posterior sampling from approximation error in prior learning. Explicit formulas for the noisy prior marginals and their scores are given in Appendix~\ref{app:closed_form_scores}.

Because DPS, $\Pi$GDM, and DAPS do not use the TRT introduced in this work, we choose benchmark priors for which the initial value function bias is minimal. In particular, we use long time horizons so that $p_{0,T}(\bsx_0,\bsx_T)\approx p_0(\bsx_0)\,p_T(\bsx_T)$, making the base process approximately memoryless and enabling a fair comparison between path-space control learning and existing diffusion-based posterior samplers. Implementation details regarding the SDEs are given in Appendix~\ref{app:benchmark_details}.

\paragraph{Reference posteriors and optimal controls}
For the two linear--Gaussian benchmarks (Sections~\ref{prob:random-sensing}--\ref{prob:inpainting}), the posterior $\pi_{\mathrm{post}}(\bsx\mid\bsy)$ is available in closed form as a Gaussian mixture, and we additionally derive the optimal path-space control in closed form. These expressions provide exact reference solutions for validating both sampled posteriors and learned controls; details are deferred to Appendix~\ref{app:linear_posteriors_controls}. For the nonlinear benchmarks (Sections~\ref{prob:xray}--\ref{prob:phase}), we use the high-quality MCMC reference samples released with \cite{scopecrafts2025benchmarking}.

We now describe the four benchmark inverse problems considered in this work. Their main features are summarized in Table~\ref{tab:benchmarks}, and full parameter specifications are given in Appendix~\ref{app:benchmark_details}.

\begin{table}[t]
	\centering
	\caption{Summary of benchmark inverse problems. Complete parameter choices are provided in Appendix~\ref{app:benchmark_details}.}
	\label{tab:benchmarks}
	\begin{tabular}{lccc}
		\hline
		Problem & $(D,K)$ & Forward model $\mathcal{F}(\bsx)$ & Likelihood \\ \hline
		Random linear sensing & $(20,20)$ & $\mathbf{H}\bsx$ & $\mathcal{N}(\mathbf{H}\bsx,\boldsymbol{\Gamma})$ \\
		Inpainting & $(10,8)$ & $\mathbf{M}\bsx$ & $\mathcal{N}(\mathbf{M}\bsx,\sigma^2\mathbf{I})$ \\
		X-ray tomography & $(10,15)$ & $I_0\exp(-\mathbf{C}\bsx)$ & $\mathrm{Poi}(I_0\exp(-\mathbf{C}\bsx))$ \\
		Phase retrieval & $(10,5)$ & $(\mathbf{N}\bsx)^{\odot 2}$ & $\mathcal{N}((\mathbf{N}\bsx)^{\odot 2},\sigma^2\mathbf{I})$ \\
		\hline
	\end{tabular}
\end{table}

\subsubsection{Random linear sensing problem} \label{prob:random-sensing}
We consider the linear-Gaussian observation model
\begin{equation}
	\bsy \;=\; \mathbf{H}\bsx + \boldsymbol{\eta}, 
	\qquad 
	\boldsymbol{\eta} \sim \mathcal{N}(\mathbf{0},\boldsymbol{\Gamma}),
\end{equation}
where $\mathbf{H}\in\mathbb{R}^{K\times D}$ is a fixed realization with i.i.d.\ $\mathcal{N}(0,1)$ entries and $\boldsymbol{\Gamma}\in\mathbb{R}^{K\times K}$ is a diagonal (heteroscedastic) noise covariance. We set $D=K=20$. 

With the Gaussian mixture prior \eqref{eq:gmm_prior}, the posterior remains a Gaussian mixture, so both exact reference samples and the optimal path-space control are available in closed form; see Appendices~\ref{app:closed_form_scores} and~\ref{app:linear_posteriors_controls}.

\subsubsection{Inpainting problem} \label{prob:inpainting}
This benchmark models partial observations with additive Gaussian noise,
\begin{equation}
	\bsy \;=\; \mathbf{M}\bsx + \boldsymbol{\eta},
	\qquad
	\boldsymbol{\eta} \sim \mathcal{N}(\mathbf{0},\sigma^2\mathbf{I}_K),
\end{equation}
where $\mathbf{M}\in\{0,1\}^{K\times D}$ is a selection matrix that extracts a subset of coordinates of $\bsx$.
We set $D=10$ and observe $K=8$ coordinates; the observed index set and the noise level $\sigma$ are specified in Appendix~\ref{app:benchmark_details}. As in the random linear sensing problem, the posterior and optimal control are available in closed form; see Appendix~\ref{app:linear_posteriors_controls}. In contrast to the random sensing benchmark, this problem probes the ability of a sampler to recover posterior uncertainty along unobserved directions of the parameter space.

\subsubsection{X-ray tomography problem}\label{prob:xray}
We next consider a nonlinear transmission tomography model with Poisson noise. Following \cite{scopecrafts2025benchmarking}, the forward intensity map is
\begin{equation}
	f(\bsx) \;=\; I_0 \exp(-\mathbf{C}\bsx),
\end{equation}
with exponential applied componentwise, and the likelihood is
\begin{equation}
	\pi_{\mathrm{like}}(\bsy\mid\bsx) \;=\; \mathrm{Poi}(f(\bsx)).
\end{equation}
Here $\mathbf{C}\in\mathbb{R}^{K\times D}$ and $I_0$ are fixed problem parameters (Appendix~\ref{app:benchmark_details}). We set $(D, K) = (10,15)$ and use the same Gaussian mixture prior as in the inpainting benchmark. Since the resulting posterior is non-Gaussian and analytically intractable, we validate against the high-quality MCMC reference samples from \cite{scopecrafts2025benchmarking}, generated using MT-DREAM(ZS) \cite{laloy2012high}.

\subsubsection{Phase retrieval problem}\label{prob:phase}
Finally, we consider a nonlinear and underdetermined phase retrieval model with additive Gaussian noise,
\begin{equation}
	\bsy \;=\; (\bsN\bsx)^{\odot 2} + \boldsymbol{\eta},
	\qquad
	\boldsymbol{\eta} \sim \mathcal{N}(\mathbf{0},\sigma^2\mathbf{I}_K),
\end{equation}
where $(\cdot)^{\odot 2}$ denotes elementwise squaring and $\bsN\in\mathbb{R}^{K\times D}$ has i.i.d.\ $\mathcal{N}(0,1)$ entries. We set $(D,K)=(10,5)$ with a noise level $\sigma$ chosen to yield a challenging low-SNR regime (Appendix~\ref{app:benchmark_details}). As in the tomography case, we evaluate against the reference MCMC posterior samples from \cite{scopecrafts2025benchmarking}.

\subsection{Algorithm Implementations} In this subsection, we provide details regarding the implementation of the trust-region approach. The implementation of the DPS,  $\Pi$GDM, and DAPS algorithms used for comparison is detailed in Appendix \ref{app:comp_alg_imp}. 

We implement the trust-region optimization (TR) method described in Section~\ref{sec:solution_approach}, with the log-variance divergence $\mathcal{D}_{\mathrm{logvar}}$ in \eqref{eq:logvar_diverge} together with the Radon--Nikodym representation \eqref{eq:radon_nikodym_optimal} used to solve the sub-problem given by \eqref{eq:ido_subproblem}. 

The control $\bsu \in \mathcal U([-1,T))$ is parameterized using a time-dependent function $\bsu_\theta(\bsx,t)$ represented by a feedforward neural network. The network takes as input the state $\bsx \in \mathbb R^D$ and time $t$. Time is encoded using a fixed sinusoidal embedding of dimension $D$ followed by a learned linear projection \cite{ho_2020_denoising}. The encoded timestep is concatenated with $\bsx$ and mapped to a hidden dimension of $D_{\mathrm{width}} = 100$, with a Softplus activation function, denoted $\phi(\cdot)$, applied. This is followed by four linear layers with Softplus activations. We then apply an intermediate mixing layer that concatenates the network state with the time embedding, uses a linear layer to map back to the hidden dimension, and applies a gated nonlinearity of the form $\phi(h)\odot h$, which is structurally similar to self-gated activations like Swish \cite{ramachandran2017searching}. Next, four additional linear layers are applied, with each one again followed by a Softplus activation. Finally, the network state is again concatenated with the time embedding and mapped back to $\mathbb R^D$ using another linear layer. 

At each outer iteration, a replay buffer is constructed containing $N_{\mathrm{buffer}}$ full trajectories, Brownian noise increments, the control values used during sampling, and the corresponding importance-weight functionals $w$ appearing in \eqref{eq:path_space_weight}. All buffer quantities are detached from the computational graph, enabling off-policy optimization. For a fixed trust-region multiplier $\lambda_i$, the control parameters $\theta$ are updated by minimizing a discrete-time Monte Carlo approximation of the trust-region log-variance objective in \eqref{eq:logvar_diverge}.

The trust-region multiplier $\lambda_i$ is updated at each outer iteration by solving the associated one-dimensional dual problem \eqref{eq:dual_objective} for the KL constraint using scalar line search. We consider two trust-region radii, $\epsilon \in \{0.1, 0.01\}$, in order to compare more aggressive and more conservative control updates; an interesting direction for future work is the use of adaptive trust-region strategies \cite{wright1999numerical}. The maximum number of outer-loop iterations in \eqref{eq:trust_region_subproblem} is set to $I = 300$. Optimization is terminated early when $\lambda_i < 0.1$. In the experiments below, we record the number of completed outer iterations and report the frequency of early termination in Appendix~\ref{app:optimization_details}.

The network parameters are optimized using Adam \cite{kingma2014adam} with learning rate $5\times10^{-4}$ and gradient clipping with a value of $1$. Training is performed for a maximum of $1.2\times 10^5$ gradient steps, using mini-batches of size $2000$ drawn from the replay buffer. The size of the replay buffer is set as $N_{\mathrm{buffer}} = 50,000$.

\subsection{Evaluation Methods}
\label{sec:evaluation_methods}

We evaluate posterior sampling accuracy and uncertainty quantification across the benchmark inverse problems described above. We consider the trust-region (TR) path-space sampler alongside DPS, $\Pi$GDM, and DAPS. Since in general $\Pi$GDM is applicable only to linear inverse problems, we report it only for the random linear sensing and inpainting benchmarks. For the linear--Gaussian benchmarks, reference samples are drawn from the exact posterior distribution, and the closed-form optimal control is available, enabling direct assessment of both posterior accuracy and control approximation error. For the nonlinear benchmarks, we use the high-quality MCMC reference samples released in \cite{scopecrafts2025benchmarking}. Computational efficiency is analyzed separately in Section~\ref{sec:computational_efficiency}.

Each experiment consists of $10$ independent trials. In the $i$th trial, the measurement $\bsy_i$ is sampled independently from the measurement distribution, and each method generates $10{,}000$ samples from the corresponding posterior $\pi_{\mathrm{post}}(\bsx\mid\bsy_i)$. Metrics are computed per trial and reported as mean $\pm$ standard deviation across trials.

We assess posterior accuracy using posterior mean error, covariance error, maximum mean discrepancy (MMD), and central moment discrepancy (CMD). For the linear--Gaussian benchmarks, we also report control approximation error, comparing each method's induced control with the analytically available optimal control. This diagnostic is motivated by Theorem~\ref{theorem:girsanov_bound}, which relates posterior sampling error to path-integrated control mismatch. In our experiments, all controls are evaluated on common reference trajectories generated under the optimal control. Detailed metric definitions are given in Appendix~\ref{app:evaluation_metrics}.

For samplers with tractable path-space change-of-measure representations, we additionally report importance-weighted diagnostics: normalized effective sample size (NESS), self-normalized importance-sampled posterior mean error, and the percentage reduction in posterior mean error due to reweighting. These diagnostics are reported for DPS, $\Pi$GDM, and the proposed trust-region method, but not for DAPS, which does not define a single controlled diffusion path measure with tractable Radon--Nikodym weights of the form in Theorem~\ref{theorem:reweighting}.

Finally, we include qualitative diagnostics based on two-dimensional projections of generated and reference samples onto the eigenvectors of the reference posterior covariance associated with its smallest and largest eigenvalues. These projections visualize posterior geometry along the most concentrated and most diffuse directions.

\subsection{Computational efficiency analysis}
\label{sec:computational_efficiency}
To analyze the computational efficiency of the proposed approach, we focus on two dominant costs: the evaluation of the likelihood function and its score, and the evaluation of the base diffusion drift. Here the distinct training and inference stages of the proposed approach naturally suggest an amortized analysis framework. We therefore report the number of evaluations of the likelihood function and/or its score, and the number of evaluations of the drift term of the base diffusion model, in three different regimes: training, per-sample inference, and the total cost required to obtain $10^4$ posterior samples. 

We compare all implemented posterior samplers using these metrics. Only the proposed method requires additional neural-network training beyond the base SDE, so the training cost of DPS, $\Pi$GDM, and DAPS is zero. The reported costs are implementation-dependent; for example, the number of drift evaluations in DPS and $\Pi$GDM depends on the SDE discretization. For DAPS, we report the more computationally efficient variant used for the linear inverse problems.

Table \ref{table:comp_effec} displays the results of the computational efficiency analysis. In contrast to the posterior sampling results reported in Section~\ref{sec:results}, these quantities are not specific to any one benchmark problem, but instead reflect the algorithmic and implementation-level costs of the methods considered in this work. 

The proposed method incurs a substantial upfront training cost and, even amortized over $10^4$ posterior samples, requires substantially more diffusion drift evaluations than DPS, $\Pi$GDM, and DAPS. However, its amortized number of likelihood evaluations is smaller than that of DAPS. This cost profile is most favorable when likelihood evaluations dominate the computational budget or when a very large number (e.g., $> 10^6$) of posterior samples is required.

\begin{table}[ht]
    \centering
    \small 
    \setlength{\tabcolsep}{3pt} 
    \begin{tabular*}{\textwidth}{@{\extracolsep{\fill}} c ccc ccc @{}}
        \toprule
        & \multicolumn{3}{c}{\textbf{Diffusion Drift Evaluations}} & \multicolumn{3}{c}{\textbf{Likelihood Evaluations}} \\
        \cmidrule(lr){2-4} \cmidrule(lr){5-7} 
        \textbf{Method} & Train & Per-Sample & Total\textsuperscript{*} & Train & Per-Sample & Total\textsuperscript{*}  \\
        \midrule
        DAPS & $0$ & $1000$ & $10^7$ & $0$ & $2 \times 10^4$ & $2 \times 10^8$\\
        DPS & $0$ & $100$ & $10^6$ & $0$ & $100$ & $10^6$ \\
        $\Pi$GDM & $0$ & $100$ & $10^6$ & $0$ & $100$ & $10^6$  \\
         TR & $1.5 \times 10^9$ & $100$ & $\approx 1.5 \times 10^9$  & $1.5 \times 10^7$ & $0$ & $1.5 \times 10^7$ \\
        \bottomrule
    \end{tabular*}
    \begin{flushleft}
        \footnotesize \textsuperscript{*}`Total' $=$ total training and sampling costs to obtain $10^4$ posterior samples.
    \end{flushleft}
            \caption{The number of evaluations of the likelihood function $\pi_{\mathrm{like}}(\bsy \mid \bsx)$ and drift term $\bsA$ of the base SDE for different diffusion posterior sampling algorithms. Note that results are specific to the particular implementation of the base process and algorithms used in this paper. }
    \label{table:comp_effec}
\end{table}

\section{Results}
\label{sec:results}

We report quantitative and qualitative results across the four benchmark inverse problems. We also examine the role of the trust-region formulation in stabilizing optimization of the path-space objective; the corresponding ablation study is reported in Appendix~\ref{app:ablation}.

\subsection{Random Linear Sensing Problem}

\begin{table}[h]
\centering
\small
\resizebox{\textwidth}{!}{%
\begin{tabular}{llcccc}
\toprule
Problem & Method 
& {\makecell{Mean\\Error $\downarrow$}} 
& {\makecell{Cov\\Error $\downarrow$}} 
& {\makecell{MMD $\downarrow$}} 
& {\makecell{CMD $\downarrow$}} \\
\midrule
\multirow{5}{*}{\makecell[l]{Random\\linear\\sensing}}
& DPS 
& $11.9 \pm 1.8$ 
& $4.74 \pm 0.55$ 
& $0.525 \pm 0.094$ 
& $1.57 \pm 0.23$ \\

& $\Pi$GDM 
& $2.57 \pm 1.35$ 
& $1.55 \pm 0.92$ 
& $0.016 \pm 0.008$ 
& $0.694 \pm 0.252$ \\

& DAPS 
& $5.01 \pm 1.49$ 
& $4.96 \pm 0.44$ 
& $0.127 \pm 0.040$ 
& $1.42 \pm 0.22$ \\

& TR ($\epsilon=0.01$) 
& $0.248 \pm 0.110$ 
& $0.647 \pm 0.050$ 
& $0.0037 \pm 0.0007$ 
& $0.074 \pm 0.025$ \\

& TR ($\epsilon=0.1$) 
& $\mathbf{0.239 \pm 0.091}$ 
& $\mathbf{0.644 \pm 0.050}$ 
& $\mathbf{0.0037 \pm 0.0006}$ 
& $\mathbf{0.074 \pm 0.021}$ \\

\midrule

\multirow{5}{*}{Inpainting}
& DPS 
& $2.77 \pm 0.72$ 
& $1.71 \pm 0.95$ 
& $0.266 \pm 0.058$ 
& $0.463 \pm 0.159$ \\

& $\Pi$GDM 
& $2.51 \pm 1.88$ 
& $1.35 \pm 1.09$ 
& $0.055 \pm 0.038$ 
& $0.708 \pm 0.320$ \\

& DAPS 
& $7.77 \pm 5.18$ 
& $3.95 \pm 1.01$ 
& $0.415 \pm 0.188$ 
& $1.94 \pm 0.61$ \\

& TR ($\epsilon=0.01$) 
& $0.585 \pm 0.381$ 
& $0.661 \pm 0.684$ 
& $0.0078 \pm 0.0028$ 
& $0.121 \pm 0.064$ \\

& TR ($\epsilon=0.1$) 
& $\mathbf{0.513 \pm 0.431}$ 
& $\mathbf{0.436 \pm 0.067}$ 
& $\mathbf{0.0070 \pm 0.0036}$ 
& $\mathbf{0.091 \pm 0.061}$ \\
\bottomrule
\end{tabular}
}
\caption{Quantitative posterior sampling results for the two linear--Gaussian benchmark problems. Entries report mean $\pm$ standard deviation across trials. Lower values are better for all metrics. The best result within each benchmark and metric is bolded.}
\label{tab:linear_posterior_metrics}
\end{table}

\begin{table}[h]
\centering
\small
\resizebox{\textwidth}{!}{%
\begin{tabular}{llcccc}
\toprule
Problem & Method 
& {\makecell{Control\\Error $\downarrow$}} 
& {\makecell{NESS $\uparrow$}} 
& {\makecell{Reweighted\\Mean Error $\downarrow$}} 
& {\makecell{Percent Reduction in\\Mean Error $\uparrow$}} \\
\midrule
\multirow{5}{*}{\makecell[l]{Random\\linear\\sensing}}
& DPS 
& $0.926 \pm 0.489$ 
& $0.0002 \pm 0.0001$ 
& $11.1 \pm 2.3$ 
& $7.4 \pm 11.1$ \\

& $\Pi$GDM 
& $0.015 \pm 0.005$ 
& $0.596 \pm 0.154$ 
& $0.184 \pm 0.054$ 
& $\mathbf{82.8 \pm 29.2}$ \\

& DAPS 
& -- 
& -- 
& -- 
& -- \\

& TR ($\epsilon=0.01$) 
& $0.0046 \pm 0.0013$ 
& $0.811 \pm 0.038$ 
& $0.187 \pm 0.070$ 
& $11.0 \pm 48.9$ \\

& TR ($\epsilon=0.1$) 
& $\mathbf{0.0044 \pm 0.0012}$ 
& $\mathbf{0.812 \pm 0.040}$ 
& $\mathbf{0.183 \pm 0.067}$ 
& $13.1 \pm 47.9$ \\

\midrule

\multirow{5}{*}{Inpainting}
& DPS 
& $0.335 \pm 0.102$ 
& $0.0022 \pm 0.0018$ 
& $3.53 \pm 1.98$ 
& $-27.5 \pm 65.3$ \\

& $\Pi$GDM 
& $0.021 \pm 0.005$ 
& $0.535 \pm 0.157$ 
& $0.206 \pm 0.145$ 
& $\mathbf{80.9 \pm 26.3}$ \\

& DAPS 
& -- 
& -- 
& -- 
& -- \\

& TR ($\epsilon=0.01$) 
& $0.0055 \pm 0.0028$ 
& $0.804 \pm 0.072$ 
& $0.191 \pm 0.143$ 
& $56.8 \pm 38.0$ \\

& TR ($\epsilon=0.1$) 
& $\mathbf{0.0052 \pm 0.0013}$ 
& $\mathbf{0.807 \pm 0.070}$ 
& $\mathbf{0.162 \pm 0.125}$ 
& $42.9 \pm 39.6$ \\
\bottomrule
\end{tabular}
}
\caption{Control and importance sampling diagnostics for the two linear--Gaussian benchmark problems. Entries report mean $\pm$ standard deviation across trials. Lower values are better for control error and reweighted mean error; higher values are better for NESS and percent reduction in mean error. The best result within each benchmark and metric is bolded.}
\label{tab:linear_is_metrics}
\end{table}

Tables~\ref{tab:linear_posterior_metrics} and~\ref{tab:linear_is_metrics} report the results for the random linear sensing benchmark. The trust-region method outperforms DPS, $\Pi$GDM, and DAPS across all posterior accuracy metrics, reducing mean error by more than an order of magnitude relative to DAPS and $\Pi$GDM and by nearly two orders of magnitude relative to DPS. The diagnostics show that TR also attains the lowest control error and highest NESS, indicating more accurate approximation of the optimal path-space control and more stable importance weights. Importance reweighting improves posterior mean estimates for all methods with tractable weights, with the largest improvement observed for $\Pi$GDM.

The qualitative results in Figure~\ref{fig:random_linear_sensing_qualitative} are consistent with the quantitative metrics. The figure shows two-dimensional projections of the generated samples for a representative trial, together with contour lines of the projected ground-truth posterior density, exact posterior samples, and the corresponding prior distribution. In this random linear sensing benchmark, the posterior is available in closed form as a Gaussian mixture. Relative to DPS, $\Pi$GDM, and DAPS, the trust-region method yields samples that more closely follow the geometry of the target posterior, with improved recovery of the dominant modes and their anisotropic spread. By contrast, DPS exhibits mode collapse, $\Pi$GDM remains overdispersed, and DAPS is overly diffuse.

\begin{figure}
\centering

\begin{subfigure}[t]{0.24\textwidth}
    \centering
    \includegraphics[width=\linewidth]{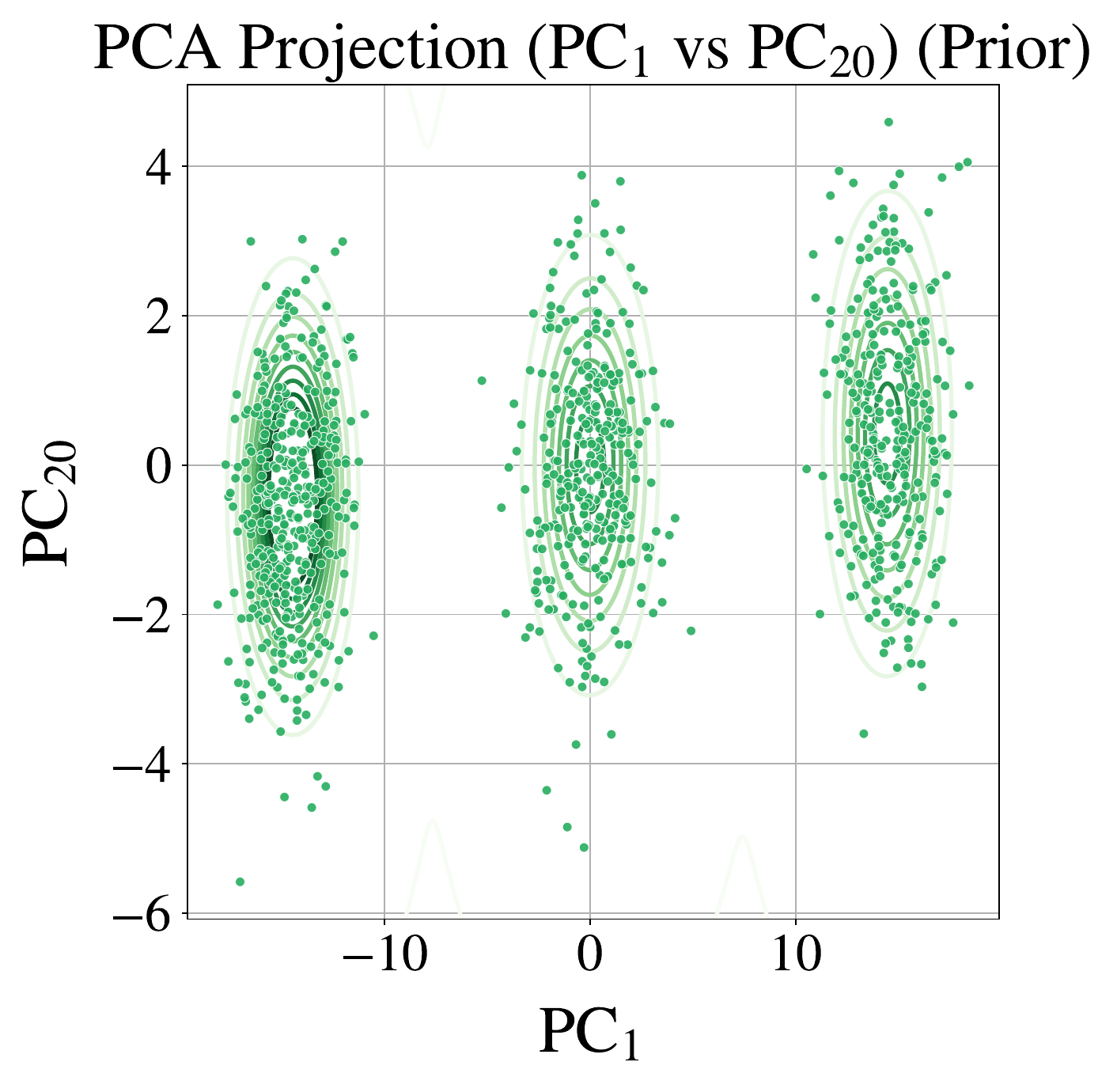}
    \caption{Prior}
\end{subfigure}
\hspace{1cm}
\begin{subfigure}[t]{0.24\textwidth}
    \centering
    \includegraphics[width=\linewidth]{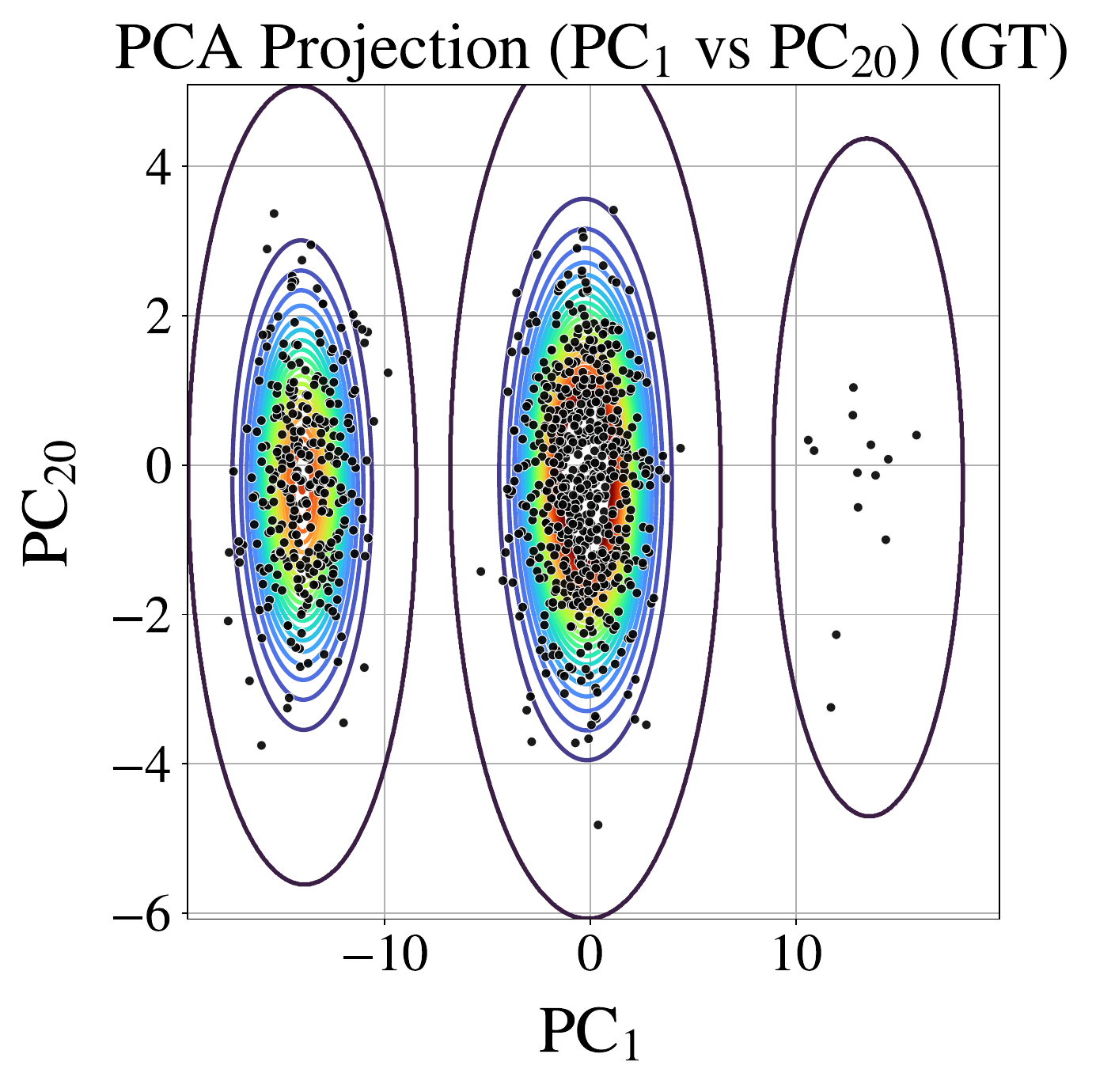}
    \caption{GT posterior}
\end{subfigure}

\vspace{0.1cm}

\begin{subfigure}[t]{0.24\textwidth}
    \centering
    \includegraphics[width=\linewidth]{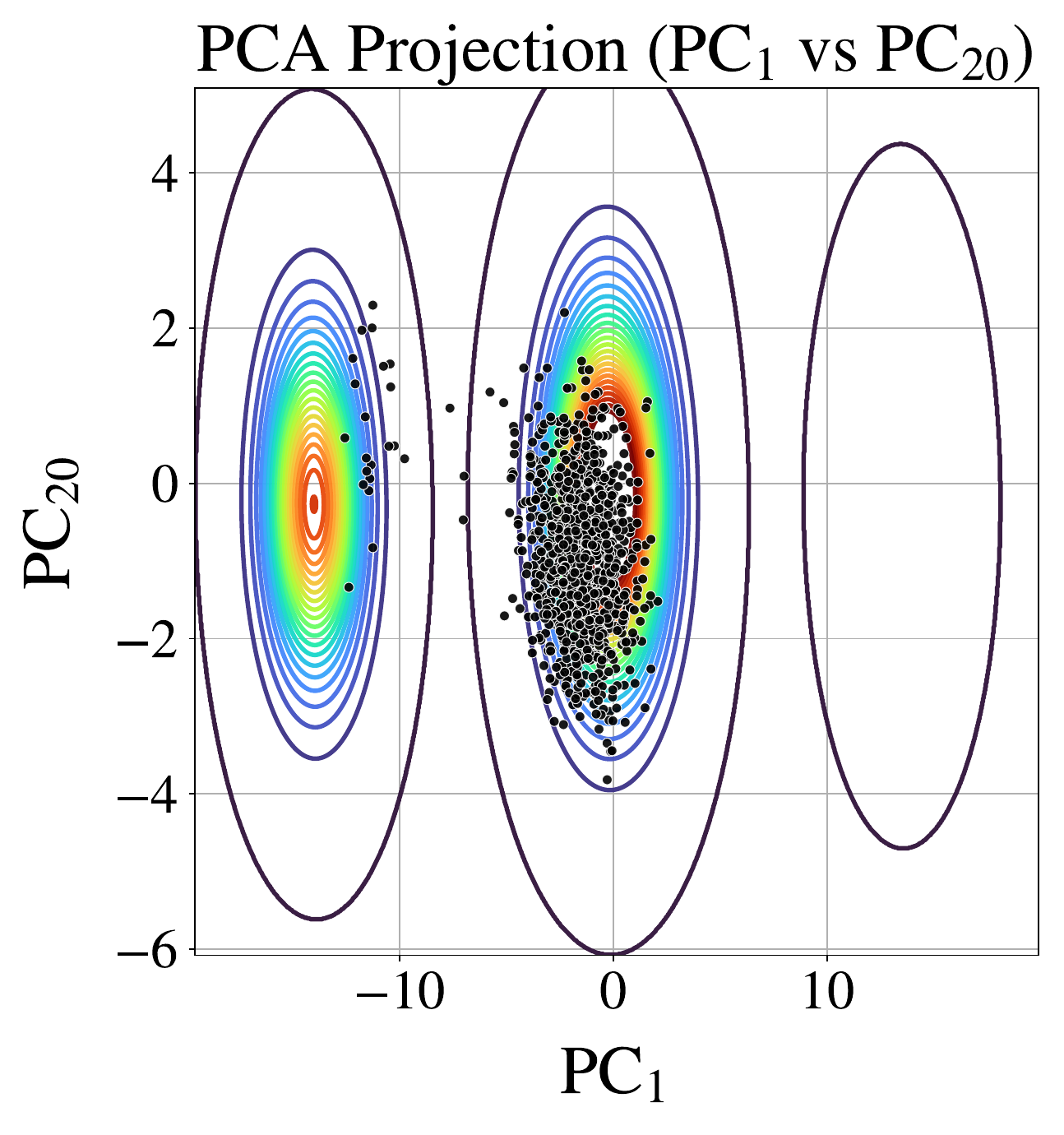}
    \caption{DPS}
\end{subfigure}
\hfill
\begin{subfigure}[t]{0.24\textwidth}
    \centering
    \includegraphics[width=\linewidth]{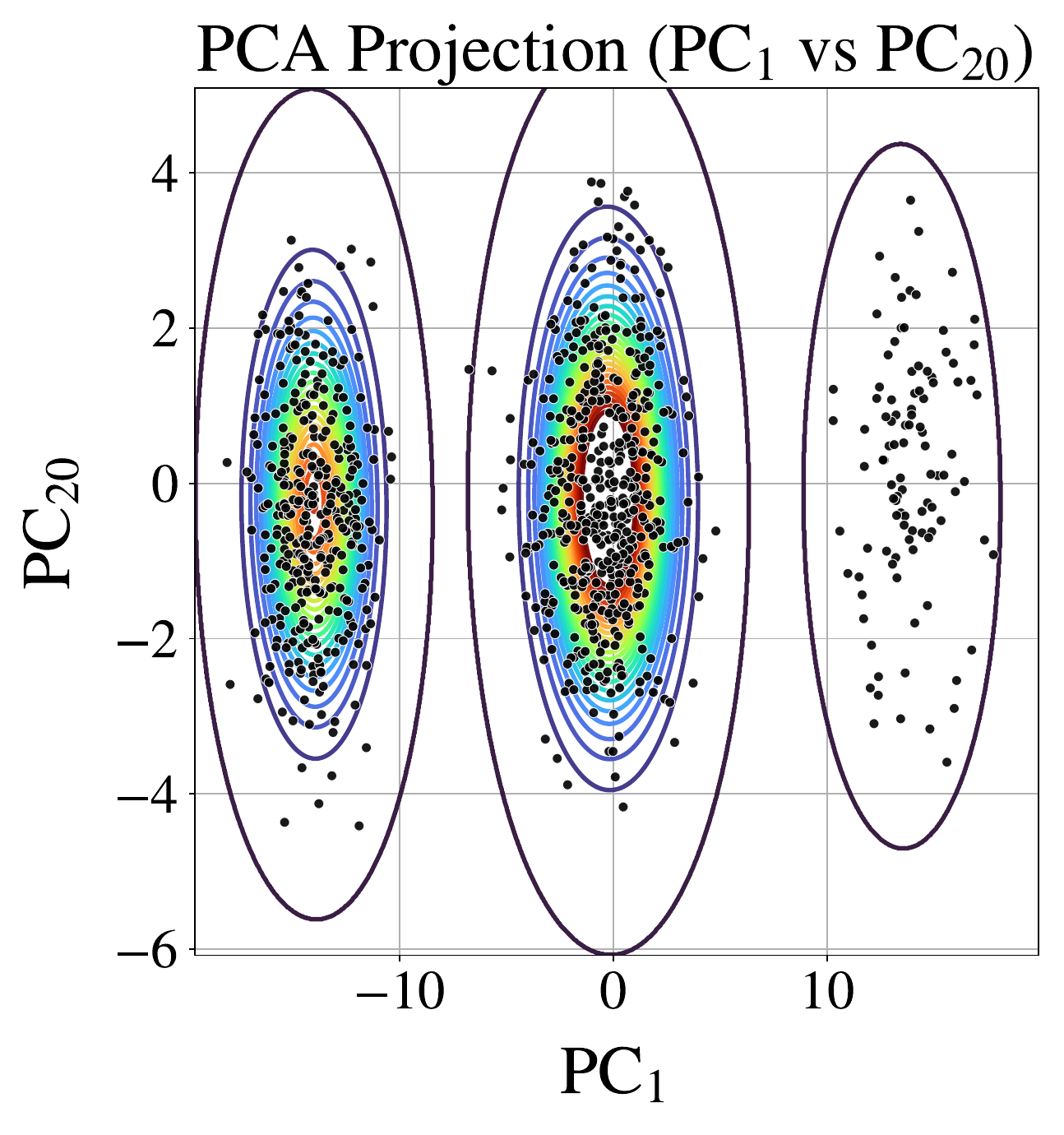}
    \caption{$\Pi$GDM}
\end{subfigure}
\hfill
\begin{subfigure}[t]{0.24\textwidth}
    \centering
    \includegraphics[width=\linewidth]{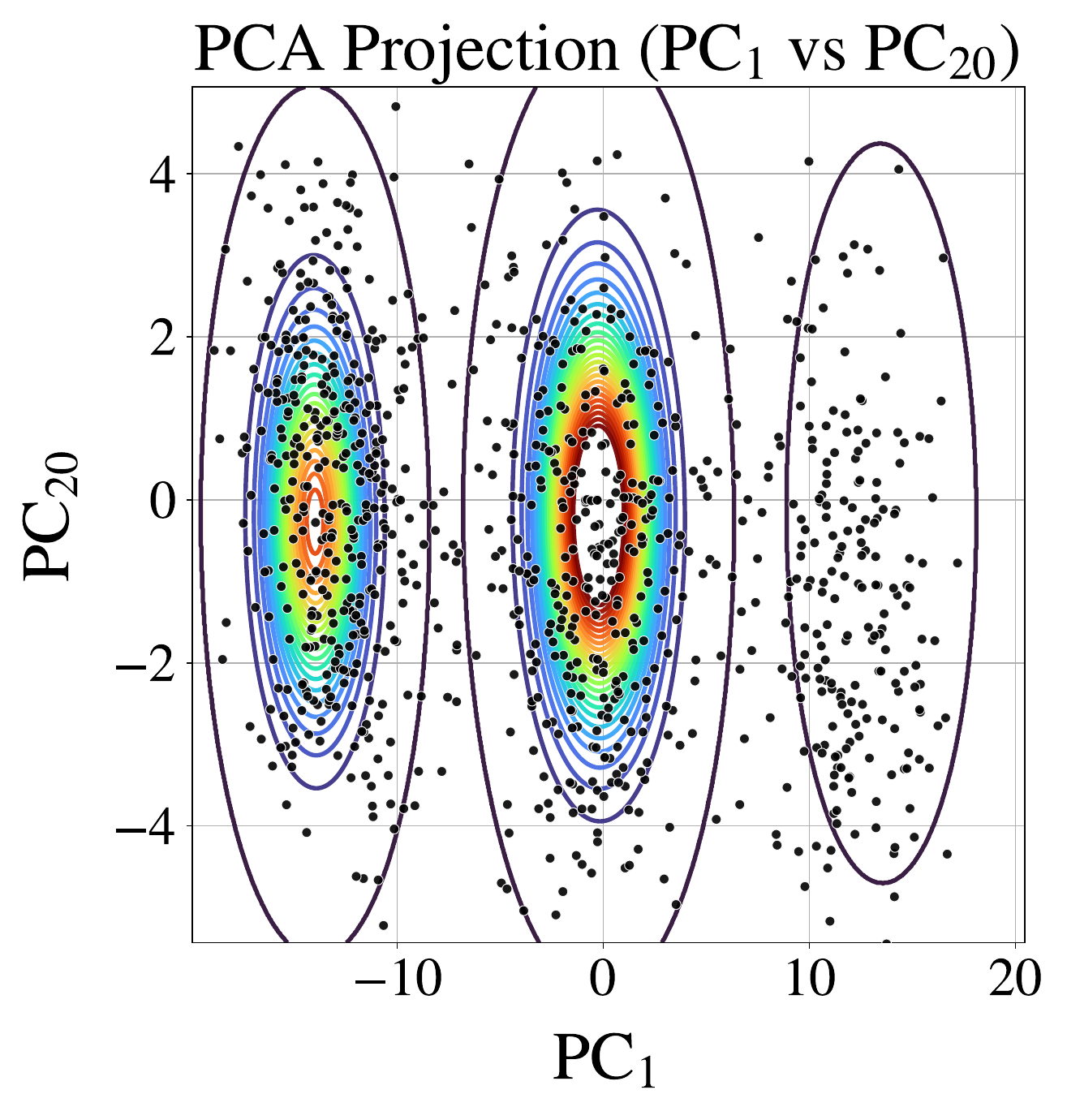}
    \caption{DAPS}
\end{subfigure}
\hfill
\begin{subfigure}[t]{0.24\textwidth}
    \centering
    \includegraphics[width=\linewidth]{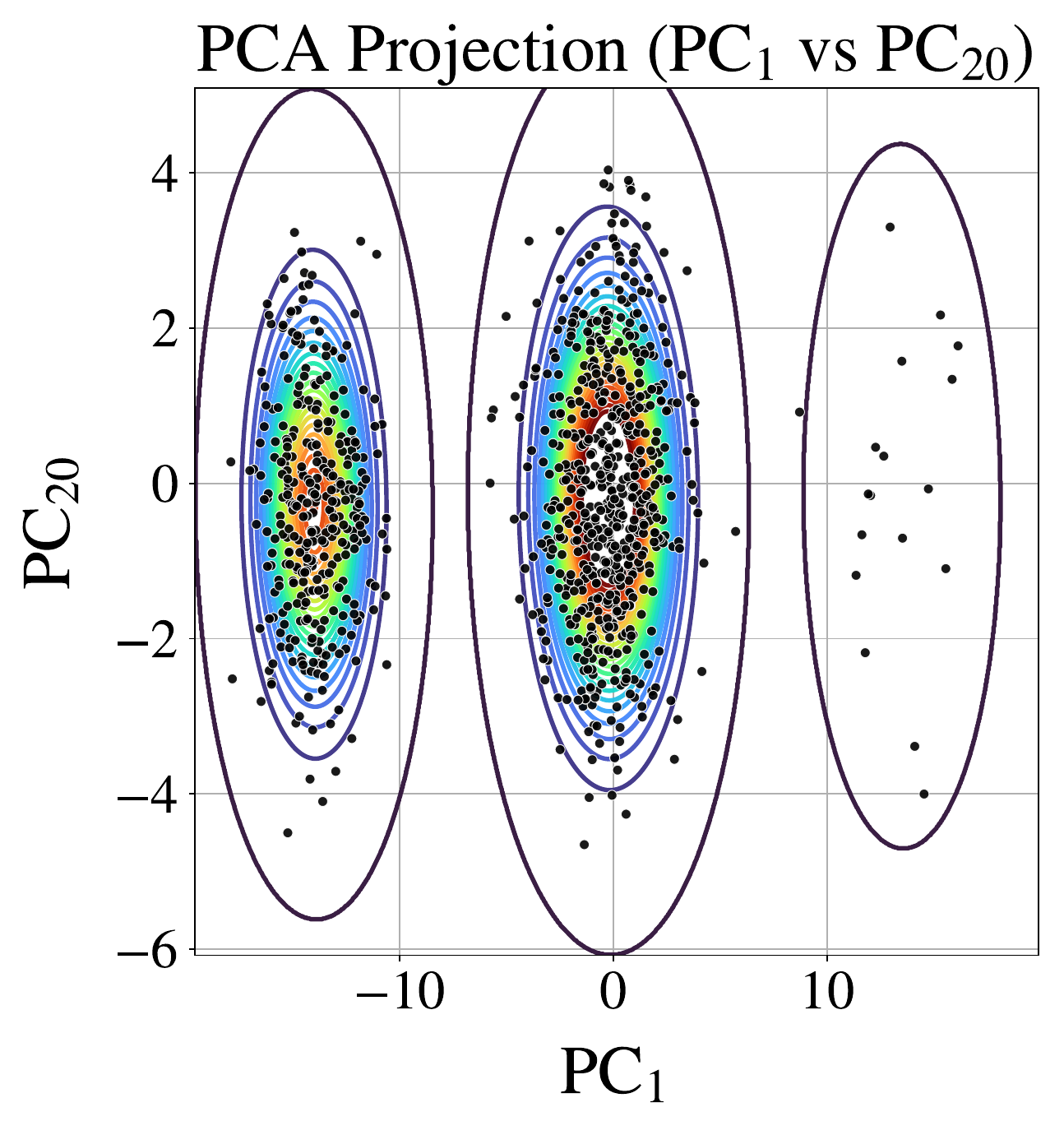}
    \caption{TR ($\epsilon = 0.1$)}
\end{subfigure}

\caption{
Qualitative comparison of posterior samples for the \emph{random linear sensing} benchmark. Top row: projections of the prior and ground-truth posterior, with contour lines showing the projected densities. Bottom row: projected samples from each method, shown against the same ground-truth posterior contours. All samples are projected onto the two-dimensional subspace spanned by the eigenvectors of the reference posterior covariance corresponding to its smallest and largest eigenvalues.
}
\label{fig:random_linear_sensing_qualitative}

\end{figure}

\subsection{Inpainting problem}
The corresponding results for the inpainting benchmark are shown in Tables~\ref{tab:linear_posterior_metrics} and \ref{tab:linear_is_metrics}. This problem tests posterior recovery in directions that are not directly constrained by the observations, requiring accurate modeling of the conditional structure induced by partial observation. Both trust-region configurations outperform DPS, DAPS, and $\Pi$GDM across all reported metrics.

The importance-sampling diagnostics in Table~\ref{tab:linear_is_metrics} show the same pattern: TR has the smallest control error and largest NESS, while DPS exhibits severe weight degeneracy. Reweighting substantially improves posterior mean estimates for $\Pi$GDM and TR, but not for DPS.

The qualitative projections in Figure~\ref{fig:inpainting_qualitative} show that TR best captures both the concentration and spread of the reference posterior, whereas competing methods either distort the posterior geometry or remain overly diffuse.

\begin{figure}
\centering

\begin{subfigure}[t]{0.24\textwidth}
    \centering
    \includegraphics[width=\linewidth]{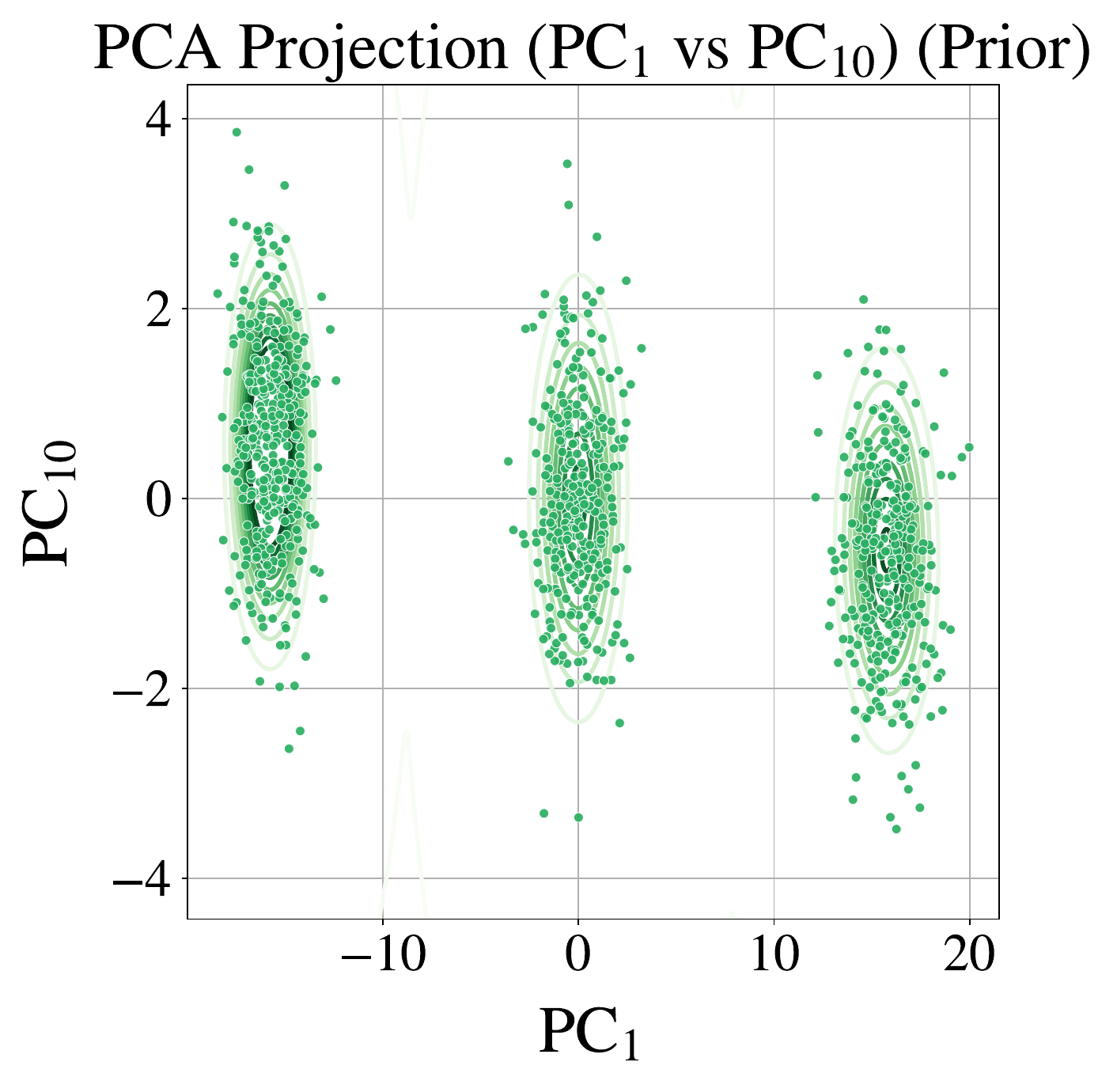}
    \caption{Prior}
\end{subfigure}
\hspace{1cm}
\begin{subfigure}[t]{0.24\textwidth}
    \centering
    \includegraphics[width=\linewidth]{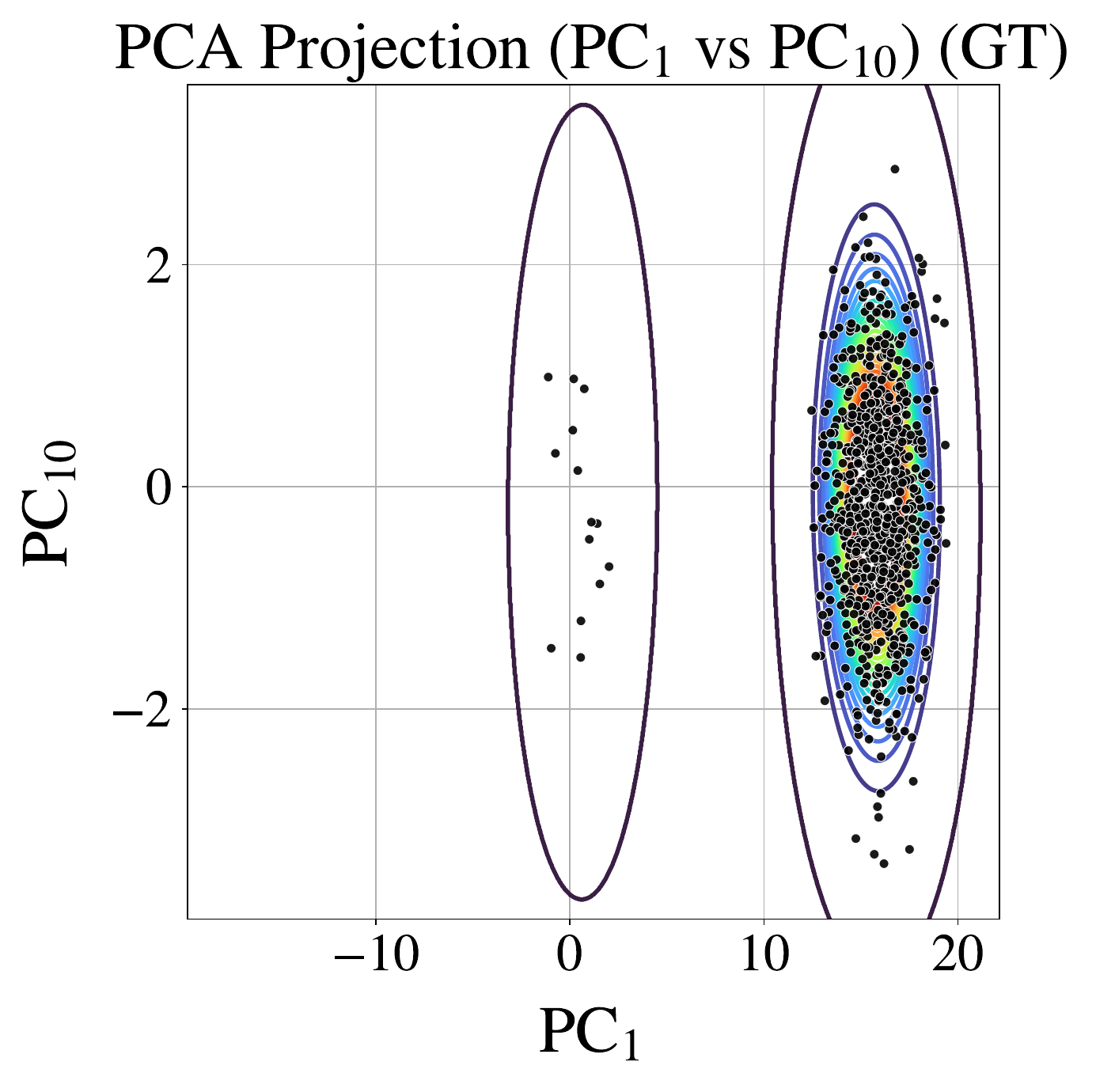}
    \caption{GT posterior}
\end{subfigure}

\vspace{0.1cm}

\begin{subfigure}[t]{0.24\textwidth}
    \centering
    \includegraphics[width=\linewidth]{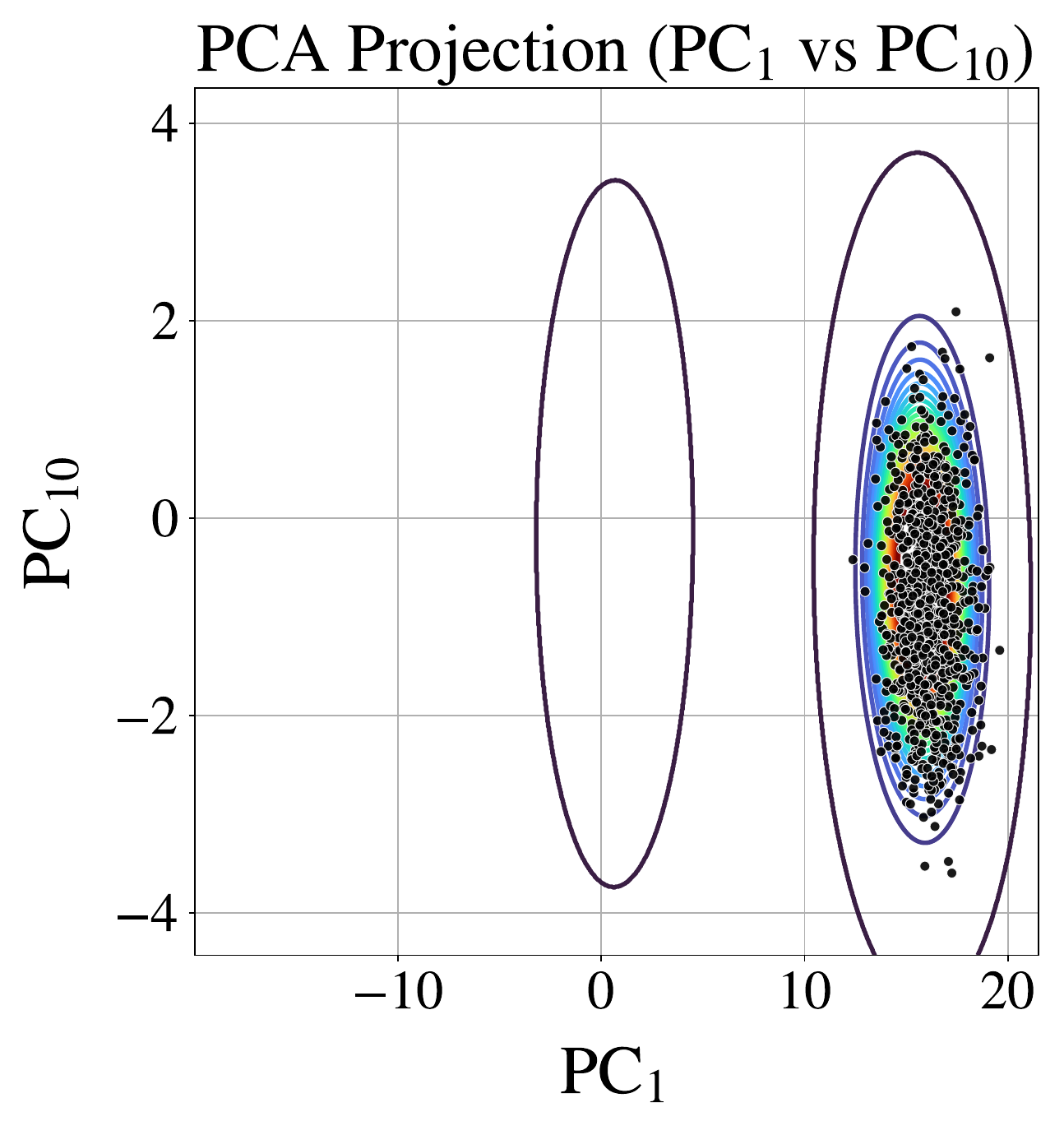}
    \caption{DPS}
\end{subfigure}
\hfill
\begin{subfigure}[t]{0.24\textwidth}
    \centering
    \includegraphics[width=\linewidth]{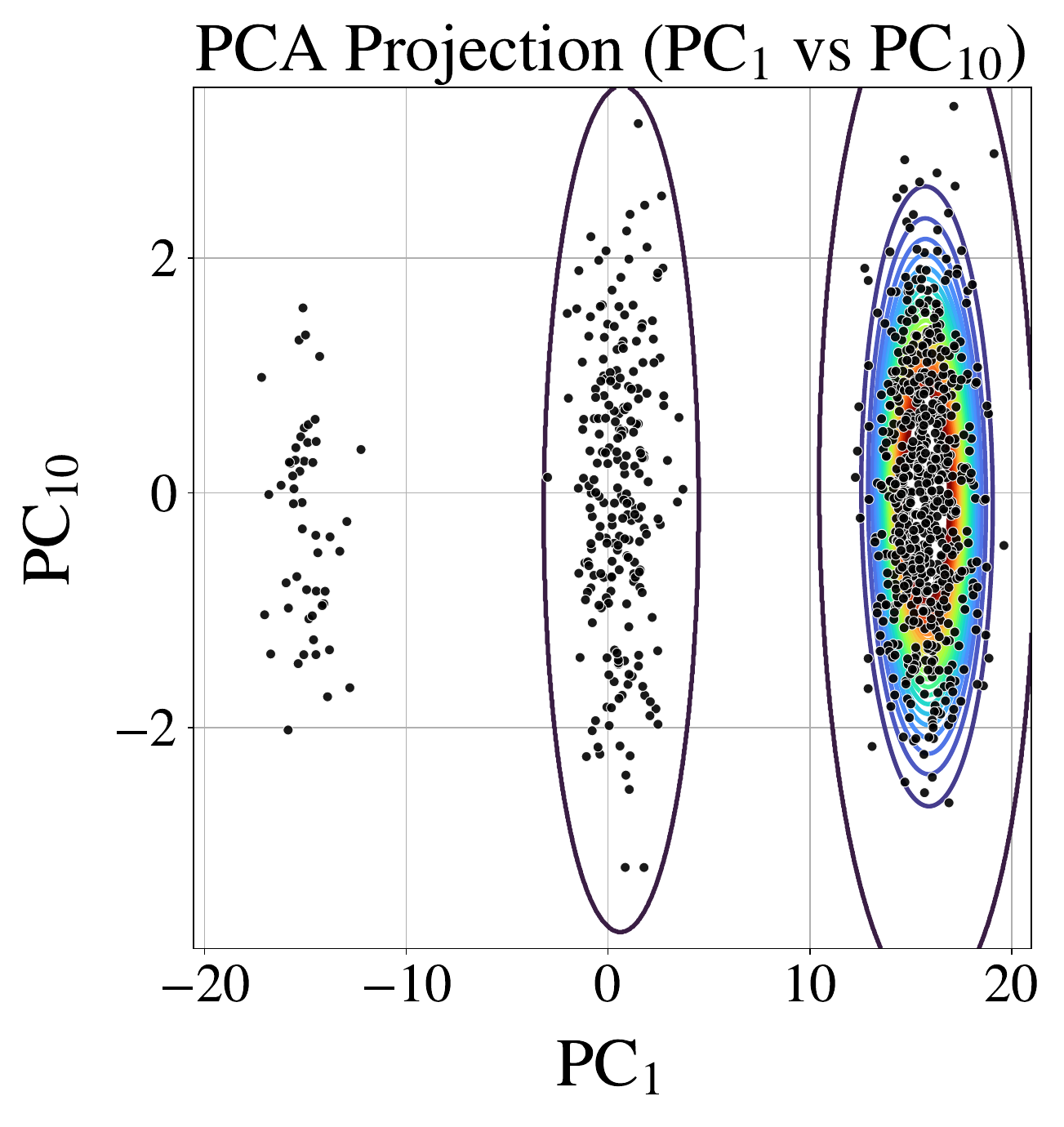}
    \caption{$\Pi$GDM}
\end{subfigure}
\hfill
\begin{subfigure}[t]{0.24\textwidth}
    \centering
    \includegraphics[width=\linewidth]{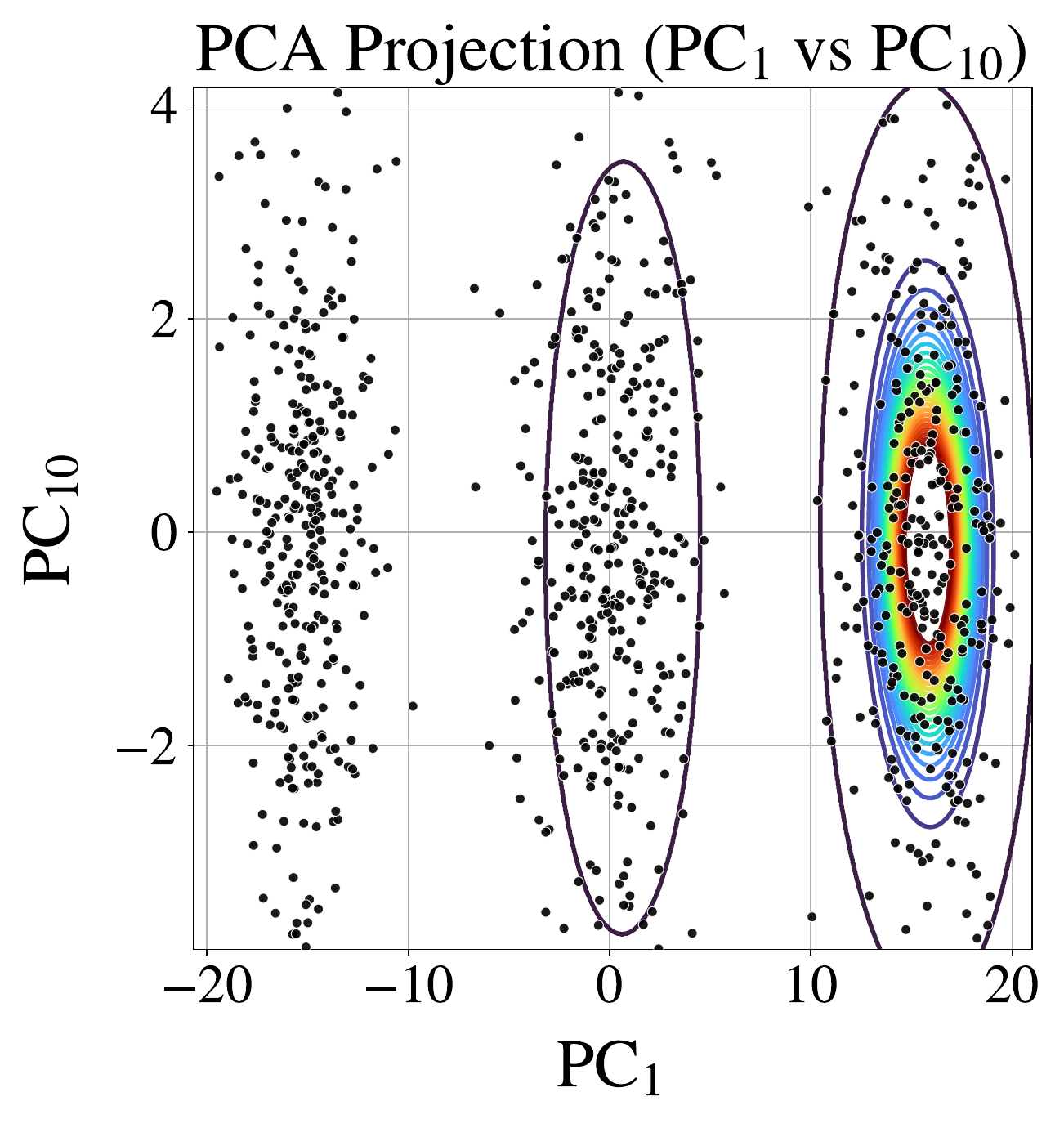}
    \caption{DAPS}
\end{subfigure}
\hfill
\begin{subfigure}[t]{0.24\textwidth}
    \centering
    \includegraphics[width=\linewidth]{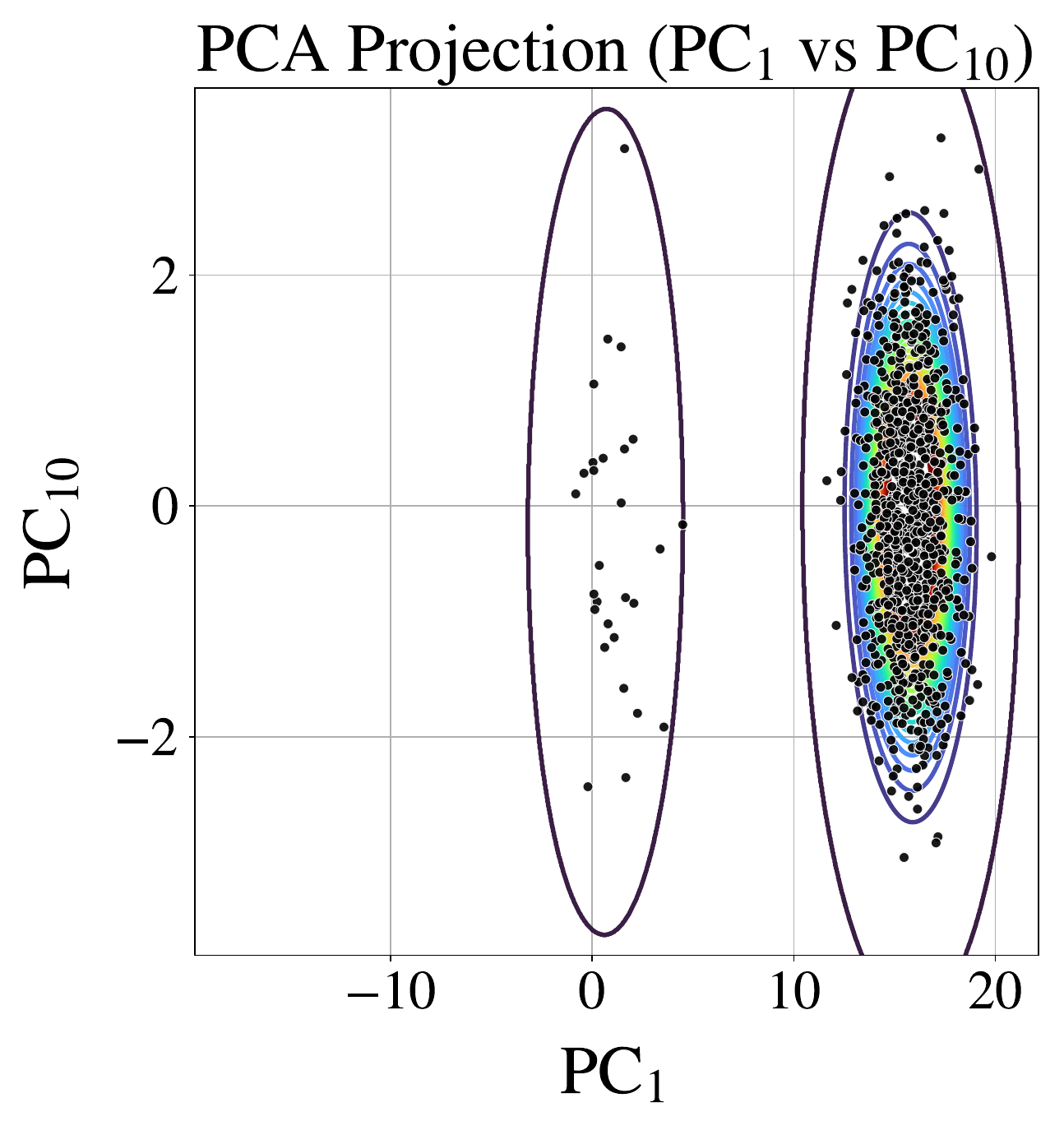}
    \caption{TR ($\epsilon = 0.1$)}
\end{subfigure}

\caption{Qualitative comparison for the \emph{inpainting} benchmark problem. Samples are shown in the two-dimensional subspace spanned by the eigenvectors of the reference posterior covariance associated with its smallest and largest eigenvalues.
}
\label{fig:inpainting_qualitative}

\end{figure}

\subsection{X-ray tomography problem}

Table~\ref{tab:nonlinear_posterior_metrics} reports results for the x-ray tomography benchmark, which is nonlinear and has no closed-form posterior or optimal control. $\Pi$GDM and control error diagnostics are therefore omitted. Both trust-region configurations outperform DPS and DAPS across all posterior accuracy metrics. In this benchmark, the more conservative trust-region radius $\epsilon=0.01$ gives the best overall performance. The importance-sampling diagnostics in Table~\ref{tab:nonlinear_is_metrics} show the same trend: TR with $\epsilon=0.01$ attains the largest NESS and the lowest reweighted mean error, indicating more stable path-space weights and more accurate importance-corrected posterior estimates.

\begin{table}[h]
\centering
\small
\resizebox{\textwidth}{!}{%
\begin{tabular}{llcccc}
\toprule
Problem & Method 
& {\makecell{Mean\\Error $\downarrow$}} 
& {\makecell{Cov\\Error $\downarrow$}} 
& {\makecell{MMD $\downarrow$}} 
& {\makecell{CMD $\downarrow$}} \\
\midrule
\multirow{4}{*}{\makecell[l]{X-ray\\tomography}}
& DPS 
& $9.03 \pm 6.66$ 
& $4.09 \pm 1.95$ 
& $0.453 \pm 0.323$ 
& $2.82 \pm 1.80$ \\

& DAPS 
& $0.402 \pm 0.119$ 
& $3.36 \pm 0.51$ 
& $0.167 \pm 0.031$ 
& $0.443 \pm 0.073$ \\

& TR ($\epsilon=0.01$) 
& $\mathbf{0.171 \pm 0.144}$ 
& $\mathbf{0.709 \pm 0.275}$ 
& $\mathbf{0.013 \pm 0.018}$ 
& $\mathbf{0.076 \pm 0.032}$ \\

& TR ($\epsilon=0.1$) 
& $0.173 \pm 0.145$ 
& $0.715 \pm 0.291$ 
& $0.014 \pm 0.019$ 
& $0.077 \pm 0.034$ \\

\midrule

\multirow{4}{*}{\makecell[l]{Phase\\retrieval}}
& DPS 
& $3.13 \pm 0.839$ 
& $3.88 \pm 0.741$ 
& $0.199 \pm 0.109$ 
& $0.572 \pm 0.041$ \\

& DAPS 
& $2.25 \pm 2.61$ 
& $5.73 \pm 1.39$ 
& $0.383 \pm 0.072$ 
& $0.453 \pm 0.218$ \\

& TR ($\epsilon=0.01$) 
& $2.513 \pm 7.030$ 
& $1.203 \pm 2.28$ 
& $0.159 \pm 0.482$ 
& $0.299 \pm 0.809$ \\

& TR ($\epsilon=0.1$) 
& $\mathbf{0.504 \pm 0.641}$ 
& $\mathbf{0.446 \pm 0.373}$ 
& $\mathbf{0.007 \pm 0.007}$ 
& $\mathbf{0.071 \pm 0.087}$ \\
\bottomrule
\end{tabular}
}
\caption{Quantitative posterior sampling results for the two nonlinear benchmark problems}
\label{tab:nonlinear_posterior_metrics}
\end{table}

\begin{table}[h]
\centering
\small
\resizebox{.85\textwidth}{!}{%
\begin{tabular}{llccc}
\toprule
Problem & Method 
& {\makecell{NESS $\uparrow$}} 
& {\makecell{Reweighted\\Mean Error $\downarrow$}} 
& {\makecell{Percent Reduction in\\Mean Error $\uparrow$}} \\
\midrule
\multirow{4}{*}{\makecell[l]{X-ray\\tomography}}
& DPS 
& $0.0006 \pm 0.0008$ 
& $2.84 \pm 1.25$ 
& $\mathbf{57.8 \pm 25.0}$ \\

& DAPS 
& -- 
& -- 
& -- \\

& TR ($\epsilon=0.01$) 
& $\mathbf{0.607 \pm 0.213}$ 
& $\mathbf{0.143 \pm 0.074}$ 
& $4.79 \pm 18.1$ \\

& TR ($\epsilon=0.1$) 
& $0.607 \pm 0.214$ 
& $0.146 \pm 0.079$ 
& $5.14 \pm 16.8$ \\

\midrule

\multirow{4}{*}{\makecell[l]{Phase\\retrieval}}
& DPS 
& $0.0001 \pm 0.0000$ 
& $7.11 \pm 5.04$ 
& $-106 \pm 91.7$ \\

& DAPS 
& -- 
& -- 
& -- \\

& TR ($\epsilon=0.01$) 
& $0.552 \pm 0.167$ 
& $2.51 \pm 7.03$ 
& $\mathbf{-7.08 \pm 68.1}$ \\

& TR ($\epsilon=0.1$) 
& $\mathbf{0.581 \pm 0.205}$ 
& $\mathbf{0.534 \pm 0.873}$ 
& $-47.6 \pm 181$ \\
\bottomrule
\end{tabular}
}
\caption{Importance sampling diagnostics for the nonlinear benchmark problems.}
\label{tab:nonlinear_is_metrics}
\end{table}

The qualitative results in Figure~\ref{fig:xray_tomography_qualitative} are consistent with the quantitative metrics. Since the posterior density is not available analytically in this nonlinear benchmark, the reference contours are obtained from a kernel density estimate of the MCMC samples. The trust-region samples are in the closest agreement with the reference posterior, whereas DPS exhibits a structural mismatch and DAPS remains comparatively diffuse.

\begin{figure}
\centering

\begin{subfigure}[t]{0.24\textwidth}
    \centering
    \includegraphics[width=\linewidth]{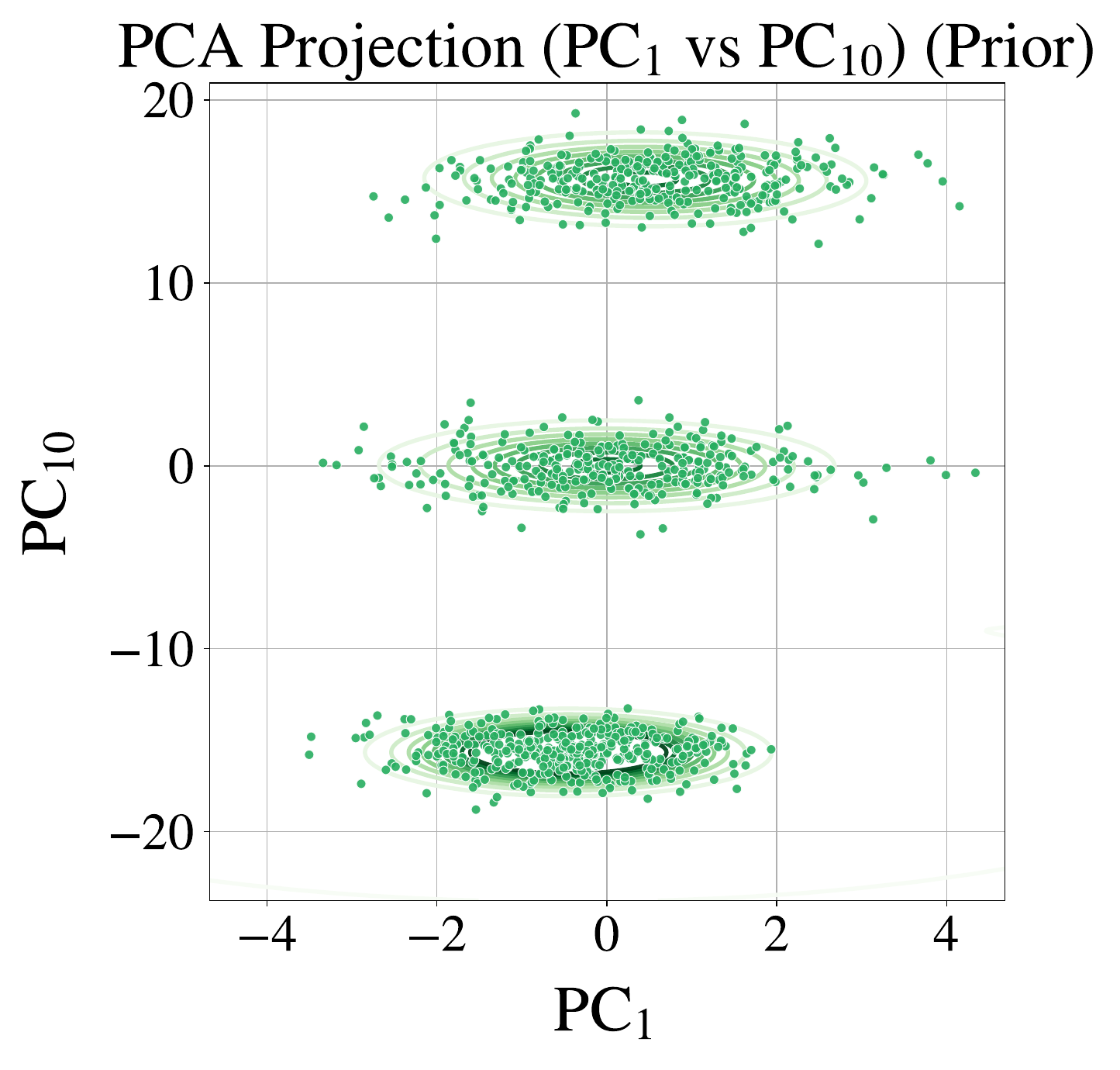}
    \caption{Prior}
\end{subfigure}
\hspace{1cm}
\begin{subfigure}[t]{0.24\textwidth}
    \centering
    \includegraphics[width=\linewidth]{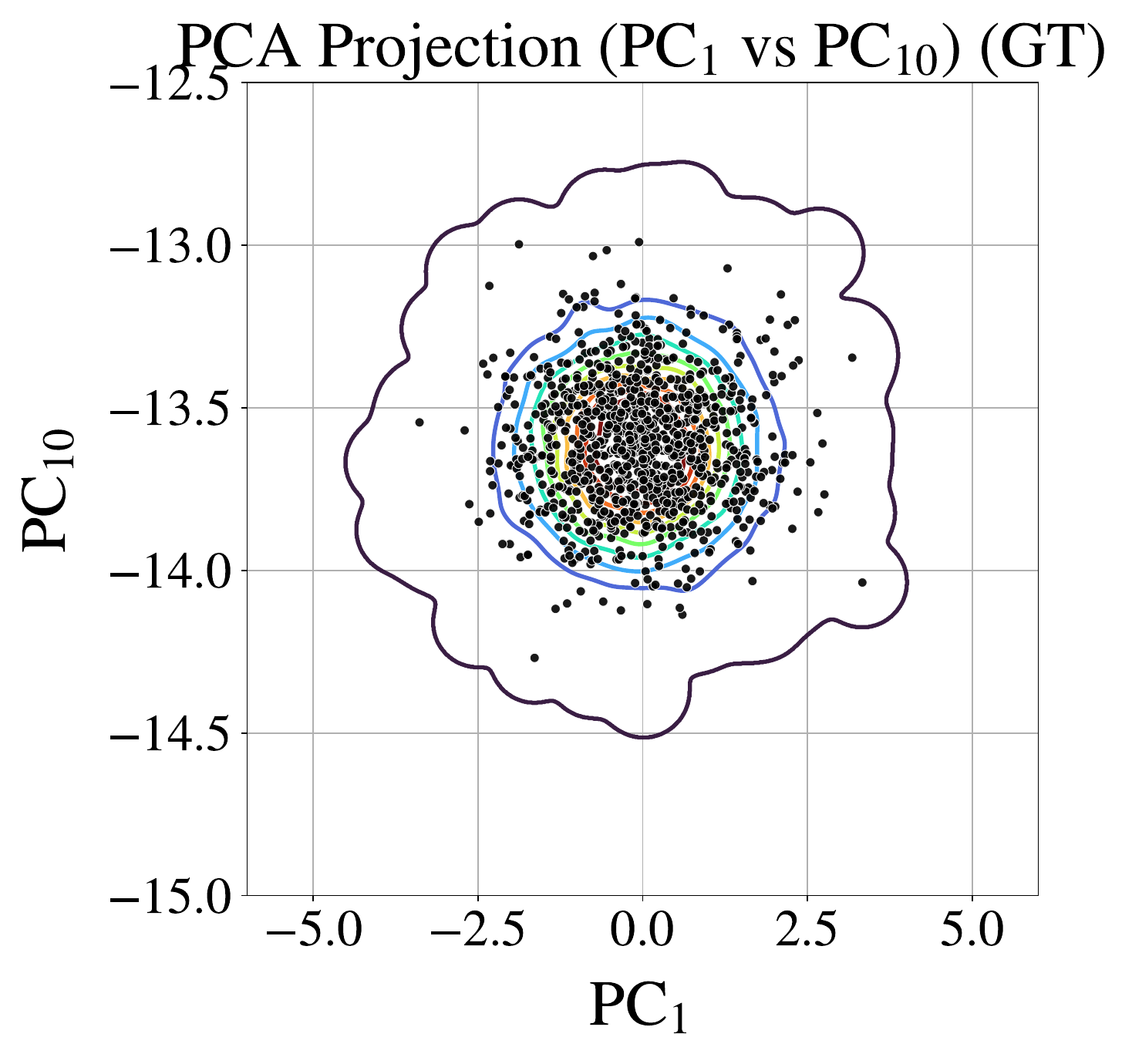}
    \caption{GT posterior}
\end{subfigure}

\vspace{0.1cm}

\begin{subfigure}[t]{0.24\textwidth}
    \centering
    \includegraphics[width=\linewidth]{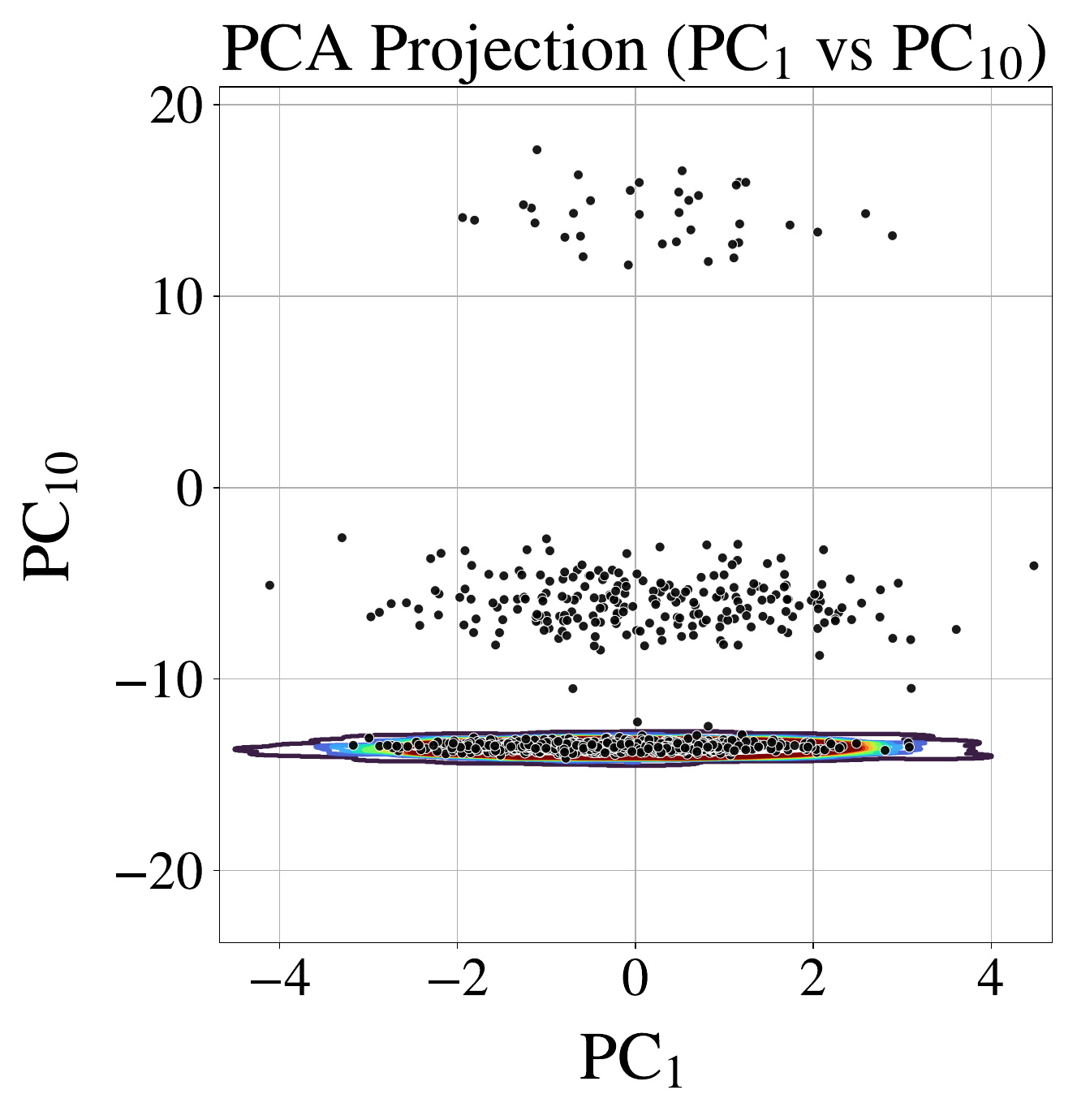}
    \caption{DPS}
\end{subfigure}
\hfill
\begin{subfigure}[t]{0.24\textwidth}
    \centering
    \includegraphics[width=\linewidth]{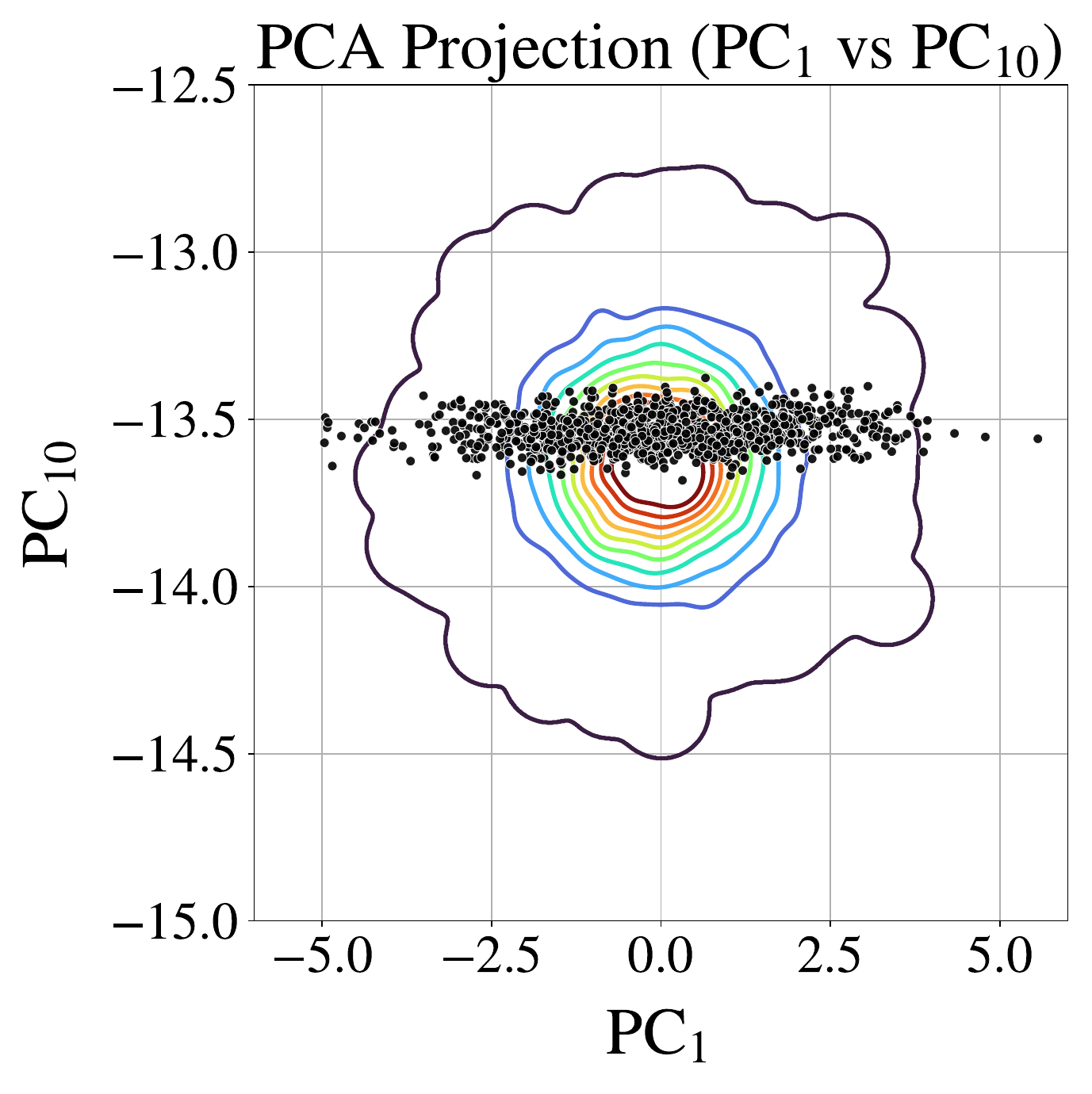}
    \caption{DAPS}
\end{subfigure}
\hfill
\begin{subfigure}[t]{0.24\textwidth}
    \centering
    \includegraphics[width=\linewidth]{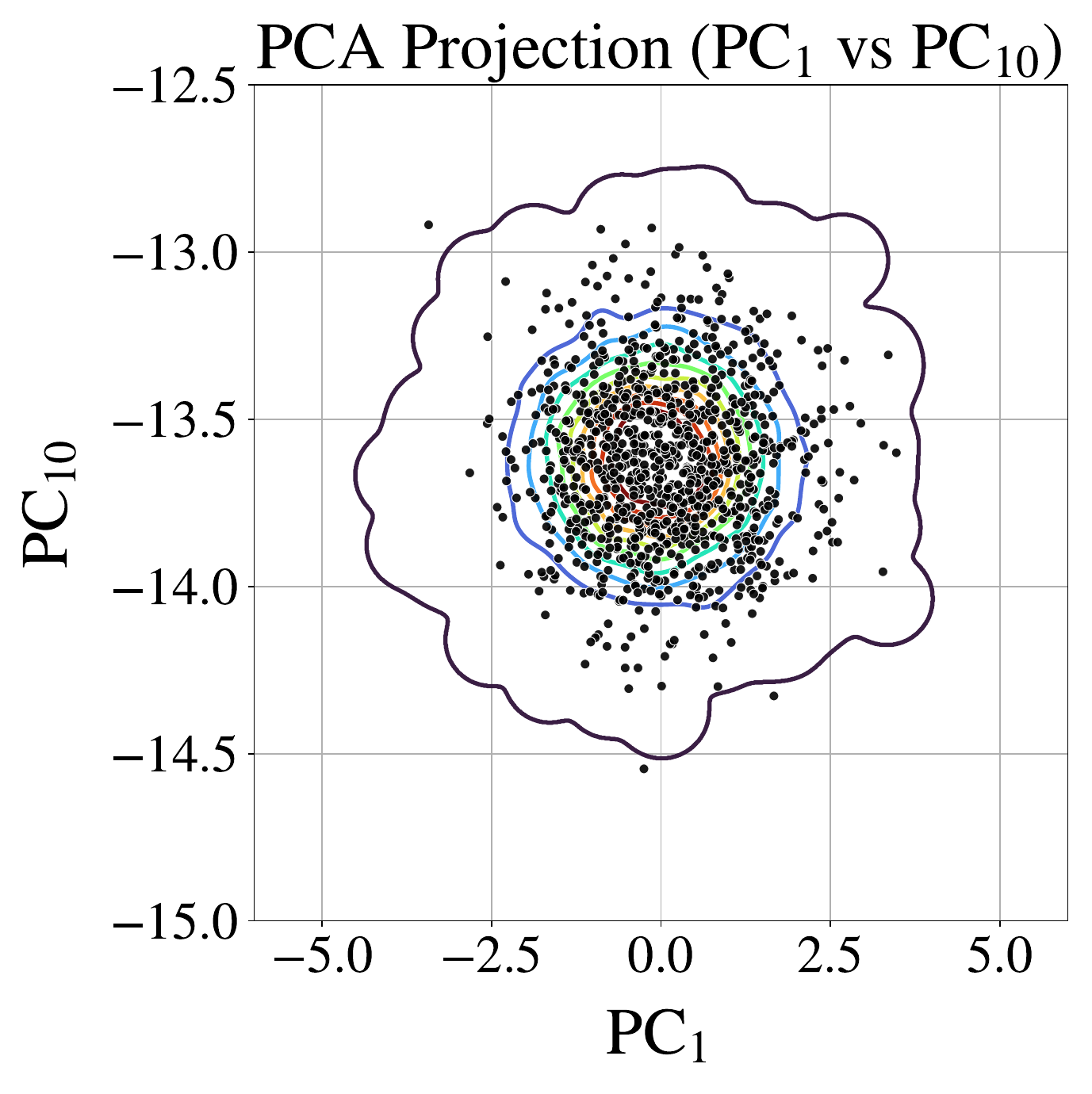}
    \caption{TR ($\epsilon = 0.01$)}
\end{subfigure}

\caption{Qualitative comparison for the \emph{x-ray tomography} benchmark problem. Reference contours are obtained from a kernel density estimate of the MCMC samples.}
\label{fig:xray_tomography_qualitative}

\end{figure}

\subsection{Phase retrieval problem}

Table~\ref{tab:nonlinear_posterior_metrics} reports results for the nonlinear and underdetermined phase retrieval benchmark. This is the most challenging problem considered and is more sensitive to the trust-region radius. The larger radius, $\epsilon=0.1$, substantially outperforms $\epsilon=0.01$ and achieves the best results across all posterior accuracy metrics, indicating that less conservative updates are beneficial in this setting. The importance-sampling diagnostics in Table~\ref{tab:nonlinear_is_metrics} show that TR also yields higher NESS than DPS, although reweighting is not beneficial for this difficult posterior. Figure~\ref{fig:phaseretrieval_qualitative} shows that TR with $\epsilon=0.1$ best matches the reference posterior projection, while DPS and DAPS deviate more strongly from the MCMC reference.

\begin{figure}
\centering

\begin{subfigure}[t]{0.24\textwidth}
    \centering
    \includegraphics[width=\linewidth]{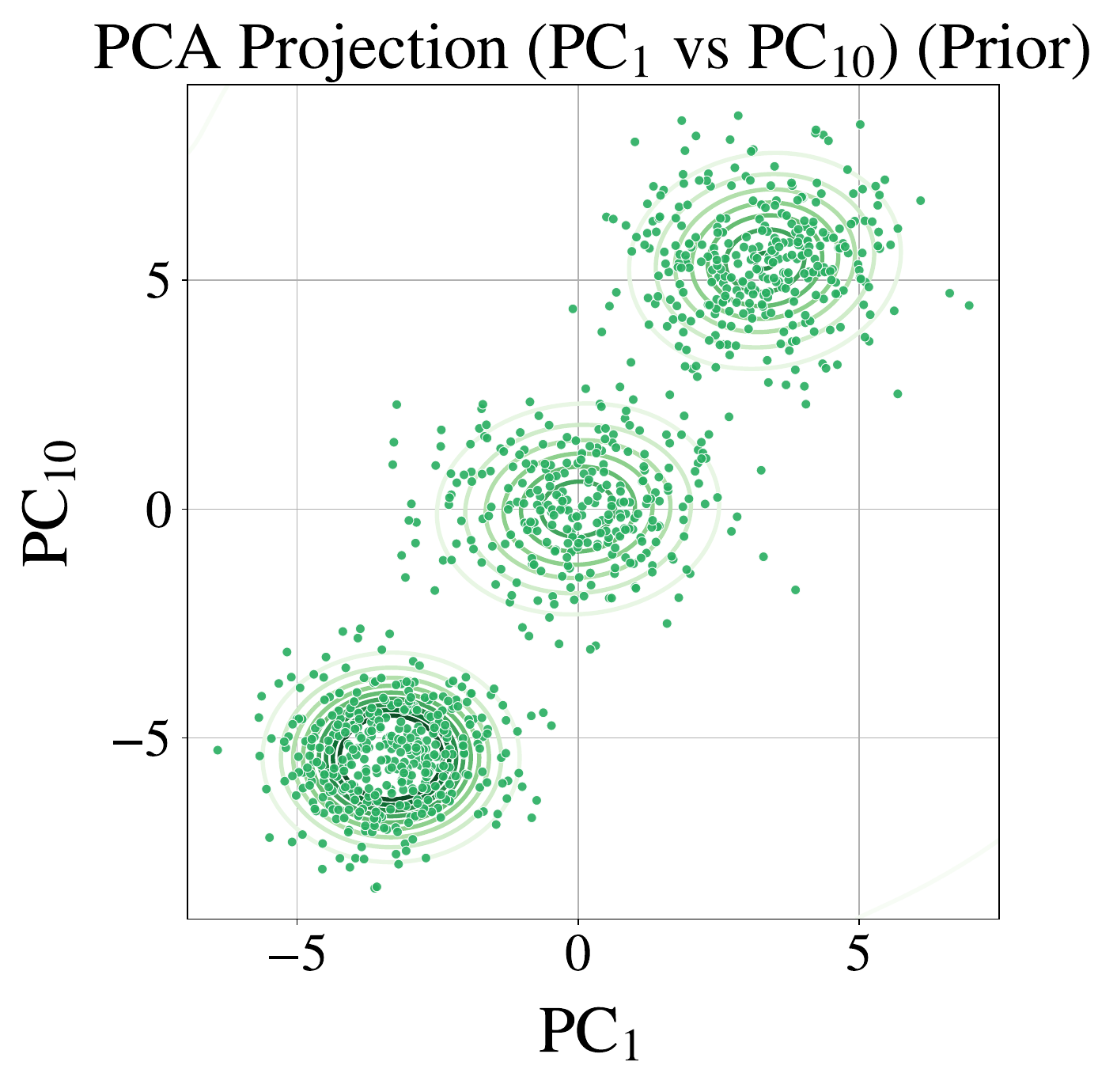}
    \caption{Prior}
\end{subfigure}
\hspace{1cm}
\begin{subfigure}[t]{0.24\textwidth}
    \centering
    \includegraphics[width=\linewidth]{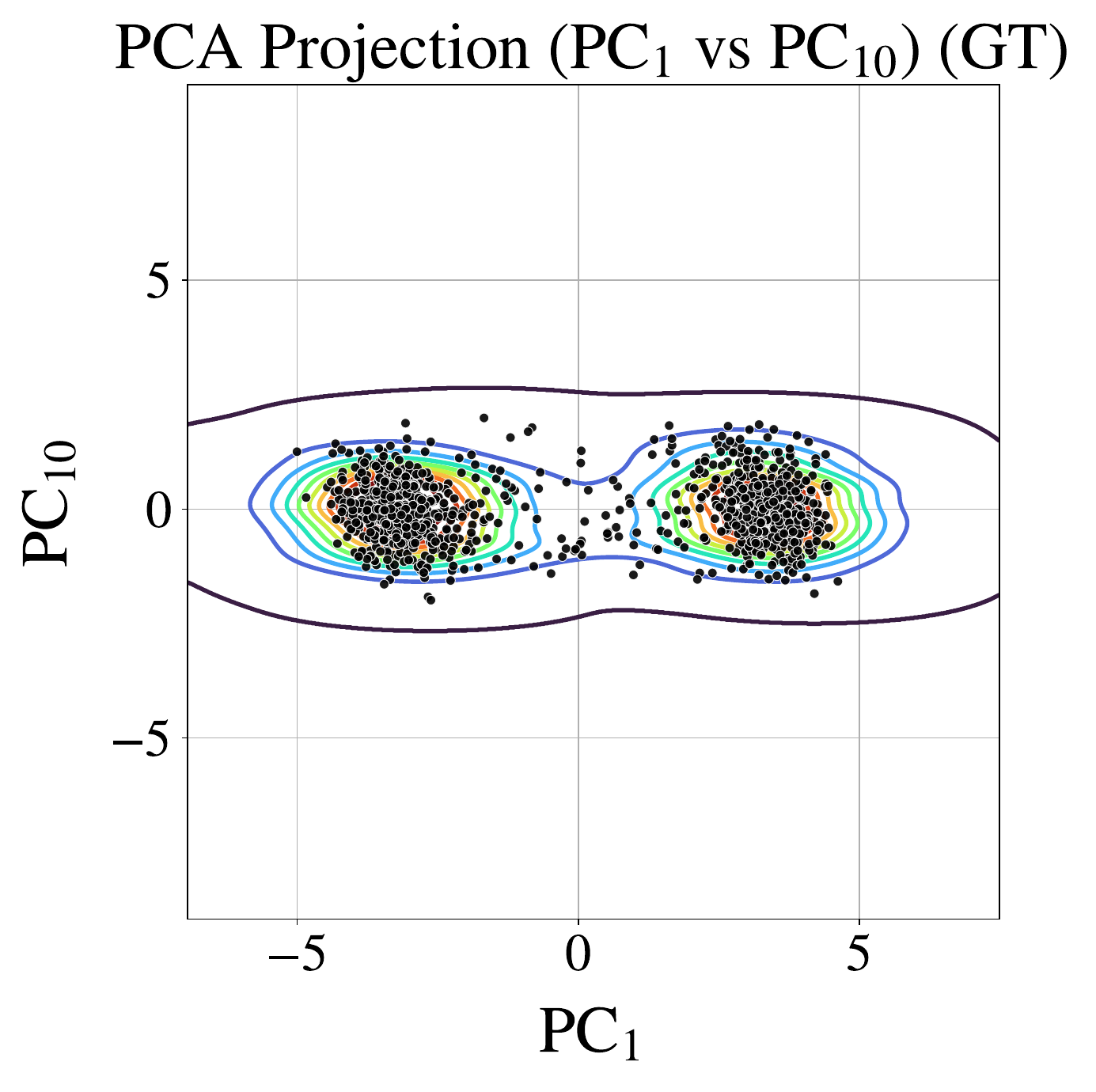}
    \caption{GT posterior}
\end{subfigure}

\vspace{0.1cm}
\begin{subfigure}[t]{0.24\textwidth}
    \centering
    \includegraphics[width=\linewidth]{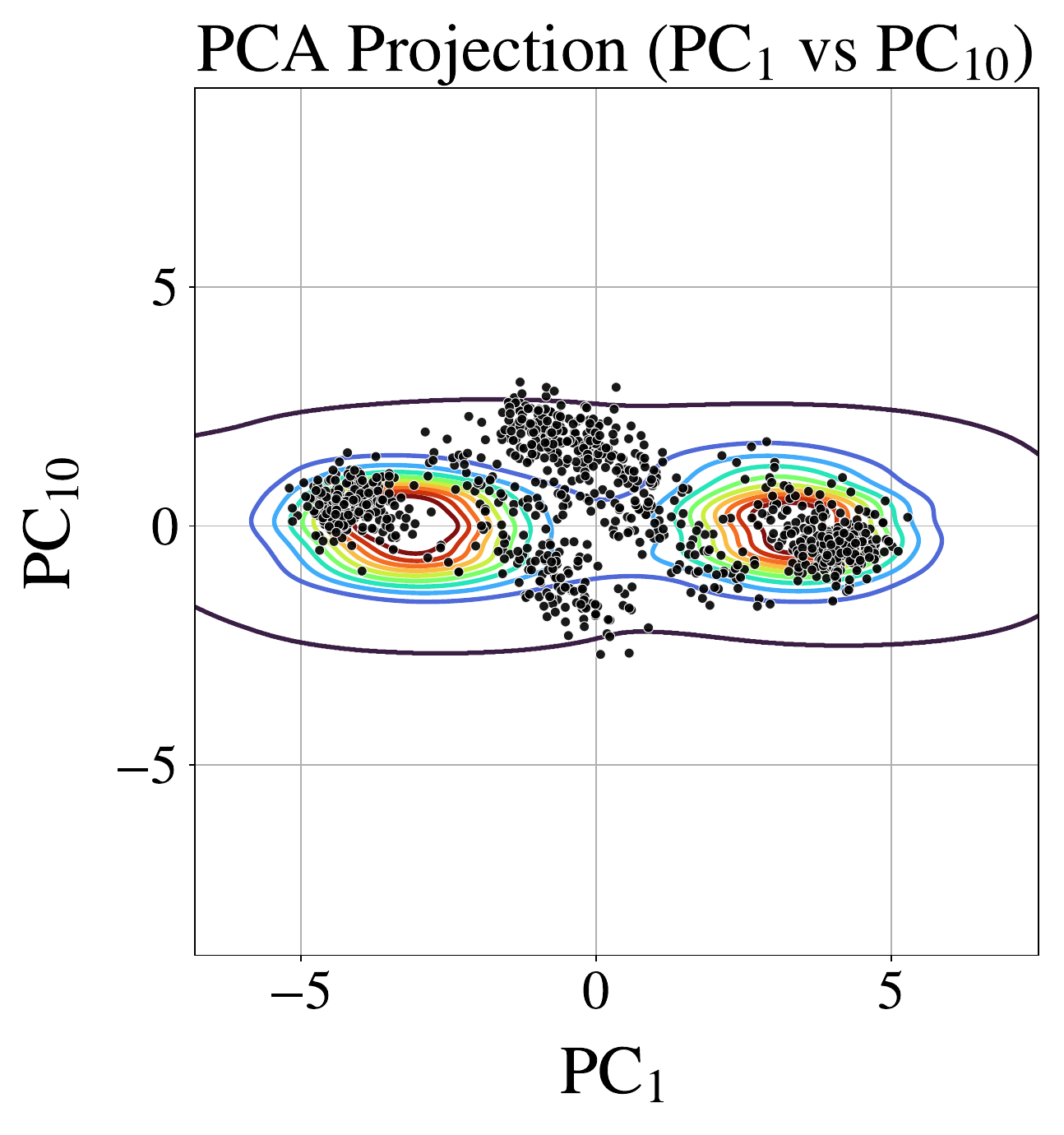}
    \caption{DPS}
\end{subfigure}
\hfill
\begin{subfigure}[t]{0.24\textwidth}
    \centering
    \includegraphics[width=\linewidth]{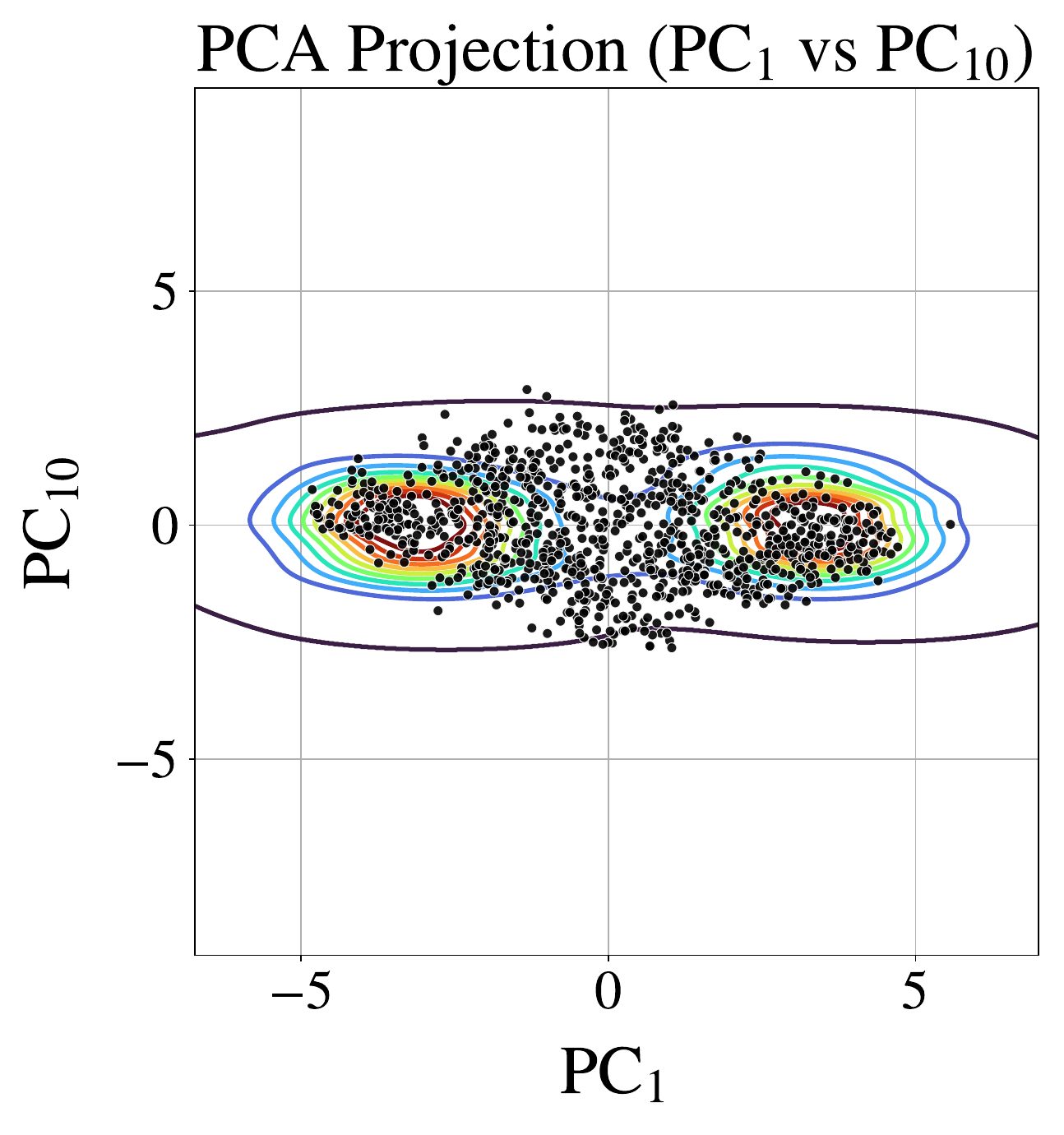}
    \caption{DAPS}
\end{subfigure}
\hfill
\begin{subfigure}[t]{0.24\textwidth}
    \centering
    \includegraphics[width=\linewidth]{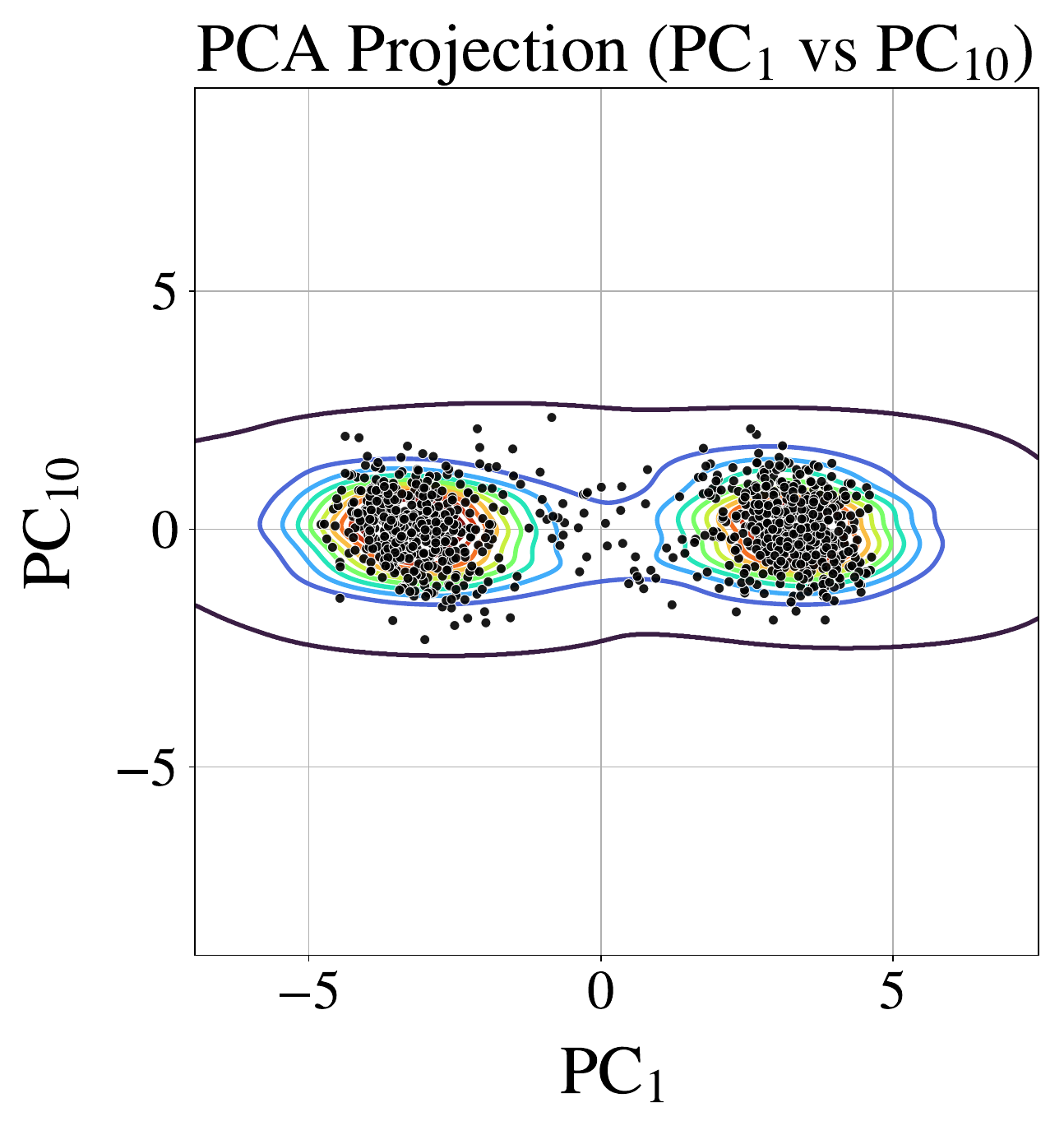}
    \caption{TR ($\epsilon = 0.1$)}
\end{subfigure}

\caption{Qualitative comparison for the \emph{phase retrieval} benchmark problem. Reference contours are obtained from a kernel density estimate of the MCMC samples.}
\label{fig:phaseretrieval_qualitative}

\end{figure}

\section{Conclusion}
\label{sec:discuss_conclude}

The path-space perspective has rich origins in Schr\"{o}dinger's pioneering work on the Schr\"{o}dinger Bridge Problem \cite{schrodinger1931umkehrung, schrodinger1932inversion} and Feynman's path-integral formalism of quantum mechanics \cite{feynman1948spacetime}. In recent years, it has garnered significant interest from practitioners in stochastic optimal control, generative modeling, and optimal transport, where advances such as the log-variance divergence \cite{nusken2021solving}, combined with neural network parameterizations, have led to state-of-the-art results in statistical physics and chemistry \cite{havens2025adjoint}.

Building on this foundation, we have developed a path-space formulation of diffusion-based posterior sampling. We introduced a time reparameterization technique that ensures well-posedness by removing the bias associated with the unknown initial value function, without requiring auxiliary training. Within this framework, posterior sampling is cast as a stochastic control problem, enabling the use of trust-region path-space optimization to learn approximate controls. Empirically, we demonstrated that this approach achieves state-of-the-art performance across a suite of rigorously characterized benchmark inverse problems.

Beyond algorithmic performance, the path-space perspective provides a unified framework for analyzing both learned and guidance-based samplers. In particular, it yields importance sampling weights as Radon--Nikodym derivatives between trajectory measures, enabling asymptotically unbiased estimation of posterior expectations. Our experiments further show that these weights can effectively correct errors in samples generated by existing diffusion-based posterior sampling methods.

Our numerical studies, based on the benchmark problems introduced in \cite{scopecrafts2025benchmarking}, enable rigorous evaluation of posterior sampling accuracy and uncertainty quantification. However, these problems involve latent dimensions of at most $20$, and therefore remain far smaller than the high-dimensional settings (e.g., image reconstruction) that motivate diffusion-based approaches. Developing high-dimensional benchmark problems that retain analytical tractability remains an important open challenge \cite{zach2025statistical}. Another important direction is to reduce the computational cost of the proposed approach by amortizing the learned posterior control across observations. Instead of learning a separate control for each measurement $\bsy$, one could train a conditional control $\bsu_\theta(\bsx,t,\bsy)$, or more generally $\bsu_\theta(\bsx,t,\bsy,\mathcal{F})$, over a distribution of observations and forward operators. Such an amortized control could provide a fast initialization, or even a reusable posterior sampler, for new inverse problem instances, while the path-space importance weights developed here could be used to diagnose and correct residual bias. 

A complementary direction is to combine the proposed posterior path space framework with recent approaches for prior training. For scientific inverse problems with equality-constrained or manifold-supported priors, manifold-aware perturbation methods \cite{keegan2026manifold} provide a promising strategy for constructing stable full-dimensional prior-generating diffusions. Such priors could then serve as the base processes transformed into posterior samplers by the path-space control framework developed in this work. Other directions for future work include investigating discretization error and developing tests for comparing the time reparameterization trick with other approaches for addressing the initial value function bias problem.

\section*{Acknowledgments}
The authors are grateful to Youssef Mroueh and Rapha\"el Pestourie for helpful discussions which helped motivate the questions explored in this paper.

\bibliographystyle{siamplain}
\bibliography{references}

\appendix 

\section{Background on Measures in Path Space}
\label{sec:appendix_pathmeasures}

In this appendix we provide a brief review of the mathematical theory relevant to our measure-theoretic path space formulation. Our presentation is based on {\O}ksendal's \textit{Stochastic Differential Equations} \cite{oksendal2003stochastic}. 

\paragraph{Path space} Throughout this work we consider path spaces given by sets of continuous functions, as is standard with diffusion processes. In particular, as described in Section \ref{sec:prob_form}, we consider path spaces of the form 
$$
\mathcal{C}(s) \triangleq \{ \bsx(t): [s, T] \to \mathbb{R}^D \mid \bsx(t) \text{ continuous} \},
$$
where here $s < T$. For ease of presentation, in this appendix we consider the setting after the time reparameterization described in Section \ref{sec:solution_approach}; we therefore take $s=-1$ and define $\mathcal{C} \triangleq \mathcal{C}(-1)$. Equipping $\mathcal{C}$ with the supremum norm induces a topology over the set. We define $\mathcal{B}$ as the corresponding Borel $\sigma$-algebra generated over $\mathcal{C}$, with $(\mathcal{C}, \mathcal{B})$ forming a measurable space.

\paragraph{Measures in path space} A stochastic process over the path space is a collection of random variables $\{\bsx_t\}$ indexed by $t \in [-1, T]$, with the law of $\bsx_t$ assumed to be absolutely continuous with respect to the Lebesgue measure in $\mathbb{R}^D$. A stochastic process induces a measure over the path space $(\mathcal{C}, \mathcal{B})$. Two measures $P$ and $Q$ over the space are equivalent if for every measurable set $B \in \mathcal{B}$, $P(B) = 0$ if and only if $Q(B) = 0$. Equivalence of measures implies the Radon-Nikodym derivative between $P$ and $Q$ is well defined; the notation $d P /dQ = f$ is used to denote the Radon-Nikodym derivative of $P$ with respect to $Q$, with $f$ the (unique almost everywhere) function satisfying 
$$
P(B) = \int_{B} f dQ. 
$$
In this context the Kullback–Leibler (KL) divergence between the two measures, denoted $\mathcal{D}_{\mathrm{KL}}$ in this work, is well defined and satisfies 
$$
\mathcal{D}_{\mathrm{KL}}(P  \, \Vert \, Q) = \mathbb{E}_P \left [\log  f \right].
$$

\paragraph{Brownian motion and martingales} We use $d\bsw_t$ to denote canonical Brownian motion in $\mathbb{R}^D$ with corresponding Wiener measure $\mathcal{W}$. The Brownian motion induces an (augmented) filtration, i.e., a family of $\sigma$-algebras $\{ \mathcal{M}_t\}_{t=-1}^T$, with $\mathcal{M}_T$ satisfying $\mathcal{M}_T = \mathcal{B}$ due to the equivalence of cylinder and Borel $\sigma$-algebras. A stochastic process $g_t(\bsw_t)$ is said to be $\mathcal{M}_t$-adapted if for each $t \in [-1, T]$ the function $\bsw \to g_t(\bsw_t)$ is $\mathcal{M}_t$-measurable. The process is said to be a martingale with respect to the filtration if it is $F_t$-adapted and satisfies $\mathbb{E} [|g_t(\bsw_t)|] < \infty$ and $\mathbb{E}[g_s(\bsw_s) | \mathcal{M}_t] = g_t$ for all $s \in (t, T]$ and all $t \in [-1, T)$. The martingale representation theorem states that every martingale with respect to the Brownian motion filtration can be written in terms of a stochastic process. This is formalized in the following theorem, which is used in the main text in the proof of Theorem \ref{theorem:sol_existence}. See, e.g., Theorem 4.3.4 in \cite{oksendal2003stochastic} for a proof. 

\begin{theorem}[Martingale Representation Theorem]
    Suppose $g_t \in \mathbb{R}$ is a scalar-valued $\mathcal{M}_t$ martingale and that $ \mathbb{E}[g_t(\bsw_t)^2] < \infty$ for $t \in [-1, T]$. Then there exists a unique stochastic process $\mathbf{f}_t(\bsw_t)$ such that 
    $$
    g_t(\bsw_t) = g_{-1} + \int_{-1}^t \langle \mathbf{f}_s(\bsw_s),  d \bsw_s \rangle
    $$
    almost surely for all $t \in [-1, T]$, where the integral in the above expression is to be interpreted in the It\^{o} sense (see Chapter 3 of \cite{oksendal2003stochastic}). 
\end{theorem}

\paragraph{It\^{o} processes}

In this work we are interested in measures over this space induced by diffusion-type SDEs (It\^{o} processes) of the form given in \eqref{eq:uncontrolled_reparam_dynamics}, i.e., 
\begin{equation}
d\bsx_{t} = \bsA(\bsx_{t}, t) \, dt + \bsB(t) \, d\bsw_{t}, \quad \bsx_{-1} = \mathbf{0}, \quad t \in [-1, T].
\label{eq:uncontrolled_reparam_dynamics_appendix}
\end{equation}
Under the assumptions given by Theorem 5.2.1 in \cite{oksendal2003stochastic} on $\bsA$ and $\bsB$, which we assume to hold in this work, the above SDE has a strong solution and this solution is adapted to the filtration $\mathcal{M}_t$, i.e., $\bsx_t$ is a measurable stochastic process. Further, $\bsx_t$ satisfies
$$
\mathbb{E} \left [ \int_{-1}^T \| \bsx_t \|_2^2 \, dt \right] < \infty.
$$ 

Given a $C^2$ (twice continuously differentiable) map $g: \mathbb{R}^D \times [-1, T] \to \mathbb{R}^P$, the dynamics of $\bsy_t = g(\bsx_t, t)$ are characterizable via It\^{o}'s formula. This formula is used in our proof of Theorem \ref{theorem:sol_existence} in the special case where $D = P = 1$, and is stated below  in this form. See Theorem 4.1.2 of \cite{oksendal2003stochastic} for a proof. 

\begin{theorem}[One-Dimensional It\^{o}'s Formula]
Consider the diffusion process
$$
dx_{t} = A(t) \, dt +  B(t) \, dw_{t}, 
$$
where $x_t \in \mathbb{R}$ and $dw_t$ is one dimensional Brownian motion, and let $g: \mathbb{R} \times  [-1, T] \to \mathbb{R}$ be a $C^2$ function. Then $y_t = g(t, x_t)$ satisfies 
$$
d y_t = \left ( \frac{\partial g}{\partial t} + A(t) \frac{\partial g}{\partial x} + \frac{B(t)^2}{2} \frac{\partial^2 g}{\partial x^2} \right) \; dt + B(t) \frac{\partial g}{\partial x} \ dw_t. 
$$
    
\end{theorem}

\paragraph{Girsanov's theorem} 

The process $\bsx_t$ in \eqref{eq:uncontrolled_reparam_dynamics_appendix} induces a measure $P$ over $(\mathcal{C}, \mathcal{B})$, referred to as the base process throughout this work. We are interested in constructing controlled versions of the base process---in particular, processes with additive Markov controls. Girsanov's theorem characterizes the change of measure between the base process and the controlled process. The following theorem gives Girsanov's theorem in the form used in this work; see Theorem 8.6.5 in \cite{oksendal2003stochastic} for a proof. 

\begin{theorem}[Girsanov's Theorem for SDEs]
Consider the base process
$$
d\bsx_{t} = \bsA(\bsx_{t}, t) \, dt + \bsB(t) \, d\bsw_{t}, \quad \bsx_{-1} = \mathbf{0}, \quad t \in [-1, T], 
$$
and let $\bsu(\bsx_t, t)$ be a control that satisfies the Novikov condition
$$\mathbb{E}_{P} \left[\mathrm{exp} \left ( \frac{1}{2} \int_{-1}^T \| \bsu(\bsx_t, t) \|_2^2 \; dt \right) \right] < \infty.$$
Define the measure $P^{\bsu}$ through
$$
\frac{d P^{\bsu}}{dP} = \mathrm{exp}\left(\int_{-1}^T \langle \bsu(\bsx_s, s), d\bsw_s \rangle - \frac{1}{2} \int_{-1}^T \| \bsu(\bsx_s, s) \|_2^2 \; dt \right).
$$
Then 
$$
d\bsw^{\bsu}_t = d\bsw_t - \bsu(\bsx_t^{\bsu}, t) \; dt 
$$
is a Brownian motion with respect to $P^{\bsu}$, and under $P^{\bsu}$ the process satisfies the controlled dynamics
$$
d\bsx_{t}^{\bsu} = \left[\bsA(\bsx_{t}^{\bsu}, t) + \bsB(t) \; \bsu(\bsx_t^{\bsu}, t) \right] \, dt + \bsB(t) \, d\bsw_{t}, \quad \bsx_{-1}^{\bsu} = \mathbf{0}, \quad t \in [-1, T].
$$
\end{theorem}

\section{Implementation of Analytic Flow Model}
\label{sec:appendix_flow}

In this appendix, we provide details regarding the construction of the toy problem used to demonstrate the time reparameterization trick in Figure \ref{fig:trt_example}.

\paragraph{Inverse problem} We consider a one-dimensional Bayesian inverse problem with prior 
$$
\pi_{\mathrm{pr}}(x) = \sum_{i=1}^{N_m} w_i \, N(m_i, \sigma_i^2), 
$$
where here we set $N_m = 2$, $w_1 = w_2 = .5$, $m_1 = -2$ and $m_2 = 2$, and $\sigma_1^2 = \sigma_2^2 = .3$. The likelihood was a denoising likelihood, i.e., $\pi_{\mathrm{like}}(y \mid x) = N(y; x, \tau^2)$, with $\tau^2 = 2.5$.

\paragraph{Base process} We consider a base process inspired by flow matching \cite{lipman2023flow}. In particular, we first construct a deterministic ODE that samples from the prior, i.e., we seek a velocity field such that the solution to \eqref{eq:flow_velocity} satisfies $p_0(x_0) = \mathcal{N}(x_0; 0, 1)$ and $p_T(x_T) = \sum_i^{N_m} w_i \, \mathcal{N}(x_T; m_i, \sigma_i^2)$, where $T = 1$. A velocity field satisfying this condition can be written in terms of the following decomposition (see Eq. 8 in \cite{lipman2023flow}): 
\begin{equation}
\label{flow_gm_formula}
v_t(x_t) = \frac{1}{\sum p_{i, t}(x_t)} \sum_{i=1}^{N_m} v_i(x_t) p_{i, t}(x_t),
\end{equation}
where $v_i(x_t)$ is a velocity field that drives $p_0$ to $\mathcal{N}(x_t; m_i, \sigma_i^2)$, and $p_{i, t}(x_t) = w_i p(x_t \mid x_T \sim \mathcal{N}(x_t; m_i, \sigma_i^2))$. There are an infinite number of velocity fields that map between given Gaussian distributions; here we consider 
\begin{equation}
\label{flow_gauss_veloc}
v_i(x_t) = \left [ a_i x_t + b_i \right], \quad a_i = \log (\sigma)/T, \quad b_i = a_i (\sigma_i - 1)^{-1} m_i.
\end{equation}
Under this choice we have that 
\begin{equation}
\label{flow_comp_marginals}
p(x_t \mid x_T \sim \mathcal{N}(x_t; m_i, \sigma_i^2))) = \mathcal{N}(x_t ; (\sigma_i^{t/T} - 1)(\sigma_i - 1)^{-1} m_i, \sigma_i^{2t/T}),
\end{equation}
from which we can construct the marginals 
\begin{equation}
\label{eq:flow_marginals}
p_t(x_t) = \sum_{i=1}^{N_m} w_i \, \mathcal{N}(x_t; (\sigma_i^{t/T} - 1)(\sigma_i - 1)^{-1} m_i, \sigma_i^{2t/T})
\end{equation}
A quick calculation shows that 
$$
p(x_0) = \sum_{i=1}^{N_m} w_i \, \mathcal{N}(x_0 ; 0, 1) = \mathcal{N}(x_0 ; 0, 1)
$$
and 
$$
p(x_T) =  \sum_{i=1}^{N_m} w_i \, \mathcal{N}(x_t ; m_i, \sigma_i^2)
$$
as desired. So plugging \eqref{flow_gauss_veloc} and \eqref{flow_comp_marginals} into \eqref{flow_gm_formula} gives an analytic expression for a velocity field that samples from the given mixture distribution. 

The above flow model is an ODE that cannot be used directly on our framework. However, as described in Section \ref{sec:background}, stochasticity can be injected into the above process using Nelson's identity. Here we set $g_s(t)$ in \eqref{eq:nelsons_identity} as $g_s(t) \equiv .5$, so that the final base process reads 
\begin{equation*}
dx_t = \left(v_t(x_t) + .125 \, \nabla_{x} \log p_t(x_t) \right) dt  + .5 \, dw_t,
\end{equation*}
with $v_t$ given as above and $\nabla_{x} \log p_t(x_t)$ computed using \eqref{eq:flow_marginals}.

\paragraph{Computation of optimal control} To compute the optimal control, we use the analytic expression for the optimal control given by \eqref{eq:analytical_optimal_control}. This formula is intractable in general due to the high variance of the conditional distribution which appears in the definition of the optimal control. However, it is computable in the simple one-dimensional setting considered here. In particular, we estimate the control using a Monte-Carlo estimator with $N_{MC} = 25 (T - t)/T$ samples used to compute the expectation in the optimal control expression
$$
u(x_t, t) = B(t)^T \, \nabla_{x} \log \mathbb{E}_P[d Q^* / d P \mid x_t],
$$
and automatic differentiation used to compute the derivative.

\section{Benchmark Specifications}
\label{app:benchmark_details}

This appendix collects the complete benchmark definitions used in Section~\ref{sec:benchmarks}. Across all experiments, the inversion parameter is $\bsx\in\mathbb{R}^D$ and the prior is a Gaussian mixture model (GMM) with $N_m=3$ components.

\subsection{Gaussian mixture prior parameterization}
\label{app:gmm_params}

We use the prior
\begin{equation}
\label{eq:app_gmm_prior}
\pi_{\mathrm{pr}}(\bsx)
=\sum_{i=1}^{N_m} w_i\,\mathcal{N}(\bsx;\boldsymbol{\mu}_i,\boldsymbol{\Sigma}_i),
\qquad
w_i\ge 0,\ \sum_{i=1}^{N_m} w_i=1,
\end{equation}
with $N_m=3$ in all benchmarks. For compactness, we define the linear spacing operator
\[
\mathrm{linspace}(a,b;D)
\triangleq \big(a + \tfrac{j-1}{D-1}(b-a)\big)_{j=1}^D\in\mathbb{R}^D.
\]
We also use $\mathbf{1}_D\in\mathbb{R}^D$ to denote the all-ones vector.

\subsection{Random linear sensing with heteroscedastic Gaussian noise}
\label{app:random_linear_details}

This benchmark corresponds to Section~\ref{prob:random-sensing}. We set $D=K=20$ and define
\begin{equation}
\bsy = \mathbf{H}\bsx + \boldsymbol{\eta},\qquad \boldsymbol{\eta}\sim \mathcal{N}(\boldsymbol{\eta}; \mathbf{0},\boldsymbol{\Gamma}).
\end{equation}
The sensing matrix $\mathbf{H}\in\mathbb{R}^{20\times 20}$ has i.i.d.\ entries $\mathbf{H}_{ij}\sim\mathcal{N}(0,1)$.
The noise covariance is diagonal and heteroscedastic:
\begin{equation}
\boldsymbol{\Gamma} = \mathrm{diag}(\sigma_1^2,\ldots,\sigma_{20}^2),
\qquad
\sigma_j^2 = 500 + \frac{j-1}{19}(1000-500),\quad j=1,\ldots,20.
\end{equation}

\paragraph{Prior}
We use $w_1=0.4$, $w_2=0.3$, $w_3=0.3$ and
\begin{align}
\boldsymbol{\mu}_1 &= \mathrm{linspace}(-1,-5;D), & \boldsymbol{\Sigma}_1 &= 2\mathbf{I}_D,\\
\boldsymbol{\mu}_2 &= \mathbf{0}, & \boldsymbol{\Sigma}_2 &= \mathrm{diag}(\mathrm{linspace}(2,3;D)),\\
\boldsymbol{\mu}_3 &= \mathrm{linspace}(1,5;D), & \boldsymbol{\Sigma}_3 &= \mathbf{Q}^\top \boldsymbol{\Lambda}\mathbf{Q},
\quad \boldsymbol{\Lambda}=\mathrm{diag}(\mathrm{linspace}(2,3;D))\,,
\end{align}
where $\mathbf{Q}$ is a random orthogonal matrix obtained via QR decomposition of a random matrix. This construction ensures $\boldsymbol{\Sigma}_3$ has the same eigenvalues as $\boldsymbol{\Sigma}_2$ but with randomly oriented eigenvectors. Under this configuration, the prior consists of three well-separated modes: two with opposing linear trends in their means (components 1 and 3) positioned in different quadrants of the parameter space, and one centered mode with anisotropic uncertainty (component 2). 

\paragraph{Prior SDE}  We implement the prior-sampling SDE as a variance-exploding (VE) SDE \cite{song2021scorebased}, which corresponds to $\mathbf{g}(\bsx_t, t) \equiv 0$ in \eqref{eq:forward_diffusion_sde}. We set the diffusion term $h(t) \equiv 1$ and the time horizon as $T = 1000$. An expression for $\nabla_{\bsx} \log p_t(\bsx_t)$ under this SDE and the Gaussian mixture prior used here is given in Appendix \ref{app:closed_form_scores}, which completes the specification of the prior SDE. During training and inference, we discretize the SDE using $N_{\mathrm{nt}} = 100$ timesteps, with the $i$th timestep given by $t_i = T - T (.01/T)^{i/N_{\mathrm{nt}}}$.

\subsection{Inpainting with isotropic Gaussian noise}
\label{app:inpainting_details}

This benchmark problem corresponds to Section~\ref{prob:inpainting}. We set $D=10$ and $K=8$ and define
\begin{equation}
\bsy = \mathbf{M}\bsx + \boldsymbol{\eta},\qquad \boldsymbol{\eta}\sim\mathcal{N}(\mathbf{0},\sigma^2\mathbf{I}_K),
\end{equation}
where $\mathbf{M}\in\{0,1\}^{K\times D}$ selects a subset of coordinates of $\bsx$.
Using one-based indexing, we observe the index set
\[
\mathcal{O}=\{1,2,3,5,6,7,8,10\},\qquad \mathcal{U}=\{4,9\},
\]
and $\mathbf{M}$ has rows equal to the standard basis vectors $\{\mathbf{e}_j^\top\}_{j\in\mathcal{O}}$.
We set $\sigma=5$, i.e.\ $\boldsymbol{\eta}\sim\mathcal{N}(\mathbf{0},25\mathbf{I}_8)$.

\paragraph{Prior}
We use $w_1=0.4$, $w_2=0.3$, $w_3=0.3$ and
\begin{align}
\boldsymbol{\mu}_1 &= -5\mathbf{1}_D, & \boldsymbol{\Sigma}_1 &= \mathbf{I}_D,\\
\boldsymbol{\mu}_2 &= \mathbf{0}, & \boldsymbol{\Sigma}_2 &= \mathrm{diag}(\mathrm{linspace}(1,2;D)),\\
\boldsymbol{\mu}_3 &= 5\mathbf{1}_D, & \boldsymbol{\Sigma}_3 &= \mathbf{Q}\boldsymbol{\Lambda}\mathbf{Q}^\top,
\quad \boldsymbol{\Lambda}=\mathrm{diag}(\mathrm{linspace}(1,2;D)).
\end{align}
In the nonlinear benchmarks below we reuse this same prior.

\paragraph{Prior SDE} The prior SDE used here is the same one as used in the random linear sensing benchmark problem, but with time-horizon $T = 500$ instead of $T = 1000$. The discretization scheme is unchanged. 

\subsection{X-ray tomography with Poisson likelihood}
\label{app:xray_details}

This benchmark corresponds to Section~\ref{prob:xray}. We set $D=10$, $K=15$, and define the
intensity map
\begin{equation}
f(\bsx) = I_0 \exp(-\mathbf{C}\bsx),
\end{equation}
where the exponential is applied componentwise, $I_0=1000$, and $\mathbf{C}\in\mathbb{R}^{15\times 10}$ has entries drawn i.i.d.\ from $\mathrm{Unif}[0.01,0.05]$. The likelihood is Poisson:
\begin{equation}
\pi_{\mathrm{like}}(\bsy\mid \bsx) = \mathrm{Poi}\big(f(\bsx)\big).
\end{equation}
The prior and prior SDE are the same as in inpainting GMM from Appendix~\ref{app:inpainting_details}.

\subsection{Phase retrieval with additive Gaussian noise}
\label{app:phase_details}

This benchmark corresponds to Section~\ref{prob:phase}. We set $D=10$, $K=5$, and define
\begin{equation}
\bsy = (\mathbf{N}\bsx)^{\odot 2} + \boldsymbol{\eta},
\qquad
\boldsymbol{\eta}\sim \mathcal{N}(\mathbf{0},\sigma^2\mathbf{I}_K),
\end{equation}
where $(\cdot)^{\odot 2}$ denotes elementwise squaring and $\mathbf{N}\in\mathbb{R}^{5\times 10}$ has i.i.d.\ entries $\mathbf{N}_{ij}\sim\mathcal{N}(0,1)$. We set $\sigma=25$. The prior and prior SDE are the same as in inpainting GMM from Appendix~\ref{app:inpainting_details}.

\section{Closed-Form Noisy Prior Marginals, Scores, and Denoising Distributions for GMM Priors}
\label{app:closed_form_scores}

A central feature of the benchmarks in Section~\ref{sec:num_studies} is that the prior admits closed-form evaluation of the noisy marginal density, its score, and the associated denoising distribution under Gaussian perturbation kernels.
This eliminates any need to pre-train a diffusion model and isolates algorithmic error in posterior sampling from approximation error in prior learning.

To connect with the notation of Sections~\ref{sec:background}--\ref{sec:prob_form}, recall that diffusion-type generative models are described in terms of their time-$t$ marginals $p_t$ and corresponding scores $\nabla_\bsx \log p_t(\bsx)$, with $\bsx_0 \sim \mathcal{N}(\mathbf{0}, \mathbf{I}_D)$ and terminal marginal $p_T = \pi_{\mathrm{pr}}$ for an exact prior-sampling base process. In the benchmark setting considered here, these intermediate marginals can be written explicitly by specifying a Gaussian perturbation kernel linking a terminal prior sample $\bsx_T \sim \pi_{\mathrm{pr}}$ to an intermediate state $\bsx_t$. This yields closed-form expressions for the marginal $p_t$, its score, and the denoising distribution $p_{T \mid t}(\bsx_T \mid \bsx_t)$.

\subsection{Gaussian perturbation kernel}
Let $\alpha_t\in\mathbb{R}$ and $\sigma_t>0$ define a time-indexed Gaussian perturbation kernel
\begin{equation}
\bsx_t = \alpha_t \bsx_T + \sigma_t \varepsilon,\qquad \varepsilon\sim \mathcal{N}(\mathbf{0}, \mathbf{I}_D).
\label{eq:app_gaussian_kernel}
\end{equation}
Equivalently,
\begin{equation}
p_{t\mid T}(\bsx_t\mid \bsx_T)=\mathcal{N}(\bsx_t;\alpha_t \bsx_T,\sigma_t^2 \mathbf{I}_D).
\label{eq:C2_qt}
\end{equation}
The expressions below hold for any choice of $(\alpha_t,\sigma_t)$ consistent with the noising schedule used by the base process. In the numerical experiments of Section~\ref{sec:num_studies}, the uncontrolled base process follows a variance-exploding (VE) SDE \cite{song2021scorebased}, so we use the specialization
\begin{equation}
\alpha_t \equiv 1,
\qquad
\sigma_t = \sigma(t).
\label{eq:app_ve_specialization}
\end{equation}

\subsection{Noisy marginal induced by a GMM prior}
Recall that the benchmark prior is the Gaussian mixture
\begin{equation*}
\pi_{\mathrm{pr}}(\bsx)
=
\sum_{i=1}^{N_m} w_i \, \mathcal{N}(\bsx;\boldsymbol{\mu}_i,\boldsymbol{\Sigma}_i),
\qquad
w_i \ge 0,
\qquad
\sum_{i=1}^{N_m} w_i = 1,
\end{equation*}
as defined in \eqref{eq:gmm_prior}. If $\bsx_T \sim \pi_{\mathrm{pr}}$ and $\bsx_t$ is given by \eqref{eq:app_gaussian_kernel}, then the time-$t$ marginal remains a Gaussian mixture:
\begin{equation}
p_t(\bsx)
=
\sum_{i=1}^{N_m} w_i \, \mathcal{N}(\bsx;\boldsymbol{\mu}_{t,i},\boldsymbol{\Sigma}_{t,i}),
\label{eq:app_noisy_marginal_prior}
\end{equation}
where
\begin{equation}
\boldsymbol{\mu}_{t,i} = \alpha_t \boldsymbol{\mu}_i,
\qquad
\boldsymbol{\Sigma}_{t,i} = \alpha_t^2 \boldsymbol{\Sigma}_i + \sigma_t^2 \mathbf{I}_D.
\label{eq:app_component_params_prior}
\end{equation}
Thus Gaussian convolution preserves the mixture structure, with the mixture weights remaining unchanged.

\subsection{Closed-form noisy prior score}
Define the posterior responsibility of mixture component $i$ at time $t$ by
\begin{equation}
r_{t,i}(\bsx)
\triangleq
\mathbb{P}(I=i \mid \bsx_t=\bsx)
=
\frac{
w_i \, \mathcal{N}(\bsx;\boldsymbol{\mu}_{t,i},\boldsymbol{\Sigma}_{t,i})
}{
\sum_{j=1}^{N_m} w_j \, \mathcal{N}(\bsx;\boldsymbol{\mu}_{t,j},\boldsymbol{\Sigma}_{t,j})
}.
\label{eq:app_mixture_responsibilities_prior}
\end{equation}
Differentiating \eqref{eq:app_noisy_marginal_prior} yields the score of the noisy marginal:
\begin{align}
\nabla_\bsx \log p_t(\bsx)
&=
\sum_{i=1}^{N_m}
r_{t,i}(\bsx)\,
\nabla_\bsx \log  \mathcal{N}(\bsx;\boldsymbol{\mu}_{t,i},\boldsymbol{\Sigma}_{t,i})
\nonumber\\
&=
\sum_{i=1}^{N_m}
r_{t,i}(\bsx)\,\boldsymbol{\Sigma}_{t,i}^{-1}(\boldsymbol{\mu}_{t,i}-\bsx).
\label{eq:app_score_prior}
\end{align}
Equation~\eqref{eq:app_score_prior} is the closed-form score used in the benchmark experiments in place of a learned score network.

\subsection{Closed-form denoising distribution $p(\bsx_T \mid \bsx_t)$}

We next characterize the denoising distribution of the terminal clean variable $\bsx_T$ given the noisy state $\bsx_t$. Conditioned on mixture component $I=i$, the pair $(\bsx_T,\bsx_t)$ forms a linear--Gaussian model:
\[
\bsx_T \mid (I=i) \sim \mathcal{N}(\boldsymbol{\mu}_i,\boldsymbol{\Sigma}_i),
\qquad
\bsx_t \mid \bsx_T \sim \mathcal{N}(\alpha_t \bsx_T,\sigma_t^2 \mathbf{I}_D).
\]
Therefore,
\begin{equation}
\bsx_T \mid (\bsx_t=\bsx,\; I=i)
\sim
\mathcal{N}(\mathbf{m}_{t,i}(\bsx),\mathbf{P}_{t,i}),
\label{eq:app_component_conditional_prior}
\end{equation}
with covariance
\begin{equation}
\mathbf{P}_{t,i}
=
\left(
\boldsymbol{\Sigma}_i^{-1} + \frac{\alpha_t^2}{\sigma_t^2} \mathbf{I}_D
\right)^{-1},
\label{eq:app_Pti}
\end{equation}
and mean
\begin{align}
\mathbf{m}_{t,i}(\bsx)
&=
\mathbf{P}_{t,i}
\left(
\boldsymbol{\Sigma}_i^{-1}\boldsymbol{\mu}_i + \frac{\alpha_t}{\sigma_t^2}\bsx
\right) \nonumber
\\
&=
\boldsymbol{\mu}_i + \alpha_t \boldsymbol{\Sigma}_i \boldsymbol{\Sigma}_{t,i}^{-1}(\bsx-\boldsymbol{\mu}_{t,i}).
\label{eq:app_mti}
\end{align}
Marginalizing over the latent component index gives the full denoising distribution
\begin{equation}
p(\bsx_T \mid \bsx_t=\bsx)
=
\sum_{i=1}^{N_m}
r_{t,i}(\bsx)\,
\mathcal{N}(\bsx_T;\mathbf{m}_{t,i}(\bsx),\mathbf{P}_{t,i}),
\label{eq:app_denoising_mixture}
\end{equation}
where the same responsibilities $r_{t,i}(\bsx)$ from \eqref{eq:app_mixture_responsibilities_prior} appear. In particular, the conditional mean estimator of $\bsx_T$ given $\bsx_t=\bsx$ is
\begin{equation}
\mathbb{E}[\bsx_T \mid \bsx_t=\bsx]
=
\sum_{i=1}^{N_m}
r_{t,i}(\bsx)\, \mathbf{m}_{t,i}(\bsx).
\label{eq:app_denoiser_mean}
\end{equation}

Finally, the score and conditional mean estimator are related by
\begin{equation}
\mathbb{E}[\bsx_T \mid \bsx_t=\bsx]
=
\frac{1}{\alpha_t}
\left(
\bsx + \sigma_t^2 \nabla_\bsx \log p_t(\bsx)
\right),
\label{eq:app_tweedie}
\end{equation}
which provides a useful consistency check on
\eqref{eq:app_score_prior} and \eqref{eq:app_denoiser_mean}.

\section{Closed-Form Reference Posteriors and Optimal Controls for Linear Benchmarks}
\label{app:linear_posteriors_controls}

This appendix collects the closed-form posterior formulas used in the linear benchmarks of Section~\ref{sec:benchmarks}, together with the corresponding representation of the optimal path-space control. These formulas provide the exact reference distributions used for sampling-based evaluation in Section~\ref{sec:evaluation_methods} and the exact control targets used to assess control learning in the linear problems.

\subsection{Linear--Gaussian conditioning of a Gaussian mixture prior}

Consider the linear observation model
\begin{equation}
\bsy = \mathbf{L}\bsx + \boldsymbol{\eta},
\qquad
\boldsymbol{\eta} \sim \mathcal{N}(\mathbf{0},\mathbf{R}),
\label{eq:app_linear_obs}
\end{equation}
where $\bsx \in \mathbb{R}^D$, $\bsy \in \mathbb{R}^K$, $\mathbf{L} \in \mathbb{R}^{K\times D}$, and $\mathbf{R} \in \mathbb{R}^{K\times K}$ is symmetric positive definite. Let the prior be the Gaussian mixture
\begin{equation}
\pi_{\mathrm{pr}}(\bsx)
=
\sum_{i=1}^{N_m} w_i\,\mathcal{N}(\bsx;\boldsymbol{\mu}_i,\boldsymbol{\Sigma}_i),
\qquad
w_i \ge 0,
\qquad
\sum_{i=1}^{N_m} w_i = 1,
\label{eq:app_linear_gmm_prior}
\end{equation}
as in \eqref{eq:app_gmm_prior}.

Introducing the latent mixture index $I \in \{1,\dots,N_m\}$, we may write
\[
\bsx \mid (I=i) \sim \mathcal{N}(\boldsymbol{\mu}_i,\boldsymbol{\Sigma}_i),
\qquad
\bsy \mid \bsx \sim \mathcal{N}(\mathbf{L}\bsx,\mathbf{R}).
\]
Conditioned on $I=i$, the model is linear--Gaussian, and standard Gaussian conditioning gives
\begin{equation}
\bsx \mid (\bsy,\; I=i) \sim \mathcal{N}(\boldsymbol{\mu}_{i\mid \bsy},\boldsymbol{\Sigma}_{i\mid \bsy}),
\label{eq:app_component_posterior}
\end{equation}
where
\begin{align}
\mathbf{S}_i
&\triangleq
\mathbf{L}\boldsymbol{\Sigma}_i\mathbf{L}^\top + \mathbf{R},
\label{eq:app_Si}
\\
\boldsymbol{\mu}_{i\mid \bsy}
&=
\boldsymbol{\mu}_i
+
\boldsymbol{\Sigma}_i \mathbf{L}^\top \mathbf{S}_i^{-1} \bigl(\bsy - \mathbf{L}\boldsymbol{\mu}_i\bigr),
\label{eq:app_component_mean}
\\
\boldsymbol{\Sigma}_{i\mid \bsy}
&=
\boldsymbol{\Sigma}_i
-
\boldsymbol{\Sigma}_i \mathbf{L}^\top \mathbf{S}_i^{-1}\mathbf{L}\boldsymbol{\Sigma}_i.
\label{eq:app_component_cov}
\end{align}
The posterior mixture weights are
\begin{equation}
\widetilde w_i(\bsy)
=
\frac{
w_i\,\mathcal{N}(\bsy;\mathbf{L}\boldsymbol{\mu}_i,\mathbf{S}_i)
}{
\sum_{j=1}^{N_m}
w_j\,\mathcal{N}(\bsy;\mathbf{L}\boldsymbol{\mu}_j,\mathbf{S}_j)
}.
\label{eq:app_component_weights}
\end{equation}
Hence the posterior distribution remains a Gaussian mixture:
\begin{equation}
\pi_{\mathrm{post}}(\bsx \mid \bsy)
=
\sum_{i=1}^{N_m} \widetilde w_i(\bsy)\, \mathcal{N}(\bsx;\boldsymbol{\mu}_{i\mid \bsy},\boldsymbol{\Sigma}_{i\mid \bsy}).
\label{eq:app_exact_posterior}
\end{equation}

\subsection{Specialization to the linear benchmarks in Appendix~\ref{app:benchmark_details}}

For the random linear sensing benchmark of Appendix~\ref{app:random_linear_details}, we identify
\[
\mathbf{L} = \mathbf{H},
\qquad
\mathbf{R} = \boldsymbol{\Gamma},
\]
with $\mathbf{H}$ and $\boldsymbol{\Gamma}$ defined in \eqref{eq:app_gmm_prior}--\eqref{eq:app_linear_obs} and Appendix~\ref{app:random_linear_details}. Therefore \eqref{eq:app_component_mean}--\eqref{eq:app_exact_posterior} give the exact posterior used to generate reference samples for that benchmark.

For the inpainting benchmark of Appendix~\ref{app:inpainting_details}, we identify
\[
\mathbf{L} = \mathbf{M},
\qquad
\mathbf{R} = \sigma^2 \mathbf{I}_K,
\]
with $\mathbf{M}$ and $\sigma$ defined in Appendix~\ref{app:inpainting_details}. The same formulas therefore give the exact posterior used to generate reference samples for the inpainting problem.

\subsection{Optimal path-space control for the linear benchmarks}

For the reparameterized base process \eqref{eq:uncontrolled_reparam_dynamics}, the optimal control solving the path-space posterior sampling problem is given by
\begin{equation}
\bsu^\star(\bsx_t,t)
= \bsB(t)^\top \nabla_{\bsx} \log g(\bsx_t,t),
\qquad
g(\bsx_t,t)
\triangleq
\mathbb{E}_{P}\!\left[dQ^\star /dP\,\middle|\,\bsx_t\right],
\label{eq:app_opt_control_general}
\end{equation}
as shown in \eqref{eq:analytical_optimal_control}. Since
\[
\frac{dQ^\star}{dP}(\bsx_{-1:T})
=
\frac{\pi_{\mathrm{like}}(\bsy \mid \bsx_T)}{Z},
\]
the normalizing constant $Z$ is irrelevant after taking the gradient of the logarithm, and we may equivalently write
\begin{equation}
g(\bsx_t,t)
\propto
\mathbb{E}_{P}\!\left[\pi_{\mathrm{like}}(\bsy \mid \bsx_T)\,\middle|\,\bsx_t\right].
\label{eq:app_g_likelihood}
\end{equation}

Under the Gaussian perturbation kernel and denoising formulas of
Appendix~\ref{app:closed_form_scores}, the conditional law of $\bsx_T$ given
$\bsx_t=\bsx$ is a Gaussian mixture,
\begin{equation}
p(\bsx_T \mid \bsx_t=\bsx)
=
\sum_{i=1}^{N_m}
r_{t,i}(\bsx)\,
\mathcal{N}\bigl(\bsx_T;\mathbf{m}_{t,i}(\bsx),\mathbf{P}_{t,i}\bigr),
\label{eq:app_xt_conditional}
\end{equation}
where $r_{t,i}(\bsx)$, $\mathbf{m}_{t,i}(\bsx)$, and $\mathbf{P}_{t,i}$ are given in Appendix~\ref{app:closed_form_scores}. Substituting \eqref{eq:app_xt_conditional} into \eqref{eq:app_g_likelihood} and evaluating the Gaussian integral componentwise yields
\begin{equation}
g(\bsx,t)
\propto
\sum_{i=1}^{N_m}
r_{t,i}(\bsx)\, \mathcal{N}\!\bigl( \bsy; \mathbf{L}\mathbf{m}_{t,i}(\bsx), \mathbf{L}\mathbf{P}_{t,i}\mathbf{L}^\top + \mathbf{R}
\bigr).
\label{eq:app_g_closed_form}
\end{equation}
Therefore the optimal control can be written explicitly as
\begin{equation}
\bsu^\star(\bsx,t)
=
\bsB(t)^\top \nabla_{\bsx} \log \left[
\sum_{i=1}^{N_m} r_{t,i}(\bsx)\, \mathcal{N}\!\bigl( \bsy; \mathbf{L}\mathbf{m}_{t,i}(\bsx), \mathbf{L}\mathbf{P}_{t,i}\mathbf{L}^\top + \mathbf{R} \bigr)
\right].
\label{eq:app_opt_control_linear}
\end{equation}
For the random linear sensing benchmark, this formula is used with
\[
\mathbf{L}=\mathbf{H},
\qquad
\mathbf{R}=\boldsymbol{\Gamma},
\]
whereas for the inpainting benchmark it is used with
\[
\mathbf{L}=\mathbf{M},
\qquad
\mathbf{R}=\sigma^2 \mathbf{I}_K.
\]
In both linear benchmarks, the exact posterior and the corresponding optimal path-space control are available as explicit finite Gaussian-mixture expressions. In practice, \eqref{eq:app_opt_control_linear} can be evaluated directly, with the gradient computed either analytically or by automatic differentiation.

\section{Implementation of Comparison Algorithms}
\label{app:comp_alg_imp}

In this appendix, we discuss the implementation of the DPS, DAPS, and $\Pi$GDM algorithms used in the numerical studies. 

DPS uses the control given by \eqref{eq:dps_control}, with $\mathbb{E}_P[\bsx_T \mid \bsx_t]$ computed analytically using Tweedie's formula \cite{efron2011tweedie}. In practical implementations, this control is typically scaled by a user-specified parameter $\xi > 0$. The scaling parameter $\xi$ is introduced to modulate the strength of the likelihood-based correction relative to the base diffusion dynamics, since the approximation underlying DPS does not in general determine the appropriate magnitude of the control. In our implementation, we tune $\xi$ to achieve stable and competitive performance.

The $\Pi$GDM algorithm only applies to linear inverse problems with additive Gaussian noise, as well as non-linear problems where the non-linear operator admits a pseudo-inverse. We therefore only implement the $\Pi$GDM algorithm for the random linear sensing and inpainting problems. For an inverse problem with forward operator $\mathbf{L}$ and noise covariance $\boldsymbol{R}$, the $\Pi$GDM control is 
$$
\bsu(\bsx_t, t) = \bsB(t)^T \nabla_{\bsx} \log \mathcal{N}(\bsx; \mathbf{L} \, \mathbb{E}_P[\bsx_T \mid \bsx_t], \boldsymbol{R} + \sigma_t^2 \mathbf{L}\mathbf{L}^T),
$$
where here $\sigma_t^2$ is the marginal variance of $p_{t \mid T}(\bsx_t \mid \bsx_T)$ and is given by the forward diffusion SDE in \eqref{eq:forward_diffusion_sde}; see Appendix \ref{app:closed_form_scores} for details. Note that this is the same as the DPS control for linear-Gaussian inverse problems, but with an additional additive covariance term given by $\sigma_t^2 \mathbf{L}\mathbf{L}^T$. There are no associated hyperparameters. 

In DAPS, the machinery of diffusion models is used to transform the original posterior sampling problem into a series of posterior sampling problems for an inverse problem with a Gaussian prior. In particular, given an iterate $\bsx_t$, a sub-problem is constructed using the ground-truth likelihood function $\pi_{\mathrm{like}}$ and a Gaussian prior with white noise covariance and mean given by an estimate of $\mathbb{E}_P[\bsx_T \mid \bsx_t]$. After obtaining a sample $\hat{\bsx}_T(\bsx_t)$ from this sub-problem (e.g., using Langevin dynamics), the subsequent iterate $\bsx_{t'}$, $t' > t$, is obtained by exploiting the knowledge of $p_{t' \mid T}(\bsx_{t'} \mid \hat{\bsx}_T(\bsx_t))$ available in the diffusion modeling context (see Appendix \ref{app:closed_form_scores}). 
 
We implement DAPS using the hyperparameter settings in the original paper \cite{zhang2025improving} as a starting point, with the number of Langevin steps tuned separately for each of the four experiments in our benchmark. For the two linear inverse problems, we adopt the setup used for the inpainting experiments in \cite{zhang2025improving}; for the two nonlinear problems, we adopt the setup used for the phase retrieval experiments. Across all experiments, we set the decay rate in the Langevin step-size scheduler to zero. These adjustments were made to obtain strong DAPS performance in our problem setting.

\section{Optimization Diagnostics}
\label{app:optimization_details}
This section reports diagnostic statistics for the trust-region outer loop used by the TR sampler. For each independent run, we record the realized number of outer iterations,
\[
I_{\mathrm{used}}
=
\begin{cases}
\min\{i:\lambda_i < 0.1\}, & \text{if early termination occurs},\\
I_{\max}, & \text{otherwise},
\end{cases}
\]
where $I_{\max}=300$. We report the mean and standard deviation of $I_{\mathrm{used}}$ across trials, together with the percentage of runs that terminate before reaching the iteration cap.

  \begin{table}[h]
  \centering
  \begin{tabular}{lccc}
  \toprule
  Problem & $\epsilon$ & $I_{\mathrm{used}}$ (mean $\pm$ std) & Early Terminated \\
  \midrule
  Random linear sensing & 0.1  & $221.0 \pm 85.2$  & $60\%$ $(6/10)$ \\
  Inpainting & 0.1  & $179.9 \pm 119.0$ & $60\%$ $(6/10)$ \\
  X-ray tomography & 0.01 & $300.0 \pm 0.0$   & $0\%$ $(0/10)$ \\
  Phase retrieval & 0.1  & $300.0 \pm 0.0$   & $0\%$ $(0/10)$ \\
  \bottomrule
  \end{tabular}
  \caption{Optimization statistics for the trust-region outer loop used by the TR sampler.}
  \label{tab:optimization_details}
  \end{table}

\section{Evaluation Metrics}
\label{app:evaluation_metrics}
This appendix provides additional details on the quantitative metrics summarized in Section~\ref{sec:evaluation_methods}.

\subsection{Posterior mean and covariance errors}
We assess agreement in low-order statistics by comparing the posterior mean and covariance estimated from the generated samples with those obtained from the reference samples. Let $\hat{\boldsymbol{\mu}}$ and $\boldsymbol{\mu}_{\mathrm{ref}}$ denote the posterior means, and let $\hat{\boldsymbol{\Sigma}}$ and $\boldsymbol{\Sigma}_{\mathrm{ref}}$ denote the corresponding posterior covariance matrices. The posterior mean error is measured using the Euclidean two-norm $
\|\hat{\boldsymbol{\mu}} - \boldsymbol{\mu}_{\mathrm{ref}}\|_2$.
To quantify discrepancies in second-order structure, we measure covariance error using the Fisher--Rao metric between covariance matrices $
d_{\mathrm{FR}}(\hat{\boldsymbol{\Sigma}}, \boldsymbol{\Sigma}_{\mathrm{ref}})
=
\left\|
\log\!\left(
\boldsymbol{\Sigma}_{\mathrm{ref}}^{-1/2}
\hat{\boldsymbol{\Sigma}}
\boldsymbol{\Sigma}_{\mathrm{ref}}^{-1/2}
\right)
\right\|_F,$
where $\|\cdot\|_F$ denotes the Frobenius norm.

\subsection{Maximum mean discrepancy (MMD)}
We also measure global distributional agreement using the maximum mean discrepancy (MMD) \cite{gretton2012kernel}, which compares probability measures through their kernel mean embeddings in a reproducing kernel Hilbert space (RKHS). For two distributions $\pi_1$ and $\pi_2$, the MMD is defined by $
\mathrm{MMD}(\pi_1, \pi_2)
=
\sup_{f \in \mathcal{F}}
\left(
\mathbb{E}_{\pi_1}[f(\bsx)]
-
\mathbb{E}_{\pi_2}[f(\bsx)]
\right),$
where $\mathcal{F}$ denotes the unit ball of the RKHS associated with a positive-definite kernel $k$. Following \cite{scopecrafts2025benchmarking}, we use a multi-scale Gaussian kernel of the form
$
k(\bsx_1, \bsx_2)
=
\sum_{i=1}^{N_b}
\exp\!\left(
-\frac{\|\bsx_1 - \bsx_2\|_2^2}{\sigma_i}
\right),
$
with $N_b = 5$ and bandwidths $\sigma_i
=
\bar{\sigma}\, 2^{\,i - \lceil N_b/2\rceil},$
where $\bar{\sigma}$ is set to the average squared Euclidean distance between reference samples. The MMD is computed using its standard empirical estimator based on the generated and reference samples.

\subsection{Central moment discrepancy (CMD)}
To assess agreement beyond second-order structure, we employ the central moment discrepancy (CMD) metric \cite{zellinger2017central}, which compares distributions through their low-order central moments. For two distributions $\pi_1$ and $\pi_2$, the CMD is defined by
$
\mathrm{CMD}(\pi_1, \pi_2)
=
\frac{1}{\alpha}
\left\|
\mathbb{E}_{\pi_1}[\bsx] - \mathbb{E}_{\pi_2}[\bsx]
\right\|_2
+
\sum_{k=2}^{\infty}
\frac{1}{\alpha^k}
\left\|
c_k(\pi_1) - c_k(\pi_2)
\right\|_2,$
where $c_k(\cdot)$ denotes the vector of $k$th-order central moments and $\alpha > 0$ is a decay parameter. In practice, we truncate the infinite series at order $K=5$, and estimate all moments empirically from the generated and reference samples. Following \cite{scopecrafts2025benchmarking}, we set $
\alpha = 4\hat{\eta}_{\max}, $
where $\hat{\eta}_{\max}$ is an empirical estimate of
$
\eta_{\max}
=
\mathbb{E}_{\bsy}
\left[
\|\boldsymbol{\eta}(\bsy)\|_{\infty}
\right],
$
and $\boldsymbol{\eta}(\bsy)\in\mathbb{R}^D$ denotes the componentwise posterior standard deviation corresponding to the measurement $\bsy$.

\subsection{Control approximation error}
For the linear benchmark problems described in Sections~\ref{prob:random-sensing} and~\ref{prob:inpainting}, the optimal path-space control $\bsu^\star(\bsx,t)$ is available in closed form; see Appendix~\ref{app:linear_posteriors_controls}. This allows us to directly compare the controls induced by different posterior sampling methods with the optimal feedback law. 
Motivated by the sampling error bound in Theorem~\ref{theorem:girsanov_bound}, we report a finite-sample control discrepancy between each method's induced control and $\bsu^\star$. The theorem bounds posterior sampling error by a path-integrated squared control mismatch averaged under the approximate controlled process. The diagnostic used here differs in two ways: it uses the unsquared $L^2$ norm and evaluates all methods on a common collection of reference trajectories generated under the optimal control. This makes the metric directly comparable across methods.
Let $\{\bsx_{t_n}^{(j),\star}\}_{j=1,\dots,M;\,n=0,\dots,N}$ denote reference trajectories generated under $\bsu^\star$, where $\{t_n\}_{n=0}^{N}$ is the temporal discretization used for simulation.

For a control $\bsu_{\mathrm{alg}}$ associated with a given sampling method, we define $
\mathcal{E}_{\mathrm{ctrl}}
=
\frac{1}{T\cdot M}
\sum_{j=1}^{M}
\sum_{n=0}^{N - 1}
\left\|
\bsu_{\mathrm{alg}}(\bsx_{t_n}^{(j)},t_n)
-
\bsu^\star(\bsx_{t_n}^{(j)},t_n) 
\right\|_2 \; \Delta t_n,
$
where $\Delta t_n \triangleq t_{n+1} - t_n$. This metric measures the mean pointwise discrepancy between the control associated with a given method and the optimal control over states visited by the optimal controlled process. It is reported only for the linear-Gaussian benchmark problems, where $\bsu^\star$ is available analytically.

\subsection{Sampling efficiency and importance-sampled posterior expectations}
For methods that admit a path-space change-of-measure representation, we assess statistical efficiency using the normalized effective sample size induced by the corresponding importance weights. Let $\{w^{(j)}\}_{j=1}^N$ denote the path-space importance weights defined in \eqref{eq:path_space_weight}, evaluated on a collection of $N$ sampled trajectories. We define the effective sample size by $
\mathrm{ESS}
=
\frac{\left(\sum_{j=1}^N w^{(j)}\right)^2}{\sum_{j=1}^N (w^{(j)})^2},
$
and the normalized effective sample size by
$
\mathrm{NESS}
=
\frac{\mathrm{ESS}}{N}.
$
This quantity measures the degree of weight degeneracy and therefore serves as a diagnostic of the practical efficiency of the path-space importance correction.

In addition to sampling efficiency, we report a self-normalized importance sampling estimate of the posterior mean. Given samples $\{\bsx_T^{(j), \bsu}\}_{j=1}^N$ generated under a possibly suboptimal control $\bsu$, we form $\hat{\boldsymbol{\mu}}_{\mathrm{IS}}
\triangleq
\left (\sum_{j=1}^N w^{(j)} \, \bsx_T^{(j), \bsu} \right) / 
    \left ( \sum_{j=1}^N w^{(j)} \right)$
where $w^{(j)}$ is the path-space importance weight associated with the $j$th sampled trajectory. 
Reporting the reweighted mean error $\|\hat{\boldsymbol{\mu}}_{\mathrm{IS}} - \boldsymbol{\mu}_{\mathrm{ref}}\|_2$ therefore quantifies the extent to which path-space importance weighting can correct bias in approximate posterior samplers. To make this correction explicit, we also report the percentage reduction in posterior mean error. For each independent trial $i$, we compute $
\mathrm{Red}_{\mathrm{mean}}^{(i)} \triangleq 100 \left( 1 - \frac{
\|\hat{\boldsymbol{\mu}}_{\mathrm{IS}}^{(i)}-\boldsymbol{\mu}_{\mathrm{ref}}^{(i)}\|_2
}{
\|\hat{\boldsymbol{\mu}}^{(i)}-\boldsymbol{\mu}_{\mathrm{ref}}^{(i)}\|_2
}
\right),$
where $
\hat{\boldsymbol{\mu}}^{(i)}
\triangleq
\frac{1}{N}\sum_{j=1}^N \bsx_{T,i}^{(j),\bsu} $
is the unweighted sample mean for trial $i$. We report the sample mean and standard deviation of $\mathrm{Red}_{\mathrm{mean}}^{(i)}$ across trials. Positive values indicate that reweighting reduces the posterior mean error, while negative values indicate that reweighting increases it.

Both NESS and the importance reweighting strategy require a sampler that induces an absolutely continuous change of measure with respect to the reference diffusion. Consequently, these quantities are reported only for DPS, $\Pi$GDM, and the proposed trust-region method, but not for DAPS, which does not define a single controlled diffusion path measure and therefore does not yield tractable Radon--Nikodym weights of the form in Theorem~\ref{theorem:reweighting}.

\section{Trust-Region Algorithm Ablation Study}
\label{app:ablation}

In this work, the trust region approach introduced in \cite{blessing2025trust} is used to solve the diffusion posterior sampling problem. In this appendix, we present the results of an ablation study designed to investigate the role of the trust region in stabilizing the optimization of the control. 

\begin{figure}
    \centering
    \includegraphics[width=0.48\linewidth]{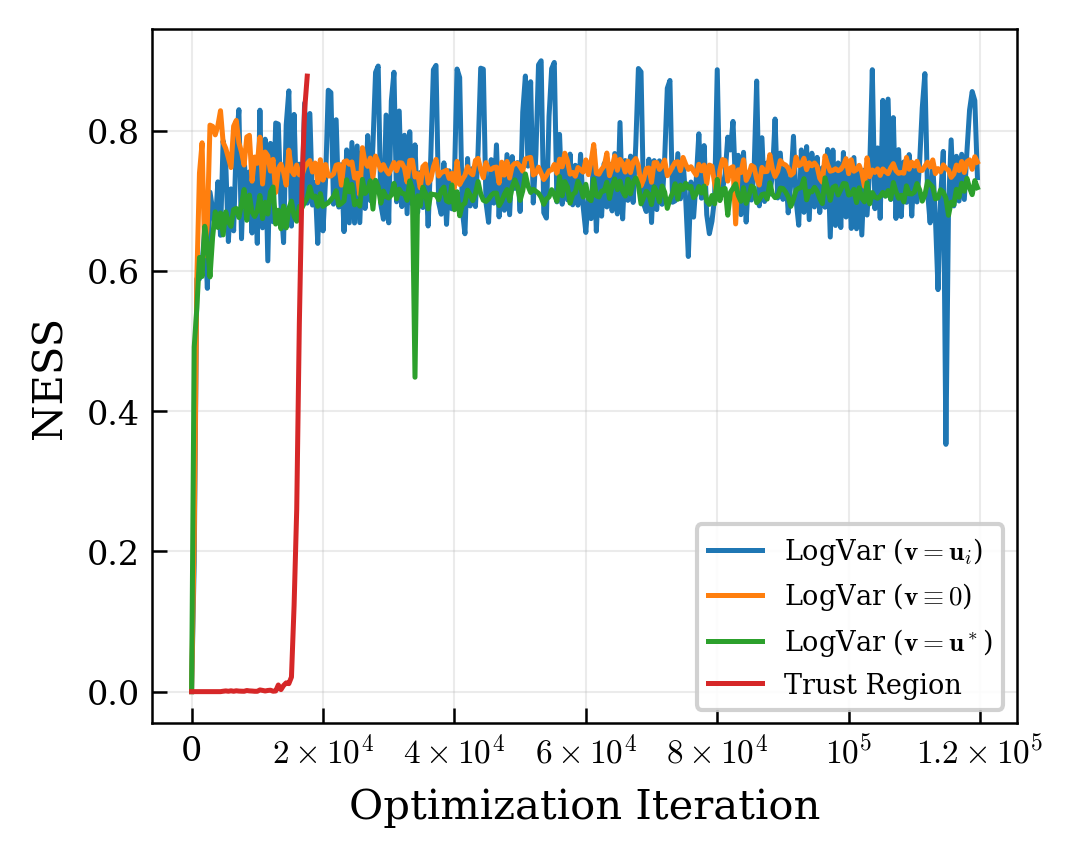}     \includegraphics[width=0.48\linewidth]{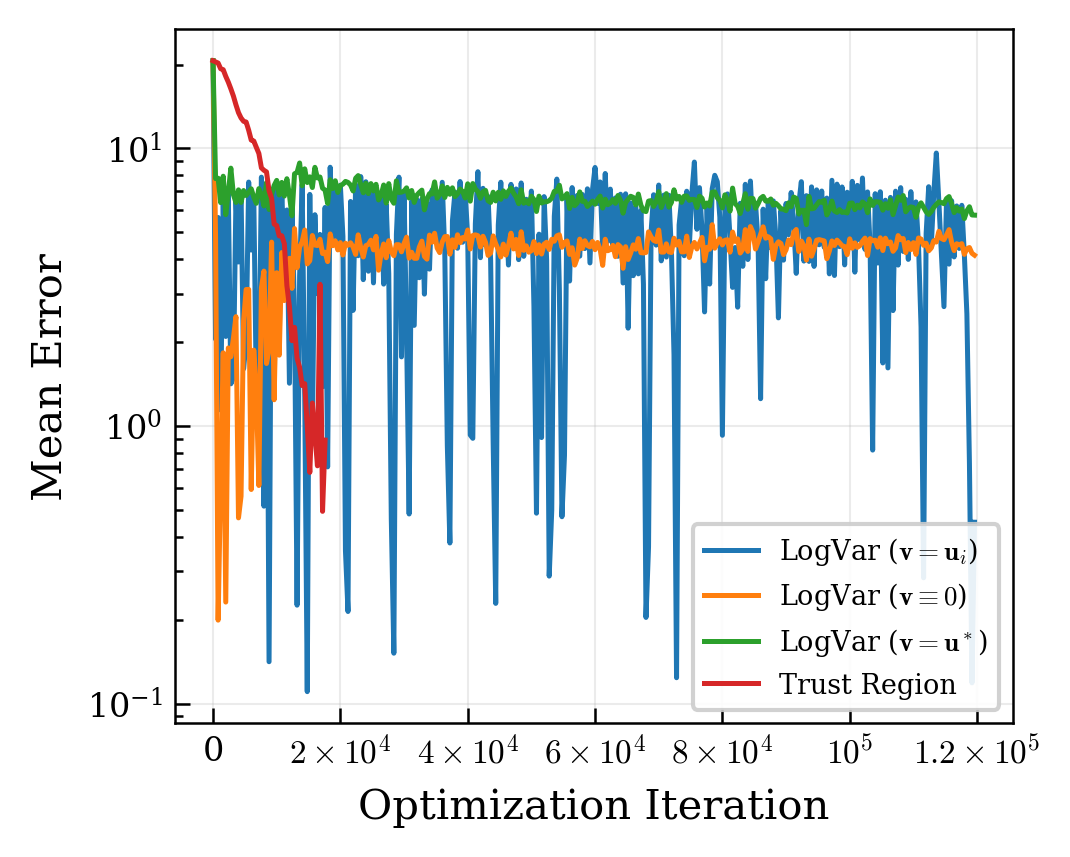} \\
    \vspace{-.1cm}
    \hspace{1cm} (a) \hspace{5.5cm} (b)
    \caption{(a) Normalized effective sample size (NESS) and (b) error in the posterior mean estimate as a function of the optimization iteration for the trust region approach, as well as several IDO variants that do not use trust regions. Note that the trust region approach terminates early due to the termination criteria discussed in Section \ref{sec:num_studies}. }
    \label{fig:ablation}
\end{figure}

The trust region approach with log-variance divergence solves the series of sub-problems given by \eqref{eq:logvar_diverge} and \eqref{eq:uiplus1_measure_exp}, with $\bsv = \bsu_i$ used as the off-policy control in the log-variance objective, and $\lambda_i$ given by solving the dual problem. In our ablations, we use the log-variance divergence to directly to solve the posterior sampling problem without trust regions; this is equivalent to \eqref{eq:logvar_diverge} and \eqref{eq:uiplus1_measure_exp} with  $\lambda_i = 0$ for all $i$. For the off-policy control $\bsv$, which generates the trajectories used for control optimization, we consider three different options: $\bsv = \bsu_i$, as in the trust region approach, $\bsv \equiv \mathbf{0}$, and $\bsv = \bsu^*$, where $\bsu^*$ is the optimal control for the problem and is assumed to be known analytically. We refer to these approaches as LogVar ($\bsv = \bsu_i$), LogVar ($\bsv \equiv \mathbf{0}$), and LogVar ($\bsv = \bsu^*$), respectively. 

To illustrate the effect of the trust region on optimization behavior, we conduct the ablation study on a representative posterior sampling trial of the inpainting problem. The implementations of the three ablation approaches used the same hyperparameters as used for the trust region approach. 

Figure \ref{fig:ablation} shows the normalized effective sample size (NESS) and posterior mean error metrics, as described in Section \ref{sec:num_studies}, as a function of the optimization iteration. As can be seen, the LogVar ($\bsv \equiv \mathbf{0}$), and LogVar ($\bsv = \bsu^*$) approaches both lead to sub-optimal performance when compared with the trust region approach. LogVar ($\bsv = \bsu_i$), which is an on-policy version of the log-variance objective, is able to provide competitive performance. However, the use of the trust region significantly improves the stability of the optimization process. We also note that the problem setting considered here is linear and relatively low-dimensional; in \cite{blessing2025trust}, the authors report that the trust region approach scales significantly better to high dimensional problem settings when compared with the LogVar ($\bsv = \bsu_i$).

\end{document}